\documentclass[twoside,11pt]{article}

\usepackage{blindtext}

\usepackage{jmlr2e}

\usepackage{amsfonts}
\usepackage{caption}
\usepackage{bm}
\usepackage{CJK}
\usepackage{graphicx,psfrag,color}
\usepackage{fancyhdr}
\usepackage{amsmath}
\usepackage{algorithm}
\usepackage{algpseudocode}
\usepackage{setspace}
\usepackage{mathrsfs}
\usepackage{indentfirst}
\usepackage{booktabs}
\usepackage{multirow}
\usepackage{epsfig, epsfig, graphics, lscape, amsbsy, amstext, natbib}
\newtheorem{assumption}{Assumption}

\usepackage{verbatim}
\usepackage[normalem]{ulem}

\usepackage{lastpage}

\begin{document}

\title{Differentially Private Conformal Prediction}

\author{\name Jiamei Wu$^*$ \email jmwu01@bjtu.edu.cn \\
       \addr School of Mathematics and Statistics\\
       Beijing Jiaotong University\\
       Beijing, Beijing 100044, China
       \AND
       \name Ce Zhang$^*$ \email ce5@ualberta.ca \\
       \addr Department of Mathematical and Statistical Sciences\\
       University of Alberta\\
       Edmonton, AB T6G 2G1, Canada
       \AND
       \name Zhipeng Cai \email  23121698@bjtu.edu.cn\\
       \addr School of Mathematics and Statistics\\
       Beijing Jiaotong University\\
       Beijing, Beijing 100044, China
       \AND
       \name Jingsen Kong 
       \email jingsen@ualberta.ca\\
       \addr Department of Mathematical and Statistical Sciences\\
       University of Alberta\\
       Edmonton, AB T6G 2G1, Canada
        \AND
       \name Bei Jiang 
       \email bei1@ualberta.ca\\
       \addr Department of Mathematical and Statistical Sciences\\
       University of Alberta\\
       Edmonton, AB T6G 2G1, Canada
       \AND
        \name Linglong Kong 
        \email lkong@ualberta.ca \\
       \addr Department of Mathematical and Statistical Sciences\\
       University of Alberta\\
       Edmonton, AB T6G 2G1, Canada
        \AND
        \name Lingchen Kong 
        \email  lchkong@bjtu.edu.cn\\
       \addr School of Mathematics and Statistics\\
       Beijing Jiaotong University\\
       Beijing, Beijing 100044, China} 

\maketitle


\begin{abstract}

Conformal prediction (CP) has attracted broad attention as a simple and flexible framework for uncertainty quantification through prediction sets. In this work, we study how to deploy CP under differential privacy (DP) in a statistically efficient manner. We first introduce differential CP, a non-splitting conformal procedure that avoids the efficiency loss caused by data splitting and serves as a bridge between oracle CP and private conformal inference. By exploiting the stability properties of DP mechanisms, differential CP establishes a direct connection to oracle CP and inherits corresponding validity behavior. Building on this idea, we develop Differentially Private Conformal Prediction (DPCP), a fully private procedure that combines DP model training with a private quantile mechanism for calibration. We establish the end-to-end privacy guarantee of DPCP and investigate its coverage properties under additional regularity conditions. We further study the efficiency of both differential CP and DPCP under empirical risk minimization and general regression models, showing that DPCP can produce tighter prediction sets than existing private split conformal approaches under the same privacy budget. Numerical experiments on synthetic and real datasets demonstrate the practical effectiveness of the proposed methods.
\end{abstract}



\section{Introduction}
\label{sec:Introduction}

Many predictive models often produce point predictions without accompanying uncertainty measures \citep{lei2018distribution,gawlikowski2023survey}. In high-risk applications, including health care, criminal justice, and financial risk assessment, prediction errors can have substantial consequences, and overconfident model outputs can lead to unreliable or unsafe decisions \citep{Kourou2014MachineLA,Nowotarski2016RecentAI,mashrur2020machine,Melotti2021ProbabilisticAF}. Conformal prediction (CP) is a model-agnostic framework for uncertainty quantification that constructs prediction sets with finite-sample coverage guarantees under exchangeability \citep{papadopoulos2002inductive,vovk2005algorithmic,shafer2008tutorial}. Its validity does not depend on the correct specification of the model and is broadly applicable to any prediction algorithm \citep{fontana2023conformal,angelopoulos2024theoretical}. Recent work has extended CP beyond the classical i.i.d. setting to accommodate more complex data regimes, including distribution and covariate shift \citep{tibshirani2019conformal,barber2023conformal,oliveira2024split}, streaming and sequential data \citep{gibbs2024conformal}, and a growing range of inferential tasks such as survival analysis \citep{gui2024conformalized}, causal inference \citep{yin2024conformal}, and selective inference \citep{jin2025model}.

Despite these advances, most of the existing CP methods do not incorporate privacy protection by design, thus limiting their applicability in privacy-sensitive settings. In particular, CP relies on calibration data and data-dependent statistics, which can introduce potential privacy risks when prediction sets or related quantities are publicly released. Differential privacy (DP), a crucial component of trustworthy machine learning, provides rigorous guarantees by controlling how much information about any single individual can be inferred from released outputs \citep{dwork2006differential, dwork2013algorithmic}. Motivated by its rigorous privacy guarantees, DP has been increasingly applied across a range of domains, including health care, financial services, and public-sector services \citep{piao2019privacy,wang2022protection,drechsler2023differential}.

This paper aims to study the integration of CP with DP and to construct prediction sets that provide rigorous uncertainty quantification while simultaneously ensuring formal privacy protection for the underlying data. Existing work on CP under DP is limited. \citet{angelopoulos2022private} combine split CP with differentially private quantile estimation to construct private prediction intervals, using a calibrated quantile adjustment to maintain finite-sample coverage. Subsequent work has explored related directions under more specialized settings. \citet{humbert2023one} investigate private quantile estimation in a federated learning framework, employing the exponential mechanism to aggregate information across distributed clients while preserving privacy. \citet{plassier2023conformal} consider privacy-preserving prediction sets by injecting noise into gradient-based training procedures, thereby achieving DP at the model-training stage rather than through explicit privatization of conformal calibration. However, empirical results on real datasets (see Section~\ref{sec:Numerical Study} and Appendix~\ref{D}) indicate that existing private CP methods can yield overly conservative prediction sets. This behavior arises primarily from the use of split CP in current differentially private conformal methods, where the data are partitioned into training and calibration subsets. Because only a fraction of the available data is used at each stage, split CP incurs an inherent loss of statistical efficiency \citep{lei2013distribution-free,lei2018distribution,barber2021predictive}. Moreover, under DP, this reduction in effective sample size further amplifies the impact of privacy noise whose magnitude scales inversely with the sample size, leading to wider and less informative prediction sets \citep{chaudhuri2019capacity,bun2017make}.

In this paper, we develop a novel differentially private conformal prediction (DPCP) method that addresses limitations of existing conformal methods by improving statistical efficiency and prediction quality while maintaining rigorous privacy guarantees. Our main contributions are summarized as follows.

\begin{itemize}
\item We introduce differential CP, a conformal procedure that avoids data splitting by leveraging the stability of a DP mechanism on adjacent datasets. By comparing the resulting predictor based on $\mathcal{A}(\mathcal{D}_n)$ with an oracle conformal predictor on the augmented dataset $\mathcal{D}_{n+1}$, we establish finite-sample coverage guarantees without sacrificing data efficiency.

\item Building on differential CP, we develop \emph{differentially private conformal prediction} (DPCP), a model-agnostic framework compatible with general DP learning mechanisms. DPCP achieves end-to-end privacy by privatizing the calibration threshold via a DP quantile estimator, and yields a sharper quantile correction of order $2/(n\varepsilon)$ than the split-based adjustment in~\citet{angelopoulos2022private}, leading to less conservative prediction sets.

\item We provide theoretical guarantees for both differential CP and DPCP. For differential CP, we establish oracle-comparison and finite-sample validity results. For DPCP, we prove end-to-end privacy, study exact marginal coverage under additional regularity conditions, present an idealized conditional coverage result under a released-model approximation, and quantify the efficiency gap relative to the non-private oracle predictor.

\item Numerical experiments on synthetic and real datasets show that DPCP attains competitive empirical coverage and often yields shorter prediction sets than existing private split conformal baselines.
\end{itemize}


The remainder of the paper is organized as follows. In Section \ref{sec:Preliminaries}, we review the foundational concepts and frameworks of CP and DP. Section \ref{sec:DPCP} introduces differential CP and the proposed DPCP framework. Section \ref{sec:theory} establishes the theoretical properties of the proposed methods, including coverage, privacy guarantees, and efficiency. In Section \ref{sec:Numerical Study}, we demonstrate the practical effectiveness of our approach through numerical experiments. Finally, Section \ref{sec:Discussion} concludes with our main findings and offers a further detailed discussion. All technical details and additional simulation results are included in the Appendix.

\section{Preliminaries}
\label{sec:Preliminaries}
In this section, we briefly review the fundamental concepts of split CP in Section \ref{sec:Conformal Prediction} and DP in Section \ref{sec: Differential Privacy}, which form the basis of the methodology proposed in this paper.


\subsection{Split Conformal Prediction}
\label{sec:Conformal Prediction}

Consider a dataset $\mathcal{D}_n = \{(X_i, Y_i)\}_{i=1}^n$ consisting of $n$ observations, where each pair $(X_i, Y_i)$ lies in the product space $\mathcal{X} \times \mathcal{Y}$. For example, in a regression setting, one may take $\mathcal{X} = \mathbb{R}^d$ for $d \ge 1$ as the covariate space and $\mathcal{Y} = \mathbb{R}$ as the response space. Given a specified miscoverage level $\alpha\in(0, 1)$ and under the mild assumption that the data points $\{(X_i, Y_i)\}_{i=1}^{n+1}$ are exchangeable, CP provides a prediction set $C_{\alpha}(\cdot)$ that includes the unknown response $Y_{n+1}$ for a new covariate vector $X_{n+1}$ with finite-sample marginal coverage at least $1-\alpha$, i.e.,
\begin{align}
\label{eq:2}
{\rm Pr}\left( {Y_{n + 1} \in C_{\alpha}\left( X_{n + 1} \right)} \right) \geq 1 - \alpha,
\end{align}
where the probability is taken over the joint distribution of $\{(X_i,Y_i)\}_{i=1}^{n+1}$.

The split CP procedure typically consists of two main steps: training and calibration. Generally, let $\mathcal{D}_{tra}=\{(X_i, Y_i)\}_{i=1}^{m}$ and $\mathcal{D}_{cal}=\{(X_i, Y_i)\}_{i=m+1}^n$ denote the subsets of data used for training and calibration, respectively. In the training step, we fit a predictive model using a learning algorithm $\mathcal{A}$, and denote the resulting predictor by $\widehat{\mu}= \mathcal{A}\left(\mathcal{D}_{tra}\right)$. In the calibration step, we evaluate how well the trained model fits each point in 
$\mathcal{D}_{cal}$ by computing the nonconformity scores $R_i = R((X_i, Y_i), \widehat{\mu}), (X_i, Y_i) \in \mathcal{D}_{cal}$, where $\widehat{\mu} = \mathcal{A}(\mathcal{D}_{tra})$ is the predictor obtained from the training step. These scores quantify the degree to which each calibration sample deviates from the model's predictions. To construct the CP set, we compute the empirical $(1-\alpha)$ quantile of the nonconformity scores. We denote this quantile by $q(\alpha, \mathcal{D}_{tra}, \mathcal{D}_{cal})$. Specifically, $q(\alpha, \mathcal{D}_{tra}, \mathcal{D}_{cal})$ is defined as the $\lceil (1-\alpha)(|\mathcal{D}_{cal}| + 1) \rceil$-th smallest value among the calibration scores $\{R_i : (X_i, Y_i) \in \mathcal{D}_{cal}\}$, where
$
R_i = R\bigl((X_i, Y_i), \mathcal{A}(\mathcal{D}_{tra})\bigr)
$. This quantile serves as the threshold that determines the width of the resulting prediction interval. Accordingly, the split CP set for a new covariate value $X_{n+1}$ is given by
\begin{align}
\label{eq:4}
C_{\alpha}^{s}\left(X_{n + 1}\right) = \left\{y \mid R\left((X_{n + 1}, y), \mathcal{A}(\mathcal{D}_{tra}) \right) \leq q(\alpha, \mathcal{D}_{tra}, \mathcal{D}_{cal}) \right\}.
\end{align}

A key requirement for the validity of CP is that the observations be exchangeable. This assumption ensures that the distribution of the nonconformity scores is independent of the order in which the data are observed. The following definition formalizes this notion.

\begin{definition}[\citealp{fontana2023conformal}]
\label{the:exchangeability}
Random variables $R_1, \ldots, R_n$ are said to be exchangeable if their joint distribution is invariant under permutations; that is, for every permutation $\pi$ of $\{1,\ldots,n\}$,
\begin{align*}
(R_{\pi(1)},\ldots,R_{\pi(n)}) \stackrel{d}{=} (R_1,\ldots,R_n).
\end{align*}
\end{definition}

The exchangeability of nonconformity scores arises naturally when the data are exchangeable and the learning algorithm $\mathcal{A}$, together with the score construction, is invariant to the ordering of the samples. It is slightly weaker than the classical independent and identically distributed (i.i.d.) assumption and is often sufficient for conformal validity; see \citet{shafer2008tutorial} for further discussion. 

\begin{proposition}[Standard split conformal validity]
\label{prop:split_cp_validity}
Assume that the calibration scores
$
R_{m+1},...,R_{n}
$
together with the test score
$
R_{n+1}
$
are exchangeable. Then the split conformal predictor
$C_\alpha^{s}(X_{n+1})$ defined in \eqref{eq:4} satisfies
\[
{\rm Pr}\!\left(Y_{n+1}\in C_\alpha^{s}(X_{n+1})\right)\ge 1-\alpha .
\]
Moreover, if the scores are almost surely distinct, then
\[
{\rm Pr}\!\left(Y_{n+1}\in C_\alpha^{s}(X_{n+1})\right)
\le 1-\alpha+\frac{1}{|\mathcal{D}_{cal}|+1}.
\]
\end{proposition}

This is a standard result in CP following from the rank-uniformity argument under exchangeability; see, for example, \citet{shafer2008tutorial, angelopoulos2023conformal}.

\subsection{Differential Privacy}
\label{sec: Differential Privacy}

DP \citep{dwork2006differential, dwork2013algorithmic} has become a widely adopted standard for protecting individual information in statistical and machine-learning algorithms. Its central idea is simple: the output of a randomized algorithm should not change much when a single individual's data is added or removed. In other words, two adjacent datasets differing in only one data point should lead to nearly indistinguishable outputs. This requirement provides a strong, quantifiable notion of privacy that is robust even against adversaries with substantial background knowledge. Adjacent datasets can be defined in two common ways: (i) datasets of equal size that differ by one entry, (ii) datasets whose sizes differ by exactly one. We adopt the latter convention here, though the theoretical developments in this paper hold under either definition.

\begin{definition}
[\citealp{dwork2013algorithmic}]
\label{the:dp_definition}
A randomized algorithm $\mathcal{A}:\mathcal{O}\to\mathcal{S}$ is said to satisfy $(\varepsilon,\delta)$-DP if, for every measurable set $\mathcal{F}\subseteq\mathcal{S}$ and all adjacent datasets $\mathcal{D},\mathcal{D}'\in\mathcal{O}$,
\begin{equation}
\label{eq:1}
{\rm Pr}\!\left(\mathcal{A}(\mathcal{D}')\in \mathcal{F}\right)
\le 
e^{\varepsilon}\,
{\rm Pr}\!\left(\mathcal{A}(\mathcal{D})\in \mathcal{F}\right)
+ \delta .
\end{equation}
When $\delta=0$, the mechanism is simply called $\varepsilon$-DP.
\end{definition}
 The parameter $\varepsilon > 0$ quantifies the privacy budget: a smaller $\varepsilon$ implies
stronger privacy but greater statistical distortion. The parameter $\delta$ is usually
set to a negligible value of the order $O(1/n)$. If an algorithm $\mathcal{A}_1$ is $\varepsilon_1$-DP for some $\varepsilon_1\le \varepsilon$, then it is automatically $\varepsilon$-DP. Moreover, the standard composition theorem states that if $\mathcal{A}_1$ and $\mathcal{A}_2$ are run on the same dataset and satisfy $(\varepsilon_1, \delta_1)$-DP and $(\varepsilon_2, \delta_2)$-DP respectively, then the combined procedure $\mathcal{A}_1 \circ \mathcal{A}_2$ satisfies $(\varepsilon_1+\varepsilon_2, \delta_1+\delta_2)$-DP. The probability in \eqref{eq:1} is taken over the internal randomness of $\mathcal{A}$, with the dataset treated as a fixed input. Since \eqref{eq:1} holds pointwise for all adjacent datasets, the guarantee also holds conditionally on any realized dataset and hence under any data-generating distribution. For our purposes, where DP interacts with CP, it is more natural to view the data as random and treat the privacy guarantee as a conditional probability given the observed dataset. This conditioning does not alter the meaning of DP but aligns the notation with the probabilistic framework used throughout this paper.

\begin{remark}
\citet{dwork2013algorithmic} showed that under an $\varepsilon$-DP algorithm, an individual's decision to participate changes their expected utility by at most a multiplicative factor of $\exp(\varepsilon) \approx 1+\varepsilon$. Thus, meaningful privacy protection requires $\varepsilon$ to be small. The parameter $\delta$ must be even smaller, typically no larger than the inverse of the dataset size to ensure that ``bad'' privacy events occur only with negligible probability. In accordance with this recommendation, we will assume throughout that $\delta = O(n^{-1})$.
\end{remark}

\section{Differentially Private Conformal Prediction}
\label{sec:DPCP}

In this section, we develop the Differentially Private Conformal Prediction (DPCP) approach, which adapts CP to settings where DP must be rigorously enforced. We begin by formulating the formal problem setup that defines the learning and privacy environment in Section \ref{sec:problem_setup}, then introduce the concept of differential CP in Section \ref{sec:dCP}, which connects the conformal validity principle to the adjacency structure of DP. Finally, Section \ref{subsec:Construction of DPCP Set} extends differential CP to a fully private procedure, DPCP, which achieves both statistical coverage and privacy-preserving guarantees.

\subsection{Problem Setup}
\label{sec:problem_setup}

Let $\mathcal{D}_n = \{(X_i, Y_i)\}_{i=1}^n$ denote an i.i.d. sample drawn from an unknown distribution $P_{XY}$ supported in $\mathcal{X} \times \mathcal{Y}$, where $\mathcal{X} \subseteq \mathbb{R}^d$ represents the feature space and $\mathcal{Y} \subseteq \mathbb{R}$ the response domain. Given a new covariate vector $X_{n+1}$, our goal is to construct a set-valued DPCP predictor $C_{\alpha}^{\mathrm{dp}}(X_{n+1}) \subseteq \mathcal{Y}$ that satisfies
\begin{equation}
  \label{eq:dp_coverage_goal}
  {\rm Pr}\left(Y_{n+1} \in C_{\alpha}^{\mathrm{dp}}(X_{n+1})\right) \ge 1 - \alpha,
\end{equation}
for a user-specified miscoverage level $\alpha \in (0,1)$, while ensuring that the procedure for obtaining $C_{\alpha}^{\mathrm{dp}}(\cdot)$ obeys a formal $(\varepsilon,\delta)$-DP constraint with respect to $\mathcal{D}_n$. Here, the probability in \eqref{eq:dp_coverage_goal} is taken over both the data-generating process and the internal randomness of the private mechanism. In the classical CP framework, the validity is based on the exchangeability of $(n+1)$ samples $\{(X_i, Y_i)\}_{i=1}^{n+1}$, which implies that the rank of a new nonconformity score among the observed scores is uniformly distributed. However, when a DP mechanism $\mathcal{A}$ is imposed, data access and randomization through $\mathcal{A}(\mathcal{D}_n)$ generally destroy exact exchangeability. This creates a fundamental trade-off between privacy protection and statistical efficiency: noise added to guarantee privacy may enlarge prediction sets or distort coverage.


To formalize this setting, consider a randomized learning algorithm subject to DP constraints, $\mathcal{A}$ that maps a dataset to a fitted model $\hat{\mu} = \mathcal{A}(\mathcal{D}_n), \hat{\mu}:\mathcal{X} \to \mathcal{Y}$. For each observation $Z_i = (X_i, Y_i)$, define a nonconformity score $R_i = R(Z_i, \hat{\mu})$ that measures how a atypical $Z_i$ is under the fitted model $\hat{\mu}$. We require that $\mathcal{A}$ satisfies the $(\varepsilon, \delta)$-DP property. Our goal is therefore to construct a novel DPCP method that simultaneously achieves the desired $1-\alpha$ coverage guarantee and satisfies $(\varepsilon,\delta)$–DP. This dual requirement brings about three technical challenges. First, split conformal methods divide the data into training and calibration subsets, which reduces statistical efficiency. Second, private quantile release introduces additional randomness, which must be tightly controlled to maintain coverage. Third, when both model training and quantile estimation are privatized, their combined privacy cost must remain within $(\varepsilon,\delta)$. We next show how these challenges can be addressed by reformulating CP in the context of DP.

\subsection{Differential Conformal Prediction}
\label{sec:dCP}


Before introducing our full DPCP framework, we first develop an intermediate CP method, differential CP, motivated by the stability properties of $(\varepsilon,\delta)$-differentially private mechanisms. A differentially private mechanism $\mathcal{A}$ controls how much the probabilities of measurable events can change when evaluated on adjacent datasets. This stability suggests a new way to construct conformal predictors: when a new test point $(X_{n+1}, Y_{n+1})$ is added, the enlarged dataset
$\mathcal{D}_{n+1}=\mathcal{D}_n \cup \{(X_{n+1}, Y_{n+1})\}$
is adjacent to the original dataset $\mathcal{D}_n$. Because $\mathcal{A}$ behaves similarly on adjacent datasets, its stability can be directly converted into coverage guarantees for a CP set. Motivated by this observation, we develop differential CP, a conformal method that explicitly leverages the adjacency stability of DP to obtain distribution-free prediction sets.


Formally, let $\hat{\mu}_n=\mathcal{A}(\mathcal{D}_n)$ denote the model trained on
$\mathcal{D}_n$, and let $R(Z,\hat{\mu}_n)$ be the associated nonconformity score. We define the differential CP set as
\begin{equation}
\label{eq:dCP}
C_{\alpha}^{d}(X_{n+1})
  = \bigl\{\, y : R\big((X_{n+1},y),\hat{\mu}_n\big)
      \le q\big(e^{-\varepsilon}(\alpha-\delta),\mathcal{D}_n,\mathcal{D}_n\big)
    \bigr\}.
\end{equation} 
The adjusted quantile level $e^{-\varepsilon}(\alpha-\delta)$ arises directly from the
definition of $(\varepsilon,\delta)$-DP. Because a differentially private mechanism can change event probabilities by at most $e^{\varepsilon}$ and an additive term $\delta$ when moving to an adjacent dataset, the miscoverage level in the differential CP construction must be reduced accordingly. When the privacy parameter $\varepsilon$ is small, the adjusted level $e^{-\varepsilon}(\alpha-\delta)$ deviates only slightly from the target $\alpha$, so the resulting differential CP set remains close to the non-private oracle set. In contrast, when $\varepsilon$ is larger, the factor $e^{-\varepsilon}$ reduces the effective quantile level more substantially, making the procedure more conservative and producing wider prediction intervals. An alternative refinement based on partitioning the dataset space into regions of varying local sensitivity and invoking exchangeability within each region can yield sharper coverage control.  
When $\varepsilon$ is relatively large, the adjusted level $e^{-\varepsilon}(\alpha-\delta)$ may become overly conservative. A possible refinement is to partition the dataset space into regions with different local sensitivity levels and then calibrate the coverage adjustment more adaptively within each region. Such a refinement may lead to less conservative prediction sets, but a rigorous treatment is beyond the scope of the present paper. For completeness, we provide a heuristic discussion and a simple numerical illustration in Appendix~\ref{C}.

Although the method developed above exploits the stability guarantees of $(\varepsilon,\delta)$-DP mechanisms and already yields prediction sets with coverage guarantees without data splitting, it is crucial to note that the procedure itself is not private. Both the nonconformity scores and the empirical quantile $q(\cdot)$ in \eqref{eq:dCP} are computed directly from the raw dataset and are released without any privacy protection. Thus, the construction should be regarded as a DP-inspired conformal predictor that benefits from privacy-induced stability, rather than a fully private algorithm. To obtain an end-to-end private CP set while preserving the full-data efficiency of this approach, we next introduce a privatized extension of differential CP. The DPCP framework incorporates two key ingredients: a DP training mechanism and a DP quantile estimator. Combining these components guarantee that the resulting prediction set satisfies $(\varepsilon,\delta)$-DP while retaining the statistical advantages of the non-private
construction. In this sense, DPCP serves as a practical and privacy-preserving refinement of the method developed above.

\subsection{Differentially Private Conformal Prediction Set}
\label{subsec:Construction of DPCP Set}

The differential CP method introduced in Section~\ref{sec:dCP} attains valid coverage by leveraging the adjacency stability of a DP mechanism, thereby avoiding data splitting and repeated model fitting. However, it is not an end-to-end private procedure: both the nonconformity scores and the empirical quantile are computed directly from the non-private data. To achieve full DP while preserving the full-data efficiency afforded by differential CP, we now develop a privatized extension, which we refer to as DPCP. The idea is straightforward: we retain the structure of the differential CP set but replace its two data-dependent components with privacy-preserving counterparts. First, the fitted model $\hat{\mu}_n$ is obtained through a DP training mechanism $\mathcal{A}_{\mathrm{train}}$ that guarantees $(\varepsilon_1,\delta)$-DP. Second, the empirical quantile of the nonconformity scores is replaced with a differentially private quantile estimator $\hat{q}$ obtained via the exponential mechanism. Under this construction, the resulting prediction set inherits the statistical efficiency of differential CP while providing a formal $(\varepsilon,\delta)$ privacy guarantee.

Following \citet{angelopoulos2022private}, we discretize the score range into $M$ bins and use the exponential mechanism to select a privatized quantile. Assume that the nonconformity scores are normalized to lie in $[0,1]$. Let $0=e_0 < e_1 < \cdots < e_M = 1$ denote the bin boundaries. Algorithm~\ref{alg:DPQ} provides an $(\varepsilon_2,0)$-DP estimator of the corrected score quantile, which is then substituted into \eqref{eq:5} to construct the final DPCP set. Formally, let
$
\hat{\mu}_n=\mathcal{A}_{\mathrm{train}}(\mathcal{D}_n)
$,
where $\mathcal{A}_{\mathrm{train}}$ satisfies $(\varepsilon_1,\delta)$-DP, and let
$
R_i = R(Z_i,\hat{\mu}_n), i=1,\ldots,n.
$
Define
$
\alpha_1=e^{-\varepsilon_1}(\alpha-\delta)
$.
Then the DPCP set for a new covariate vector $X_{n+1}$ is defined as
\begin{align}
\label{eq:5}
C_{\alpha}^{dp}(X_{n+1})
  = \left\{
        y :
        R\left((X_{n+1},y),\hat{\mu}_n\right)
        \le \hat{q} \left(\alpha_1,\mathcal{D}_n,\mathcal{D}_n\right)
    \right\},
\end{align}
where $\hat{q}$ is the output of the DP quantile mechanism described in Algorithm~\ref{alg:DPQ}. We note that Algorithm~1 does not return the exact corrected empirical quantile. Rather, it returns a differentially private approximation obtained via the exponential mechanism. The correction $
\alpha_1=e^{-\varepsilon_1}(\alpha-\delta)$
follows directly from the coverage analysis in Section~\ref{sec:dCP}, where $(\varepsilon_1,\delta)$ are the privacy parameters of the DP training mechanism. The quantile release uses privacy budget $\varepsilon_2$. By the composition theorem, the overall procedure satisfies $(\varepsilon_1+\varepsilon_2,\delta)$-DP. Finally, Algorithm~\ref{alg:DPCP} summarizes the full procedure for constructing the DPCP set $C_\alpha^{dp}(X_{n+1})$.

\begin{algorithm}[t]
\caption{Differentially Private Quantile for Conformal Prediction}
\label{alg:DPQ}
\begin{algorithmic}[1]
\Require Calibration dataset $\mathcal{D}_{cal}$ with size $N$, fitted model $\hat{\mu}=\mathcal{A}(\mathcal{D}_{tra})$, bins $I_1,\ldots,I_M$, privacy level $\varepsilon>0$, and input level $\beta\in(0,1)$ satisfying $\beta>\frac{2}{N\varepsilon}$
\Ensure A private quantile $\hat{q} \left(\beta,\mathcal{D}_{tra},\mathcal{D}_{cal}\right)$

\State Compute $\alpha_0 = \beta - \frac{2}{N\varepsilon}$
\For{$i \in \{1,\ldots,N\}$}
    \State $R_i = R(Z_i,\hat{\mu})$
\EndFor
\For{$j \in \{1,\ldots,M\}$}
    \State $w_j = \max\left\{
      \frac{\left|\{i:R_i < e_j\}\right|}{1-\alpha_0},
      \frac{\left|\{i:R_i > e_j\}\right|}{\alpha_0}
      \right\}$
    \State $p_j = \exp\!\left(- \frac{\varepsilon w_j}{2\Delta}\right)$, where
    $\Delta = \max\left(\frac{1}{1-\alpha_0},\frac{1}{\alpha_0}\right)$
\EndFor
\State Select $\hat{q}=e_j$ with probability $p_j/\sum_{j=1}^M p_j$
\State \Return $\hat{q} \left(\beta,\mathcal{D}_{tra},\mathcal{D}_{cal}\right)$
\end{algorithmic}
\end{algorithm}

The most closely related method is the private split conformal approach of \citet{angelopoulos2022private}, which applies the exponential mechanism to calibration scores obtained from a data-splitting procedure. DPCP differs from this method in several important ways. First, split conformal prediction partitions the dataset into training and calibration subsets, so only part of the data contributes to each stage of the procedure. Under DP, where each data point is valuable, this loss of information can be particularly costly. DPCP avoids data splitting entirely and uses the full dataset for both model fitting and score computation, which improves statistical efficiency and can yield tighter prediction sets in practice. Second, the private quantile used in \citet{angelopoulos2022private} requires an $O(\log n / n)$-type correction to ensure finite-sample validity. 

\begin{algorithm}[t]
\caption{Differentially Private Conformal Prediction}
\label{alg:DPCP}
\begin{algorithmic}[1]
\Require Training dataset $\mathcal{D}_n$, new test sample $X_{n+1}$, DP training mechanism $\mathcal{A}_{\mathrm{train}}$ satisfying $(\varepsilon_1,\delta)$-DP, coverage level $1-\alpha$, privacy budget $\varepsilon>\varepsilon_1$, and bins $\{I_1,\ldots,I_M\}$
\Ensure Differentially private conformal prediction set $C_\alpha^{dp}(X_{n+1})$

\State Compute $\hat{\mu}_n=\mathcal{A}_{\mathrm{train}}(\mathcal{D}_n)$
\State Compute $\alpha_1=e^{-\varepsilon_1}(\alpha-\delta)$ and $\varepsilon_2=\varepsilon-\varepsilon_1$
\For{$i \in \{1,\ldots,n\}$}
    \State $R_i = R(Z_i,\hat{\mu}_n)$
\EndFor
\State Use Algorithm~\ref{alg:DPQ} to compute $\hat{q}=\hat{q}(\alpha_1,\mathcal{D}_n,\mathcal{D}_n)$ with privacy level $\varepsilon_2$
\State \Return $C_\alpha^{dp}(X_{n+1})=\left\{y \mid R\left((X_{n+1},y),\hat{\mu}_n\right)\le \hat{q}\right\}$
\end{algorithmic}
\end{algorithm}

In contrast, our analysis leads to an explicit adjustment of $2/(n\varepsilon)$. The relative magnitude of these two corrections depends on $n$, $\varepsilon$, and the associated constants, but our formulation yields a simple closed-form correction that integrates naturally with the differential conformal prediction framework. Third, DPCP naturally composes with general differentially private training algorithms, including DP-ERM, DP-SGD, and DP Bayesian methods, whereas the private split conformal framework does not exploit this compositional structure. Finally, because DPCP is built on differential conformal prediction rather than split conformal prediction, it extends readily to weighted conformal prediction, distribution shift settings, and other contexts where an oracle conformal predictor can be defined. Overall, DPCP achieves full DP while retaining the coverage guarantees of the differential conformal prediction framework and offering greater flexibility than existing private split conformal methods.

\section{Theoretical Guarantees}
\label{sec:theory}

In this section, we study the theoretical properties of DPCP from three complementary perspectives. We first establish its end-to-end differential privacy guarantee, showing that the overall procedure remains private under the composition of the private training mechanism and the private quantile mechanism. We then turn to statistical validity and derive coverage guarantees for DPCP, including an exact marginal coverage result under additional regularity conditions and an idealized conditional coverage result under a released-model approximation. Finally, we analyze the efficiency of DPCP by quantifying the discrepancy between the DPCP interval and its non-private oracle counterpart. Several auxiliary results for differential CP are deferred to Appendix~A, while the present section focuses on the main theoretical properties of DPCP.

\subsection{Privacy Guarantee}
\label{subsec:privacy}

We begin with the privacy property of DPCP. As noted in Section~\ref{sec:dCP}, the dCP construction alone does not provide an end-to-end privacy guarantee, because the conformal quantile is still computed directly from the calibration scores. In contrast, DPCP privatizes both components of the procedure: the model is learned through a differentially private training mechanism, and the conformal threshold is selected through a private quantile mechanism. The following theorem formalizes the resulting privacy guarantee of the full DPCP procedure.

\begin{theorem}
\label{thm:privacy_dpcp}
Suppose the training mechanism $\mathcal A$ is $(\varepsilon_1,\delta)$-DP. Then the following statements hold.
\begin{enumerate}
\item For any input score vector, Algorithm~\ref{alg:DPQ} with privacy budget $\varepsilon_2$ is $\varepsilon_2$-DP.
\item The DPCP set $C_\alpha^{dp}(X_{n+1})$ obtained by combining $\mathcal A$ with Algorithm~\ref{alg:DPQ} satisfies $(\varepsilon_1+\varepsilon_2,\delta)$-DP.
\end{enumerate}
\end{theorem}

Theorem~\ref{thm:privacy_dpcp} shows that DPCP achieves end-to-end privacy at the procedural level. We next turn to its statistical validity and study how the private quantile step affects the coverage behavior of the resulting prediction set.

\subsection{Coverage Guarantees}
\label{sec:coverage_dpcp}

We now study the statistical validity of the fully private conformal predictor introduced in Section~\ref{subsec:Construction of DPCP Set}. In contrast to standard conformal prediction, DPCP employs a randomized private threshold selected by Algorithm~\ref{alg:DPQ}. As a result, its coverage behavior is governed not by a single fixed conformal level, but by the interaction between the random threshold selection rule and the conditional tail behavior of the new score. We first establish an exact marginal coverage guarantee under two additional regularity conditions, and then present an idealized conditional coverage result under a released-model approximation.

To formulate the exact marginal coverage result, let
\[
\mathcal G_n:=\sigma(\mathcal D_n,\hat\mu_n),\qquad
\hat\mu_n=\mathcal A(\mathcal D_n),
\]
and define
\[
\pi_j:=\Pr(\hat q=e_j\mid \mathcal G_n),\qquad
\alpha_j':=\frac{\lfloor w_j\alpha_0\rfloor}{n},
\qquad j=1,\ldots,M.
\]
Unlike standard conformal arguments, which rely on exchangeability for a fixed threshold, DPCP uses a randomized private threshold. Therefore, exact coverage is governed by the mixed tail probability
\[
\sum_{j=1}^M \pi_j\,\Pr(R_{n+1}>e_j\mid \mathcal G_n),
\]
rather than by a single fixed-level conformal event. The following two assumptions directly control this quantity.

\begin{assumption}
\label{assump:dpcp_avg_nominal}
Almost surely,
\[
\sum_{j=1}^M \pi_j \alpha_j' \le \alpha_1.
\]
\end{assumption}

\begin{assumption}
\label{assump:dpcp_pointwise_validity}
Almost surely, for every \(j\in[M]\),
\[
\Pr(R_{n+1}>e_j\mid \mathcal G_n)\le e^{\varepsilon_1}\alpha_j' + \delta.
\]
\end{assumption}

Assumption~\ref{assump:dpcp_avg_nominal} requires that the private quantile mechanism, although randomized, does not systematically favor overly aggressive thresholds. Equivalently, the nominal miscoverage level induced by the selected private threshold remains conservative on average. This is consistent with the design goal of Algorithm~\ref{alg:DPQ}, which is to randomize around the target rank without introducing a systematic downward bias.

Assumption~\ref{assump:dpcp_pointwise_validity} controls the true conditional tail probability at each candidate threshold individually. It states that, once the realized training data and the released model are fixed, every threshold \(e_j\) that may be selected by the private quantile mechanism behaves like a valid fixed-level threshold, up to the usual \((\varepsilon_1,\delta)\) adjustment inherited from private training. Since this condition is imposed only on the finite candidate set used by Algorithm~\ref{alg:DPQ}, it is a local requirement tailored to the actual randomized decision rule.

Taken together, Assumptions~\ref{assump:dpcp_avg_nominal} and~\ref{assump:dpcp_pointwise_validity} control the two main difficulties introduced by private randomized thresholding: the first limits how aggressive the selected threshold can be on average, and the second ensures that the candidate thresholds remain conditionally valid at the score-distribution level.

\begin{theorem}
\label{the:DPCP_cpverage}
Let \(C_\alpha^{dp}(X_{n+1})\) be the prediction set returned by Algorithm~\ref{alg:DPCP}. If Assumptions~\ref{assump:dpcp_avg_nominal} and~\ref{assump:dpcp_pointwise_validity} hold, then
\[
\Pr\!\left(Y_{n+1}\in C_\alpha^{dp}(X_{n+1})\right)\ge 1-\alpha.
\]
\end{theorem}

The probability in Theorem~\ref{the:DPCP_cpverage} is taken with respect to the joint randomness of the training data \(\mathcal D_n\), the new point \((X_{n+1},Y_{n+1})\), and the randomization used in the private training and private quantile mechanisms. Thus Theorem~\ref{the:DPCP_cpverage} provides a marginal coverage guarantee for DPCP.

Marginal coverage is the standard validity notion in conformal prediction, but it may be conservative once the training dataset and the released model are viewed as fixed. In many applications, after model training is completed, the realized dataset \(\mathcal D_n\) and the released model \(\hat\mu_n\) are no longer random objects from the user's perspective. It is therefore natural to ask how the DPCP interval behaves conditionally on the realized training output. In fact, under the same assumptions one can obtain a stronger statement at the training-conditional level.

\begin{corollary}
\label{cor:dpcp_conditional_coverage}
Under the assumptions of Theorem~\ref{the:DPCP_cpverage},
\[
\Pr\!\left(Y_{n+1}\in C_\alpha^{dp}(X_{n+1})\mid \mathcal G_n\right)\ge 1-\alpha
\qquad \text{a.s.}
\]
\end{corollary}

Corollary~\ref{cor:dpcp_conditional_coverage} shows that, under the two regularity conditions above, the DPCP interval attains the target level not only marginally but also conditionally on the realized training output \((\mathcal D_n,\hat\mu_n)\). Thus, in this regime, the randomized private threshold does not introduce any additional gap between conditional and marginal validity. Proofs are deferred to Appendix~\ref{B}.

\subsection{Efficiency Analysis}
\label{sec:efficiency_dpcp}

We now study the efficiency of DPCP by quantifying how close the resulting prediction interval is to its non-private oracle counterpart. Our goal is to compare the fully private interval \(C_\alpha^{dp}(X_{n+1})\) with the oracle conformal interval \(C_\alpha^o(X_{n+1})\), and to identify how much error is introduced by private model fitting and private quantile selection.

We begin with a standard empirical-risk-minimization (ERM) setting. Let
\[
J(\vartheta;\mathcal D_n)
=
\frac{1}{n}\sum_{i=1}^n \ell\!\left(Y_i,\mu(X_i;\vartheta)\right)+\Omega(\vartheta),
\]
where, for every \((x,y)\in\mathcal X\times\mathcal Y\), the map
$
\vartheta\mapsto \ell\!\left(y,\mu(x;\vartheta)\right)$
is convex and \(\rho L\)-Lipschitz, and \(\Omega\) is \(\lambda\)-strongly convex. Define
\[
\hat{\vartheta}^{(n)}=\arg\min_{\vartheta}J(\vartheta;\mathcal D_n).
\]
The following lemma shows that the ERM solution is stable under the addition of a single observation.

\begin{lemma}
\label{lem:ERM_l2}
Assume that \(J(\vartheta;\mathcal D_n)\) is \(\lambda\)-strongly convex and that, for every \((x,y)\), the function \(\vartheta\mapsto \ell\!\left(y,\mu(x;\vartheta)\right)\) is \(\rho L\)-Lipschitz. For any dataset \(\mathcal D_n\) and its augmented version \(\mathcal D_{n+1}=\mathcal D_n\cup\{(X_{n+1},Y_{n+1})\}\), define
\[
J(\vartheta;\mathcal D_{n+1})
=
J(\vartheta;\mathcal D_n)
+
\frac{1}{n}\ell\!\left(Y_{n+1},\mu(X_{n+1};\vartheta)\right),
\]
and
\[
\hat{\vartheta}^{(n+1)}=\arg\min_{\vartheta}J(\vartheta;\mathcal D_{n+1}).
\]
Then
\[
\sup_{\mathcal D_n,\mathcal D_{n+1}}
\left\|
\hat{\vartheta}^{(n+1)}-\hat{\vartheta}^{(n)}
\right\|_2
\le
\frac{2\rho L}{\lambda n}
=: \tau.
\]
\end{lemma}

Lemma~\ref{lem:ERM_l2} is analogous to Lemma 8 in \citet{chaudhuri2009differentially} and provides the \(\ell_2\)-sensitivity bound needed for the Gaussian mechanism. Let
\[
\widetilde{\vartheta}^{(n)}=\hat{\vartheta}^{(n)}+\zeta,
\qquad
\zeta\sim \mathcal N(0,\sigma_1^2 I_d),
\qquad
\sigma_1^2
=
2\log\!\left(\frac{1.25}{\delta}\right)\frac{\tau^2}{\varepsilon_1^2}.
\]
Then \(\widetilde{\vartheta}^{(n)}\) is an \((\varepsilon_1,\delta)\)-DP release of the ERM estimator.

Throughout this subsection, we use the absolute residual score
$
R\!\left((x,y),\mu(\cdot;\vartheta)\right):=|y-\mu(x;\vartheta)|$.
Define
$
R_i^{(n)}
=
R\!\left(Z_i,\mu(\cdot;\widetilde{\vartheta}^{(n)})\right)$, $i=1,\ldots,n$,
and
$
R_i^{(n+1)}
=
R\!\left(Z_i,\mu(\cdot;\hat{\vartheta}^{(n+1)})\right)$,
$i=1,\ldots,n+1$.
Let
$
\alpha_1=e^{-\varepsilon_1}(\alpha-\delta)$, $k=\left\lceil(1-\alpha)(n+1)\right\rceil$, $k'=\left\lceil(1-\alpha_1)(n+1)\right\rceil$.

The differential CP interval based on the privatized ERM estimator is
\[
C_\alpha^d(X_{n+1})
=
\left[
\mu(X_{n+1};\widetilde{\vartheta}^{(n)})-R_{(k')}^{(n)},
\;
\mu(X_{n+1};\widetilde{\vartheta}^{(n)})+R_{(k')}^{(n)}
\right],
\]
where \(R_{(k')}^{(n)}\) is the \(k'\)-th order statistic of \(\{R_i^{(n)}\}_{i=1}^n\).

The fully private DPCP interval is
\[
C_\alpha^{dp}(X_{n+1})
=
\left[
\mu(X_{n+1};\widetilde{\vartheta}^{(n)})-\hat q,
\;
\mu(X_{n+1};\widetilde{\vartheta}^{(n)})+\hat q
\right],
\]
where \(\hat q\) is produced by Algorithm~\ref{alg:DPQ} with input level \(\beta=\alpha_1\) and privacy budget \(\varepsilon_2\).

As a benchmark, we define the oracle conformal interval by
\[
C_\alpha^o(X_{n+1})
=
\left[
\mu(X_{n+1};\hat{\vartheta}^{(n+1)})-R_{(k)}^{(n+1)},
\;
\mu(X_{n+1};\hat{\vartheta}^{(n+1)})+R_{(k)}^{(n+1)}
\right],
\]
where \(R_{(k)}^{(n+1)}\) is the \(k\)-th order statistic of \(\{R_i^{(n+1)}\}_{i=1}^{n+1}\).

To compare \(C_\alpha^{dp}(X_{n+1})\) with \(C_\alpha^o(X_{n+1})\), we assume that the two score sequences admit the same idealized i.i.d. approximation.

\begin{assumption}
\label{assump}
There exists a distribution function \(F\) such that the idealized i.i.d. counterparts of both \(\{R_i^{(n)}\}_{i=1}^n\) and \(\{R_i^{(n+1)}\}_{i=1}^{n+1}\) are sampled from \(F\). Moreover, for the given \(\alpha\in(0,1)\), there exists \(\epsilon_0>0\) such that \(F^{-1}\) is \(\gamma\)-H\"older continuous on
\[
[1-\alpha-\epsilon_0,\;1-\alpha+\epsilon_0]
\]
with constant \(h\), that is,
\[
|F^{-1}(q_1)-F^{-1}(q_2)|
\le
h|q_1-q_2|^\gamma,
\qquad
\forall q_1,q_2\in[1-\alpha-\epsilon_0,\;1-\alpha+\epsilon_0].
\]
\end{assumption}

\begin{definition}
\label{def:approx_iid}
A random sequence \(R_1,\ldots,R_n\) with joint distribution \(P\) is said to be \(\xi_{TV}\)-approximately i.i.d. if there exists an i.i.d. sequence \(R_1',\ldots,R_n'\) with joint distribution \(Q\) such that
\[
TV(P,Q)
=
\sup_{A\in\mathcal B(\Omega^n)} |P(A)-Q(A)|
\le \xi_{TV}.
\]
\end{definition}

The \(\xi_{TV}\)-approximate i.i.d. condition allows us to compare the actual score sequences with an idealized reference model under which order-statistic gaps can be analyzed explicitly. Combined with Assumption~\ref{assump}, it yields quantitative control of the gap between the oracle and private intervals through the local regularity of the quantile map \(F^{-1}\).

To measure the discrepancy between two prediction sets \(C_1\) and \(C_2\), we use the Lebesgue measure of their symmetric difference,
\[
L(C_1\triangle C_2),
\qquad
C_1\triangle C_2:=(C_1\cup C_2)\setminus(C_1\cap C_2).
\]

The following theorem shows that, under a balanced privacy split and mild regularity conditions, the DPCP interval approaches the oracle interval in probability.

\begin{theorem}
\label{the:4.6}
Assume the conditions of Lemma~\ref{lem:ERM_l2} and \(\delta=O(n^{-1})\). Let the score be the absolute residual, assume that the parameter dimension \(d\) is fixed, and suppose that Algorithm~\ref{alg:DPQ} is implemented with a rank-based grid with \(M=O(n)\) candidate thresholds. If Assumption~\ref{assump} holds, the score sequences \(\{R_i^{(n)}\}_{i=1}^n\) and \(\{R_i^{(n+1)}\}_{i=1}^{n+1}\) are both \(\xi_{TV}\)-approximately i.i.d. with the same idealized distribution function \(F\), and the privacy budget is split evenly as
$
\varepsilon_1=\varepsilon_2=\varepsilon/2, \varepsilon=\Theta(n^{-\eta})
$
for some \(\eta\in(0,1)\), then for all \(t>0\),
\begin{align*}
    \Pr\!\left(
L\!\left(
C_\alpha^o(X)\triangle C_\alpha^{dp}(X)
\right)
>
O\!\left(
\left(
\frac{1}{n^\eta}
+
\frac{\log n}{n^{1-\eta}}
\right)^\gamma
+t
\right)
\right)
\le&
O\!\left(
\frac{1}{n^{1-\eta}}
+
\frac{1}{n^\eta\log n}
+
\exp\!\left(
-\frac{n^{2-2\eta}t^2}{\log n}
\right)
\right)\\
&+
3\xi_{TV}.
\end{align*}
\end{theorem}

The three terms in Theorem~\ref{the:4.6} admit clear interpretations. The exponential tail term comes from the Gaussian perturbation used in private model fitting. The term \(O(n^{\eta-1})=O(1/n^{1-\eta})\) reflects the discrepancy between the \(n\)-sample privatized estimator and the \((n+1)\)-sample oracle estimator. The term \(O(1/(n^\eta\log n))\) arises from the private quantile selection step in Algorithm~\ref{alg:DPQ}. Thus, under a balanced privacy allocation and an appropriate decay rate of \(\varepsilon\), the DPCP interval converges in probability to the oracle interval.

Compared with the private split conformal method of \citet{angelopoulos2022private}, DPCP uses the full dataset for both model fitting and calibration, and the private quantile mechanism introduces an error of order \(O(\log n/(n\varepsilon_2))\) in rank scale. Under the balanced split \(\varepsilon_2=\varepsilon/2\), this becomes \(O(\log n/(n\varepsilon))\). This dependence is compatible with the full-sample construction and leads to a sharper oracle approximation than split-based procedures under the same overall privacy budget.

\section{Numerical Study}
\label{sec:Numerical Study}

This section presents numerical experiments designed to evaluate the finite-sample
performance of the proposed DPCP method.
We consider both synthetic in Section \ref{synthetic} and real-world datasets in Section \ref{subsec:Real data}, with the goals of validating theoretical coverage guarantees, examining the impact of sample size, privacy budget, and miscoverage level on interval length and stability, and lastly comparing DPCP with existing private conformal methods. 

Throughout, we compare DPCP with the private split conformal prediction method of \citet{angelopoulos2022private}, which we refer to as PSCP. Since both model training and prediction set construction consume the privacy budget, careful allocation of data and privacy resources is required. Assuming a total sample size $n$ and an overall privacy requirement of $(\varepsilon,\delta)$ for the released prediction set, we adopt the allocation strategy summarized in Table~\ref{tab:tabone}. In particular, DPCP uses the full dataset for both model fitting and calibration, while PSCP splits the data evenly between training and
calibration. In both cases, the privacy budget is divided equally between the training
procedure and the quantile release step. Additional implementation details are deferred to Appendix D.

\begin{table}[h]
    \caption{Allocation of privacy budget and sample size.}
\begin{center}
\begin{tabular}{ccccc}
    \hline
    \quad & \multicolumn{2}{c}{Training} & \multicolumn{2}{c}{Calibration} \\
    \cline{2-5}
    & Privacy & Data size & Privacy & Data size \\
    \hline
    DPCP & $(\varepsilon/2,\delta)$ & $n$ & $(\varepsilon/2,0)$ & $n$ \\
    \hline
    PSCP & $(\varepsilon/2,\delta)$ & $n/2$ & $(\varepsilon/2,0)$ & $n/2$ \\
    \hline
\end{tabular}
\end{center}
    \label{tab:tabone}
\end{table}

For readability, we use several abbreviations in the figure legends throughout this section. 
DPQ stands for the differentially private quantile mechanism in Algorithm~1. 
Accordingly, $\alpha_0$-DPQ denotes Algorithm~1 run at the corrected input level used by our method, while $\alpha_0'$-DPQ denotes the same DPQ mechanism run at the private split conformal prediction method of \citet{angelopoulos2022private} for comparison. 
In addition, dCP and sCP are used in the figure legends as shorthand for differential CP and split CP, respectively. 
Likewise, p-dCP and p-sCP denote the corresponding privately calibrated versions shown in the classification experiments. 
These abbreviations are introduced only to keep the figure legends concise; throughout the main text, we use the full method names whenever possible.

\subsection{Synthetic Data} \label{synthetic}
We begin with a synthetic experiment that allows us to isolate the effects of sample size, privacy budget, and miscoverage level. The data generating process is given by
\begin{align*}
Y = X + b + \epsilon, \quad X \sim \mathcal{N}(\mu_{X}, \sigma_{X}^2), \quad \epsilon \sim \mathcal{T}(\mu_{\epsilon}, \sigma_{\epsilon}^2; -3\sigma_{\epsilon}, 3\sigma_{\epsilon}),
\end{align*}
where $\mathcal{T}$ denotes a truncated normal distribution. Throughout the experiments, we fix $b=5$, $\mu_X=\mu_\epsilon=0$, $\sigma_X=10$, and $\sigma_\epsilon=5$. The underlying regression model estimates the location parameter $b$
using the empirical mean of $Y - X$, which is the natural estimator under this additive noise model. To enforce differential privacy, we apply the Laplace mechanism to this estimator. Specifically, the private estimator of $b$ is obtained by perturbing the sample mean of $Y - X$ with Laplace noise calibrated to the sensitivity of the statistic, yielding
\begin{align*}
\hat{b}
= \frac{1}{n}\sum_{i=1}^n(Y_i - X_i) + \mathrm{Lap}\!\left(0, \frac{6\sigma_\epsilon}{n\varepsilon}\right).
\end{align*}
The resulting private prediction model then takes the form $Y = X + \hat{b}.$

\begin{figure}[h]
    \centering
    \includegraphics[width=0.8\textwidth]{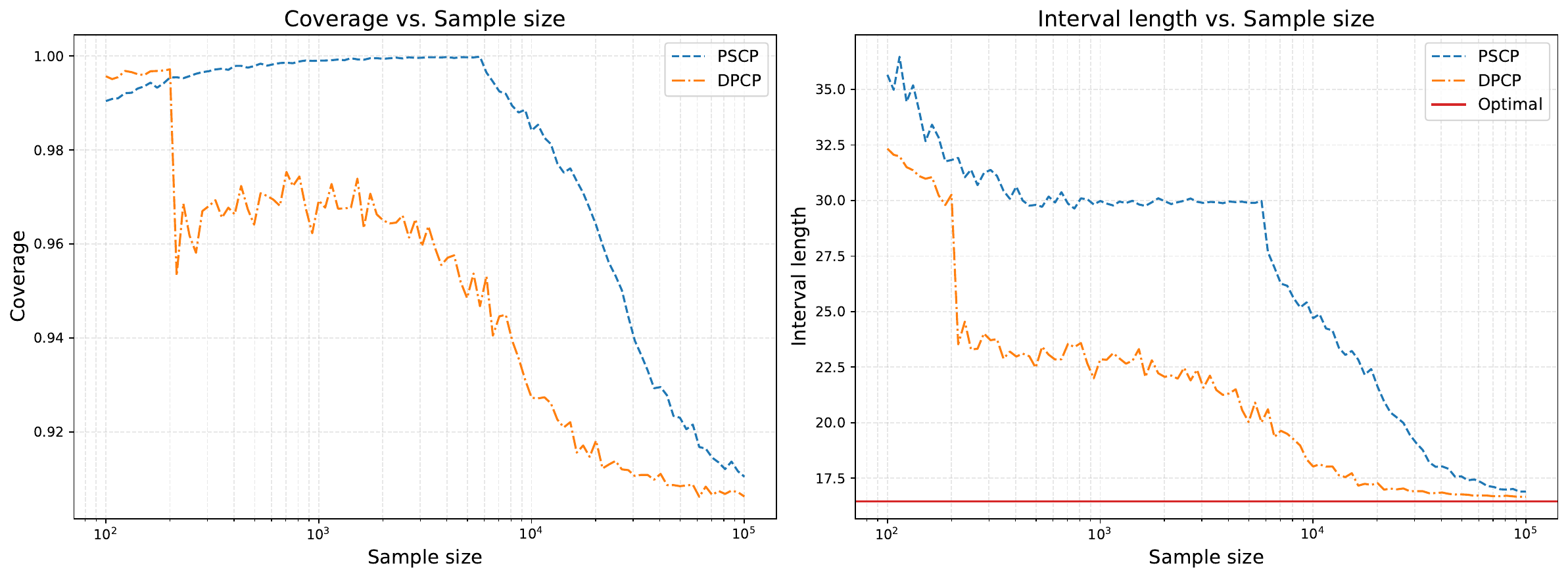}
    \caption{Coverage (left) and length (right) vary with sample size $n$. Set $\alpha = 0.1$, $\varepsilon = 0.1$}
    \label{fig:synthetic_sample_size}
\end{figure}

\begin{figure}[h]
    \centering
    \includegraphics[width=0.8\textwidth]{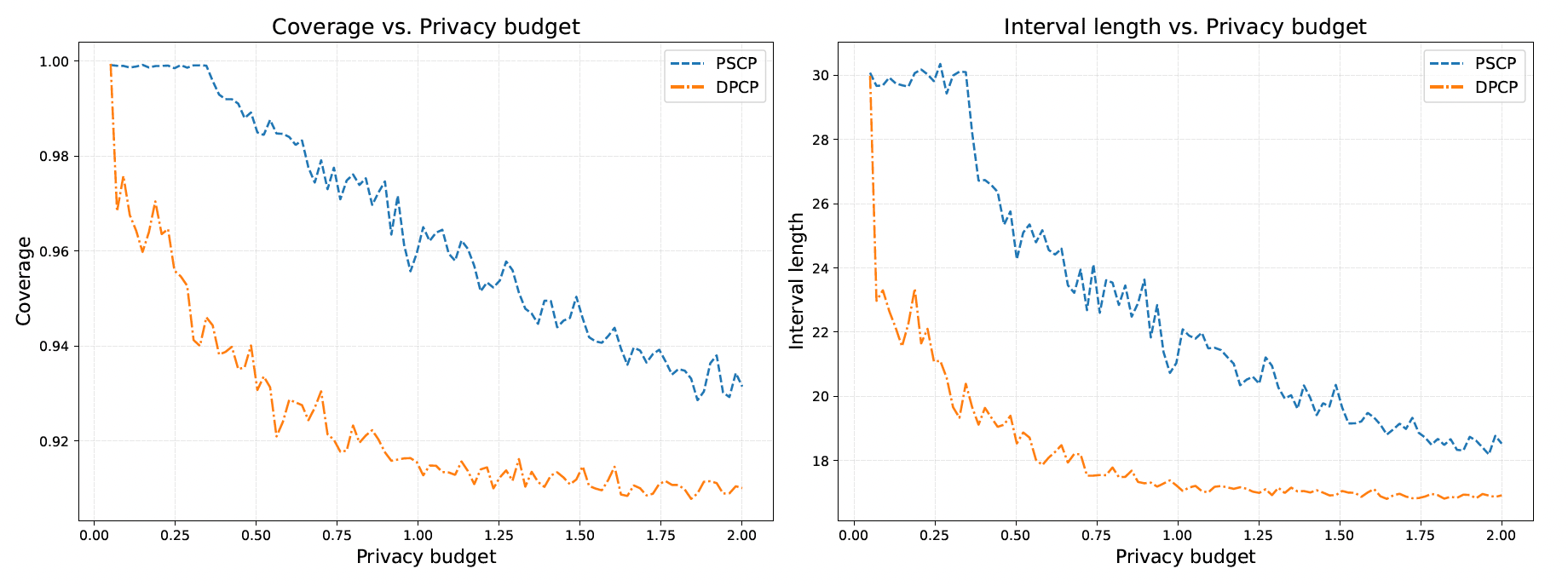}
    \caption{Coverage (left) and length (right) vary with privacy budget $\varepsilon$. The privacy budget is split as $\varepsilon_1 = 0.05$ for model training and $\varepsilon_2 = \varepsilon - \varepsilon_1$ for quantile estimation, with $\alpha = 0.1$ and sample size $n = 2000$.}
    \label{fig:synthetic_privacy_budget}
\end{figure}

\begin{figure}[h]
    \centering
    \includegraphics[width=0.8\textwidth]{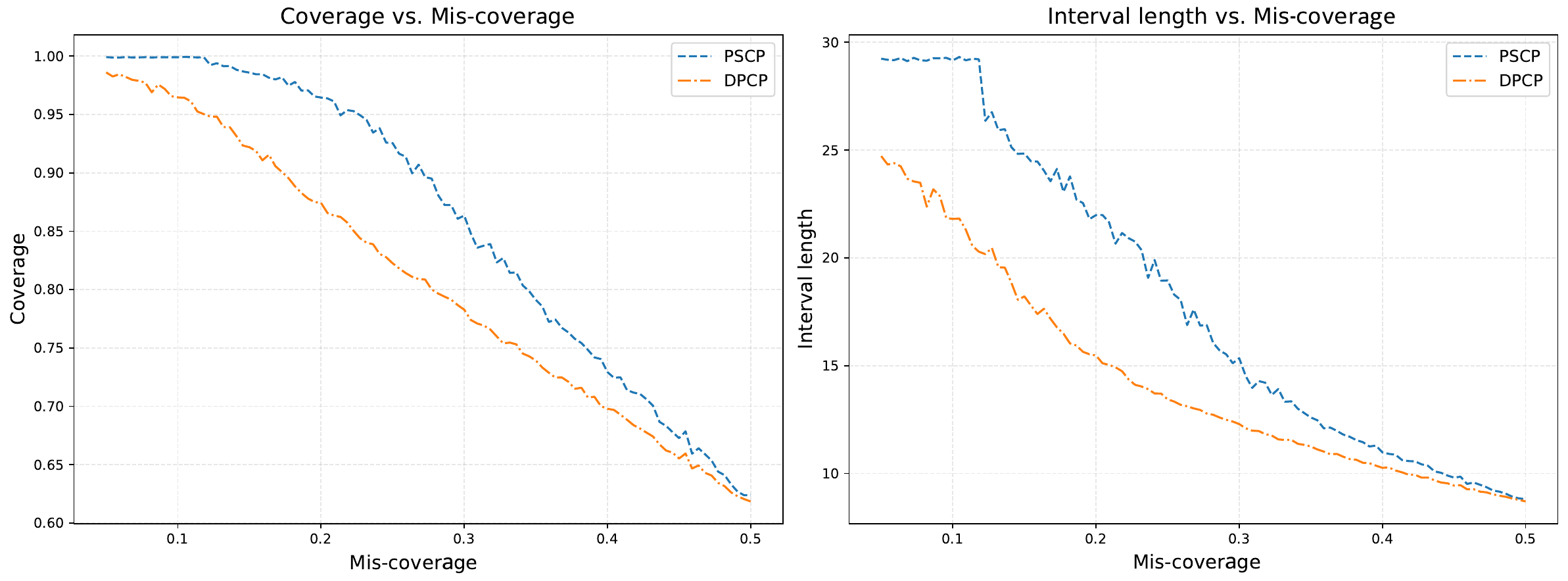}
    \caption{Coverage (left) and length (right)  vary with mis-coverage level $\alpha$. Set $\varepsilon = 0.1$, $n = 2000$}
    \label{fig:synthetic_min-covearge}
\end{figure}

Figures~\ref{fig:synthetic_sample_size}, \ref{fig:synthetic_privacy_budget}, and \ref{fig:synthetic_min-covearge} collectively illustrate the finite-sample performance of DPCP and its robustness on sample size, privacy budget, and target miscoverage level. Figure~\ref{fig:synthetic_sample_size} examines convergence as the sample size increases. As predicted by the theoretical analysis in Section~\ref{sec:theory}, DPCP achieves valid coverage while exhibiting substantially shorter prediction intervals than PSCP. In small samples, the differentially private quantile mechanisms in both DPCP and PSCP tend to be conservative, often selecting near-maximum nonconformity scores, which yields coverage close to one and produces visible jumps in interval length. As the sample size grows, these effects diminish rapidly and the DPCP intervals converge smoothly toward the oracle benchmark. In particular, this convergence occurs with fewer samples than required by PSCP, reflecting the improved data efficiency of using the full dataset for both training and calibration.

Figure~\ref{fig:synthetic_privacy_budget} examines the impact of the privacy budget $\varepsilon$ on both coverage and interval length. Across the full range of $\varepsilon$ values considered, DPCP consistently attains coverage close to the nominal level while producing substantially shorter prediction intervals than PSCP. As the privacy budget increases, the intervals produced by both methods shrink gradually, reflecting reduced privacy noise; however, DPCP exhibits a markedly faster stabilization in interval length, highlighting its improved statistical efficiency. Additional refinements to the quantile adjustment that further mitigate conservativeness at larger privacy budgets are discussed in Appendix \ref{C}. Finally, Figure~\ref{fig:synthetic_min-covearge} reports performance across a range of target miscoverage levels $\alpha$. The results confirm that DPCP adapts smoothly to changes in $\alpha$, preserving the expected trade-off between coverage and interval length. Across all levels considered, DPCP consistently achieves coverage close to the target while maintaining shorter intervals than PSCP, in agreement with the theoretical guarantees established in Section~\ref{sec:theory}.

Taken together, these experiments demonstrate that DPCP not only preserves privacy and coverage guarantees but also yields substantial efficiency gains in finite samples, particularly in regimes with limited data or strong privacy constraints. Having established the finite-sample behavior of DPCP in this synthetic setting, we next evaluate its practical performance on real datasets. These experiments assess whether the efficiency gains observed above persist under realistic data distributions, model misspecification, and higher-dimensional covariates.

\subsection{Real Data}
\label{subsec:Real data}
In the following experiment, we set miscoverage $\alpha = 0.1$ and total privacy budget $\varepsilon = 0.1$. We employed least squares regression, ridge regression, and lasso regression as the underlying algorithms and trained the private models using the Opacus library. We conducted numerical experiments on 9 standard datasets, with the complete results presented in Appendix \ref{D}. Additionally, Appendix \ref{D} includes the experimental results of classification tasks using deep learning models. We used 50\% of the data as the test set to calculate coverage, while the remaining 50\% was used for model training and prediction set construction. Each experiment was repeated 100 times.

\begin{figure}
    \centering
    \includegraphics[width=0.8\textwidth]{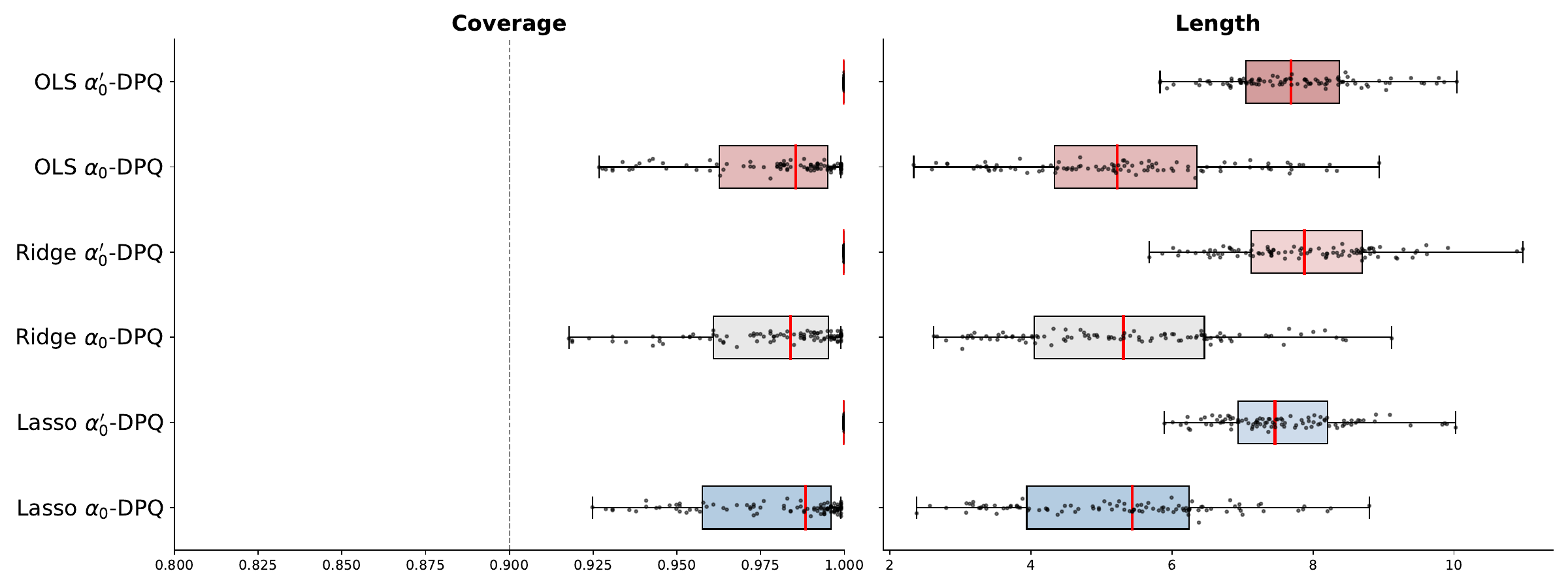}
    \includegraphics[width=0.8\textwidth]{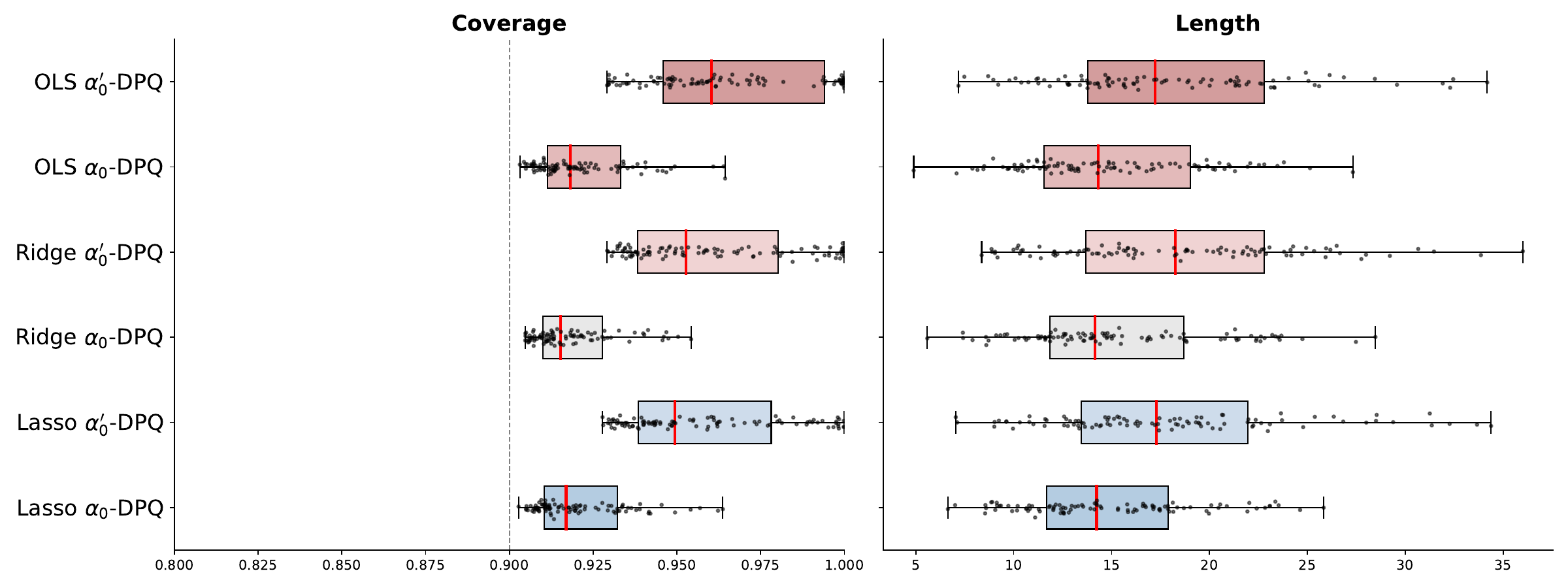}
    \caption{Coverage and length of private prediction intervals on the communities and crime dataset (upper, sample size = 1994) and the power consumption dataset (lower, sample size = 52416).}
    \label{fig:DPQ_two_dataset}
\end{figure}

Figure \ref{fig:DPQ_two_dataset} shows the performance of the private prediction sets obtained using the $\alpha_0$-DPQ and $\alpha^\prime_0$-DPQ methods. We aimed to achieve a private prediction set with a coverage level of 0.9, and our method consistently comes closer to this target. In cases with relatively small sample sizes, the $\alpha^\prime_0$-DPQ method almost always outputs the maximum nonconformity score, making private prediction sets overly conservative and nearly meaningless. In contrast, our method maintains a certain level of effectiveness under the same conditions. When the sample size is relatively large, the privacy prediction sets obtained by our method also outperform those produced by the $\alpha^\prime_0$-DPQ method. This is because our method outputs the $0.9 + 2/n\epsilon$ privacy quantile, whereas the $\alpha^\prime_0$-DPQ method always selects a greater privacy quantile.

\begin{figure}
    \centering
    \includegraphics[width=0.8\textwidth]{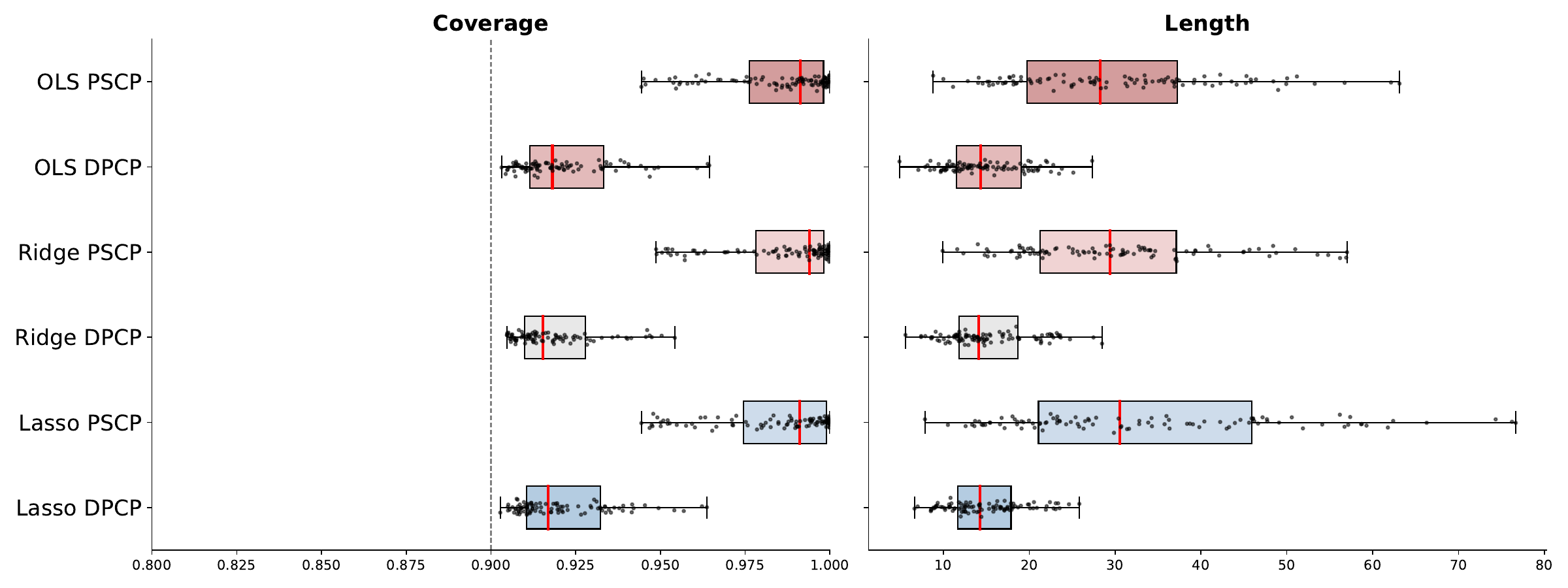}
    \caption{Coverage (left) and length (right) of private prediction intervals on the power consumption dataset.}
    \label{fig:DPCP_power}
\end{figure}

Figure \ref{fig:DPCP_power} compares the DPCP intervals and PSCP intervals. It is evident that when privacy protection is strong, the private prediction sets obtained from our strategy significantly outperform existing methods and their distribution is more concentrated. This is because the DPCP method exhibits higher statistical efficiency during both the model training phase and the prediction set construction phase, introducing less privacy noise and avoiding the randomness associated with data splitting in the split CP method.

\begin{figure}
    \centering
    \includegraphics[width=0.8\textwidth]{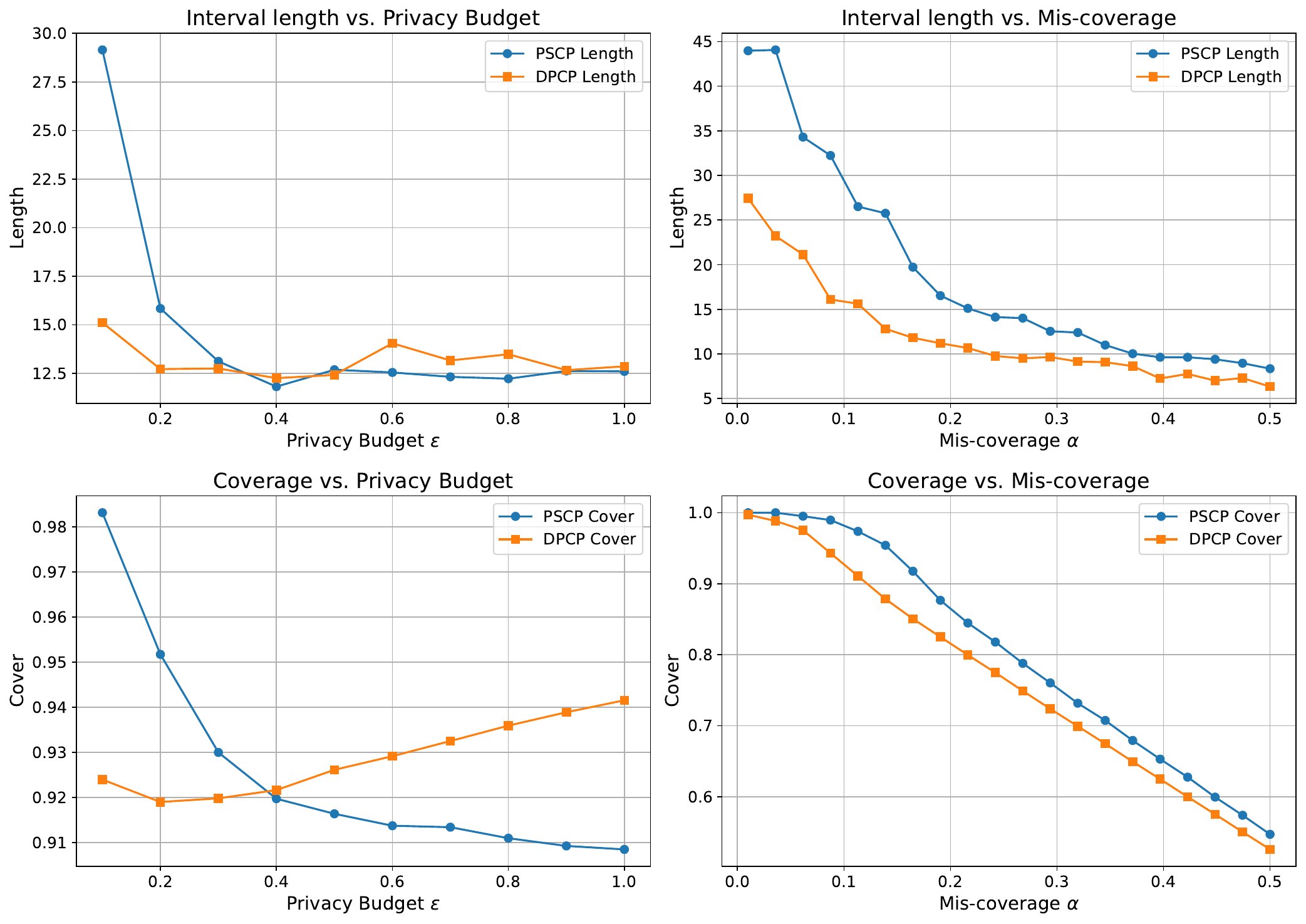}
    \caption{Private prediction intervals on the power consumption dataset. The left side shows the interval performance as the privacy budget changes with $\alpha = 0.1$, and the right side shows the interval performance as the mis-coverage changes with $\varepsilon = 0.1$.
}
    \label{fig:Real_power_vary}
\end{figure}

\section{Discussion}
\label{sec:Discussion}

This paper introduces DPCP, a new conformal inference framework that leverages the stability of differentially private learning algorithms to construct prediction sets from a single model fit. By integrating conformal calibration directly with DP, DPCP avoids the sample-splitting inefficiency of private split conformal methods and makes fuller use of the available data. As a result, it can yield more efficient prediction sets while preserving end-to-end privacy. Our theoretical analysis establishes privacy, coverage, and efficiency guarantees for DPCP, and the empirical results show that DPCP typically produces shorter prediction intervals than private split CP under comparable privacy budgets, especially in the strong-privacy regime.

The proposed framework is also flexible. It is not tied to a specific model class or learning algorithm. Its key requirement is the availability of an oracle conformal procedure together with a suitable oracle-to-private comparison. Appendix~\ref{A} develops such a comparison for differential CP. In this sense, the DPCP construction may be adapted to other conformal settings beyond the exchangeable case, including weighted conformal procedures designed for covariate shift or other forms of distributional heterogeneity \citep{tibshirani2019conformal, barber2023conformal}. More broadly, if the corresponding oracle procedure enjoys stronger validity properties, then one may expect analogous private versions to inherit approximate counterparts of these guarantees, up to the distortion introduced by the privacy mechanism.

Several directions remain for future work. First, our exact coverage analysis for DPCP relies on additional regularity conditions linking the randomized private threshold to the conditional behavior of the score distribution. It would be valuable to identify more primitive or verifiable conditions under which these assumptions hold, or to replace them with weaker assumptions that still yield useful finite-sample guarantees. Second, although our efficiency analysis shows that DPCP approaches the oracle conformal predictor under suitable asymptotic conditions, sharper nonasymptotic bounds may be possible. Finally, while our analysis is presented under the standard \((\varepsilon,\delta)\)-DP framework, the methodology is not restricted to this setting. Alternative notions such as concentrated DP \citep{bun2016concentrated}, Rényi DP \citep{mironov2017renyi}, and Gaussian DP \citep{dong2022gaussian} may lead to sharper or more interpretable oracle-to-private comparisons, and exploring these extensions is a natural topic for future research.

\newpage

\appendix

\section{Theoretical Guarantees for Differential CP}
\label{A}

This appendix collects the main theoretical results for differential CP. The goal is to provide a self-contained presentation of the formal guarantees supporting the methodology in the main text. For clarity, Appendix~A is divided into two parts: we first state the main theorems, and then provide their proofs.

\subsection{Theorems for Differential CP}
\label{A.1}
The key idea is to treat differential CP as a privacy-stability-based approximation to an oracle conformal predictor built on the augmented dataset. This perspective allows us to separate the analysis into three parts. We first compare the differential CP predictor with its oracle counterpart and quantify how differential privacy perturbs the corresponding miscoverage event. We then translate this comparison into finite-sample coverage guarantees for differential CP. After that, we study a training-conditional validity result under a released-model approximation, which clarifies how the differential CP predictor behaves once the private model has been fixed. Finally, under an empirical risk minimization framework, we quantify the efficiency gap between differential CP and the oracle conformal predictor.

The usefulness of the differential CP developed in Section~\ref{sec:dCP} becomes apparent when we relate its prediction set $C_{\alpha}^{d}(X_{n+1})$ to the oracle conformal predictor $C_{\alpha}^{o}(X_{n+1})$, namely, the infeasible conformal predictor obtained from the augmented dataset
$
\mathcal{D}_{n+1}=\mathcal{D}_n\cup\{(X_{n+1},Y_{n+1})\}.
$
Although this oracle predictor is infeasible in practice, since its construction depends on the unknown test response $Y_{n+1}$, it serves as a useful benchmark: it is the natural full-data conformal object that differential CP attempts to approximate through the stability of a differentially private mechanism. The next theorem shows that the miscoverage probability of differential CP can be sandwiched by that of oracle conformal predictors at slightly perturbed nominal levels, with the multiplicative and additive distortions governed by the DP parameters.

\begin{theorem}
\label{the:dcp_ocp}
Let $\mathcal{A}$ be an $(\varepsilon,\delta)$-DP mechanism and let $\alpha\in(0,1)$.
Define
\[
\alpha' = e^{-\varepsilon}(\alpha-\delta),\qquad
\alpha_{+} = \alpha' + \frac{1}{n+1},\qquad
\alpha_{-} = \alpha' - \frac{1}{n+1}.
\]
Then, for any dataset $\mathcal{D}_n$ and any new point $(X_{n+1},Y_{n+1})$, the differential CP set defined in \eqref{eq:dCP} satisfies
\begin{align*}
\Pr\!\left(Y_{n+1}\notin C_{\alpha}^{d}(X_{n+1}) \mid \mathcal{D}_{n+1}\right)
&\le e^{\varepsilon}
     \Pr\!\left(Y_{n+1}\notin C_{\alpha_{+}}^{o}(X_{n+1})
     \mid \mathcal{D}_{n+1}\right)
     +\delta,
\end{align*}
whenever $\alpha_{+}\in(0,1)$, and
\begin{align*}
\Pr\!\left(Y_{n+1}\notin C_{\alpha}^{d}(X_{n+1}) \mid \mathcal{D}_{n+1}\right)
&\ge e^{-\varepsilon}
     \Bigl(
       \Pr\!\left(Y_{n+1}\notin
               C_{\alpha_{-}}^{o}(X_{n+1})
               \mid \mathcal{D}_{n+1}\right)
       -\delta
     \Bigr),
\end{align*}
whenever $\alpha_{-}\in(0,1)$.
\end{theorem}

Theorem~\ref{the:dcp_ocp} should be viewed as the basic comparison principle for differential CP. It shows that, conditional on the augmented dataset, the miscoverage event of differential CP is controlled by the corresponding oracle miscoverage event up to the privacy-induced distortion. Once this oracle comparison is established, the finite-sample validity of differential CP follows by combining it with the standard validity property of the oracle conformal predictor under exchangeability.

\begin{theorem}
\label{thm:dcp_coverage}
Let $\mathcal{A}$ be an $(\varepsilon,\delta)$-DP mechanism and let $R$ be a nonconformity score. If the scores $\{R(Z_i,\mathcal{A}(\mathcal{D}_{n+1}))\}_{i=1}^{n+1}$ are exchangeable, then
\[
\Pr\!\left(Y_{n+1}\in C_{\alpha}^{d}(X_{n+1})\right)
\ge 1-\alpha-\frac{e^\varepsilon}{n+1}.
\]
Furthermore, if these scores are almost surely distinct, then
\[
\Pr\!\left(Y_{n+1}\in C_{\alpha}^{d}(X_{n+1})\right)
\le
1 - e^{-2\varepsilon}\alpha
+ e^{-\varepsilon}\!\left[\delta(1+e^{-\varepsilon}) + \frac{2}{n+1}\right].
\]
\end{theorem}

Theorem~\ref{thm:dcp_coverage} establishes the finite-sample marginal validity of differential CP and shows explicitly how the privacy parameters affect the attainable coverage level. This result isolates the statistical effect of replacing the oracle full-data conformal construction with the privacy-stable differential CP rule. To move from differential CP to a fully private conformal predictor, one must additionally privatize the threshold selection step. The next lemma provides a coverage control statement for the differentially private quantile mechanism, which later serves as a key ingredient in the analysis of DPCP.

\begin{lemma}
\label{lem:dpq_coverage}
Let $\beta\in(0,1/2]$ be the input level to Algorithm~\ref{alg:DPQ}, let $N=|\mathcal D_{cal}|$, and define
\[
\alpha_0=\beta-\frac{2}{N\varepsilon}>0.
\]
Assume that the nonconformity scores $R_1,\ldots,R_N$ are exchangeable and that
\[
\sum_{j=1}^M \frac{w_j}{N}\Pr(\hat q=e_j)
\le
1+\frac{2}{N\varepsilon\alpha_0}.
\]
Then, for any $i\in[N]$,
\[
\Pr(R_i\le \hat q)\ge 1-\beta.
\]
If, in addition, $\Pr(R_i=e_j)=0$ for all $i\in[N]$ and $j\in[M]$, then
\[
\Pr(R_i\le \hat q)\le 1-\beta+\frac{2}{N\varepsilon\beta-2}.
\]
\end{lemma}
\begin{remark}
\label{rem:weighted_condition}
The weighted-average condition in Lemma~\ref{lem:dpq_coverage} is a technical sufficient condition controlling the average penalty of the threshold selected by the exponential mechanism. Intuitively, it requires that the private quantile release place most of its probability mass on candidate thresholds with near-optimal penalties, so that the randomized threshold $\hat q$ does not deviate too far from the target quantile level on average. A stronger but more explicit sufficient scenario is that the ordered penalties grow one-by-one from their minimum value, i.e., $
w_{(m)} = N+m-1, m=1,\ldots,M$. Such a structure can be enforced, for example, by constructing the candidate thresholds through a rank-based grid aligned with the order statistics of the calibration scores.
\end{remark}

A fully private conformal procedure requires not only a private training mechanism, but also a private release of the calibration threshold. Since the threshold is now randomized through the exponential mechanism, its contribution to validity must be analyzed separately. The following lemma shows that, under a mild weighted-average condition, the private quantile still attains the target coverage level up to an explicit finite-sample correction.

The marginal guarantee above is the standard notion of conformal validity, but it does not fully describe the behavior of differential CP once the private model has already been released. In practice, after training is completed, users often regard the realized dataset and the released model as fixed, and the remaining uncertainty comes only from the future test point. This motivates a training-conditional analysis. The next theorem studies differential CP under an idealized released-model approximation, where the score sequence is conditionally i.i.d. given the released model, and shows that the resulting conditional miscoverage remains close to the target level with high probability.

\begin{theorem}
\label{thm:dcp_conditional_coverage_G}
Let $\hat\mu_n=\mathcal A(\mathcal D_n)$ and define
\[
\mathcal G_n:=\sigma(\mathcal D_n,\hat\mu_n),\qquad
R_i=R(Z_i,\hat\mu_n),\quad i=1,\ldots,n+1.
\]
Let $\alpha' = e^{-\varepsilon}(\alpha-\delta)$. Assume that, conditional on the released model $\hat\mu_n$, the score sequence $\{R_i\}_{i=1}^{n+1}$ is i.i.d. with common distribution function $F_{\hat\mu_n}$. Then, for any $\xi\in(0,1/2]$,
\[
\Pr\!\left(
\Pr\!\left(
Y_{n+1}\notin C_\alpha^{d}(X_{n+1})
\,\middle|\, \mathcal G_n
\right)
\le
\alpha+\sqrt{\frac{\log(2/\xi)}{2n}}
\right)
\ge
1-\xi.
\]
\end{theorem}

Validity alone does not characterize the practical quality of a conformal predictor; one also needs to understand how far the resulting prediction set is from the oracle benchmark in terms of efficiency. We therefore next study differential CP under a regularized empirical risk minimization framework and quantify the discrepancy between the differential CP interval and the oracle conformal interval. The goal is to make precise how the privacy perturbation in model fitting propagates to the final prediction set.

\begin{theorem}
\label{the:ERM_dCP}
Assume the conditions of Lemma~\ref{lem:ERM_l2} and $\delta=O(n^{-1})$. Let $
R\!\left((x,y),\mu(\cdot;\vartheta)\right)=|y-\mu(x;\vartheta)|$.
Assume further that the parameter dimension $d$ is fixed, that the score sequences $\{R_i^{(n)}\}_{i=1}^n$ and $\{R_i^{(n+1)}\}_{i=1}^{n+1}$ are both $\xi_{TV}$-approximately i.i.d. with the same idealized distribution function $F$. If $\varepsilon_1\to 0$ and $\varepsilon_1 n\to\infty$, then for all $t>0$,
\[
\Pr\!\left(
L\!\left(
C_\alpha^o(X)\triangle C_\alpha^d(X)
\right)
>
O(\varepsilon_1^\gamma+t)
\right)
\le
O\!\left(
\exp\!\left(
-\frac{\varepsilon_1^2 n^2 t^2}{\log n}
\right)
+
\frac{1}{\varepsilon_1 n}
\right)
+
2\xi_{TV}.
\]
\end{theorem}

Taken together, the results in this appendix clarify the role of differential CP as an intermediate bridge between oracle conformal prediction and fully private conformal inference. The oracle comparison result explains why privacy stability can be converted into coverage control, the conditional result shows how validity behaves after model release, and the ERM analysis quantifies the efficiency loss relative to the oracle predictor. These results also provide the conceptual basis for the fully private DPCP procedure studied in the main text.

\subsection{Proofs of Theorems for Differential CP}
\label{A.2}
\noindent \emph{\bf Proof of Theorem~\ref{the:dcp_ocp}.}
For any fixed candidate model $\mu$ and calibration set $\mathcal D_{cal}$, let $
\widetilde q(\beta;\mu,\mathcal D_{cal})$
denote the $\lceil(1-\beta)(|\mathcal D_{cal}|+1)\rceil$-th smallest value among
$
\left\{ R\bigl(Z,\mu\bigr): Z\in\mathcal D_{cal}\right\}$.
In particular, whenever $\mu=\mathcal A(\mathcal D_{tra})$,
$
\widetilde q(\beta;\mathcal A(\mathcal D_{tra}),\mathcal D_{cal})
=
q(\beta,\mathcal D_{tra},\mathcal D_{cal})$.
Define
\[
S_{\beta}(\mu,\mathcal D_{cal})
=
\left\{
y \,\middle|\,
R\bigl((X_{n+1},y),\mu\bigr)
\le \widetilde q(\beta;\mu,\mathcal D_{cal})
\right\}.
\]
Then
$C_{\alpha}^{d}(X_{n+1})
=
S_{\alpha'}\bigl(\mathcal A(\mathcal D_n),\mathcal D_n\bigr)$,
$C_{\alpha_+}^{o}(X_{n+1})
=
S_{\alpha_+}\bigl(\mathcal A(\mathcal D_{n+1}),\mathcal D_{n+1}\bigr)$, $C_{\alpha_-}^{o}(X_{n+1})
=
S_{\alpha_-}\bigl(\mathcal A(\mathcal D_{n+1}),\mathcal D_{n+1}\bigr)$.

Since differential privacy holds for every fixed pair of adjacent datasets, we may condition on the realized augmented dataset $\mathcal D_{n+1}$ and treat both $\mathcal D_n$ and $\mathcal D_{n+1}$ as fixed. All probabilities below are therefore taken only over the internal randomness of $\mathcal A$, conditional on $\mathcal D_{n+1}$.

Let
\[
k=\left\lceil (1-\alpha')(n+1)\right\rceil,\qquad
k_+=\left\lceil (1-\alpha_+)(n+2)\right\rceil,\qquad
k_-=\left\lceil (1-\alpha_-)(n+2)\right\rceil.
\]
Since
\[
\alpha_+ = \alpha' + \frac{1}{n+1},
\qquad
\alpha_- = \alpha' - \frac{1}{n+1},
\]
we have
\[
k_+ \le k \le k_-.
\]
Therefore, for every fixed model $\mu$,
\[
\widetilde q(\alpha_+;\mu,\mathcal D_{n+1})
\le
\widetilde q(\alpha';\mu,\mathcal D_n)
\le
\widetilde q(\alpha_-;\mu,\mathcal D_{n+1}),
\]
because $\mathcal D_{n+1}$ contains exactly one additional calibration score beyond those in $\mathcal D_n$. Hence
\[
S_{\alpha_+}(\mu,\mathcal D_{n+1})
\subseteq
S_{\alpha'}(\mu,\mathcal D_n)
\subseteq
S_{\alpha_-}(\mu,\mathcal D_{n+1}).
\]

We first prove the upper bound. Define
\[
\mathcal F_+
=
\left\{
\mu \,\middle|\,
Y_{n+1}\notin S_{\alpha_+}(\mu,\mathcal D_{n+1})
\right\}.
\]
Since $\mathcal A$ is $(\varepsilon,\delta)$-DP and $\mathcal D_n$ and $\mathcal D_{n+1}$ are adjacent,
\[
\Pr\!\left(\mathcal A(\mathcal D_n)\in\mathcal F_+ \mid \mathcal D_{n+1}\right)
\le
e^\varepsilon
\Pr\!\left(\mathcal A(\mathcal D_{n+1})\in\mathcal F_+ \mid \mathcal D_{n+1}\right)
+\delta.
\]
Moreover, from
\[
S_{\alpha_+}\bigl(\mathcal A(\mathcal D_n),\mathcal D_{n+1}\bigr)
\subseteq
S_{\alpha'}\bigl(\mathcal A(\mathcal D_n),\mathcal D_n\bigr)
=
C_\alpha^d(X_{n+1}),
\]
we obtain
\[
\{Y_{n+1}\notin C_\alpha^d(X_{n+1})\}
\subseteq
\left\{
Y_{n+1}\notin
S_{\alpha_+}\bigl(\mathcal A(\mathcal D_n),\mathcal D_{n+1}\bigr)
\right\}.
\]
Therefore,
\begin{align*}
\Pr\!\left(Y_{n+1}\notin C_\alpha^d(X_{n+1}) \mid \mathcal D_{n+1}\right)
&\le
\Pr\!\left(
Y_{n+1}\notin
S_{\alpha_+}\bigl(\mathcal A(\mathcal D_n),\mathcal D_{n+1}\bigr)
\mid \mathcal D_{n+1}
\right) \\
&=
\Pr\!\left(\mathcal A(\mathcal D_n)\in\mathcal F_+ \mid \mathcal D_{n+1}\right) \\
&\le
e^\varepsilon
\Pr\!\left(\mathcal A(\mathcal D_{n+1})\in\mathcal F_+ \mid \mathcal D_{n+1}\right)
+\delta \\
&=
e^\varepsilon
\Pr\!\left(Y_{n+1}\notin C_{\alpha_+}^{o}(X_{n+1}) \mid \mathcal D_{n+1}\right)
+\delta.
\end{align*}

Next we prove the lower bound. Define
\[
\mathcal F_-
=
\left\{
\mu \,\middle|\,
Y_{n+1}\notin S_{\alpha_-}(\mu,\mathcal D_{n+1})
\right\}.
\]
Again by differential privacy and adjacency,
\[
\Pr\!\left(\mathcal A(\mathcal D_{n+1})\in\mathcal F_- \mid \mathcal D_{n+1}\right)
\le
e^\varepsilon
\Pr\!\left(\mathcal A(\mathcal D_n)\in\mathcal F_- \mid \mathcal D_{n+1}\right)
+\delta,
\]
which implies
\[
\Pr\!\left(\mathcal A(\mathcal D_n)\in\mathcal F_- \mid \mathcal D_{n+1}\right)
\ge
e^{-\varepsilon}
\left(
\Pr\!\left(\mathcal A(\mathcal D_{n+1})\in\mathcal F_- \mid \mathcal D_{n+1}\right)
-\delta
\right).
\]
On the other hand,
\[
C_\alpha^d(X_{n+1})
=
S_{\alpha'}\bigl(\mathcal A(\mathcal D_n),\mathcal D_n\bigr)
\subseteq
S_{\alpha_-}\bigl(\mathcal A(\mathcal D_n),\mathcal D_{n+1}\bigr),
\]
so
\[
\{Y_{n+1}\notin C_\alpha^d(X_{n+1})\}
\supseteq
\left\{
Y_{n+1}\notin
S_{\alpha_-}\bigl(\mathcal A(\mathcal D_n),\mathcal D_{n+1}\bigr)
\right\}.
\]
Hence
\begin{align*}
\Pr\!\left(Y_{n+1}\notin C_\alpha^d(X_{n+1}) \mid \mathcal D_{n+1}\right)
&\ge
\Pr\!\left(
Y_{n+1}\notin
S_{\alpha_-}\bigl(\mathcal A(\mathcal D_n),\mathcal D_{n+1}\bigr)
\mid \mathcal D_{n+1}
\right) \\
&=
\Pr\!\left(\mathcal A(\mathcal D_n)\in\mathcal F_- \mid \mathcal D_{n+1}\right) \\
&\ge
e^{-\varepsilon}
\left(
\Pr\!\left(\mathcal A(\mathcal D_{n+1})\in\mathcal F_- \mid \mathcal D_{n+1}\right)
-\delta
\right) \\
&=
e^{-\varepsilon}
\left(
\Pr\!\left(Y_{n+1}\notin C_{\alpha_-}^{o}(X_{n+1}) \mid \mathcal D_{n+1}\right)
-\delta
\right).
\end{align*}
\hfill$\square$

\noindent \emph{\bf Proof of Theorem~\ref{thm:dcp_coverage}.}
Let $\alpha' = e^{-\varepsilon}(\alpha-\delta)$, and define
$\alpha_+ = \alpha' + \frac{1}{n+1}$ and $\alpha_- = \alpha' - \frac{1}{n+1}$.

By Theorem~\ref{the:dcp_ocp},
\[
\Pr\!\left(Y_{n+1}\notin C_{\alpha}^{d}(X_{n+1}) \mid \mathcal{D}_{n+1}\right)
\le
e^\varepsilon
\Pr\!\left(Y_{n+1}\notin C_{\alpha_+}^{o}(X_{n+1}) \mid \mathcal{D}_{n+1}\right)
+\delta.
\]
Taking expectation with respect to $\mathcal{D}_{n+1}$ gives
\[
\Pr\!\left(Y_{n+1}\notin C_{\alpha}^{d}(X_{n+1})\right)
\le
e^\varepsilon
\Pr\!\left(Y_{n+1}\notin C_{\alpha_+}^{o}(X_{n+1})\right)
+\delta.
\]
Since the oracle scores are exchangeable, the oracle conformal predictor at level $\alpha_+$ satisfies
$\Pr\!\left(Y_{n+1}\notin C_{\alpha_+}^{o}(X_{n+1})\right)\le \alpha_+$.
Therefore,
\[
\Pr\!\left(Y_{n+1}\notin C_{\alpha}^{d}(X_{n+1})\right)
\le
e^\varepsilon \alpha_+ + \delta
=
e^\varepsilon\!\left(e^{-\varepsilon}(\alpha-\delta)+\frac{1}{n+1}\right)+\delta
=
\alpha+\frac{e^\varepsilon}{n+1}.
\]
Hence
\[
\Pr\!\left(Y_{n+1}\in C_{\alpha}^{d}(X_{n+1})\right)
\ge
1-\alpha-\frac{e^\varepsilon}{n+1}.
\]

Now assume in addition that the oracle scores are almost surely distinct. By Theorem~\ref{the:dcp_ocp},
\[
\Pr\!\left(Y_{n+1}\notin C_{\alpha}^{d}(X_{n+1}) \mid \mathcal{D}_{n+1}\right)
\ge
e^{-\varepsilon}
\Bigl(
\Pr\!\left(Y_{n+1}\notin C_{\alpha_-}^{o}(X_{n+1}) \mid \mathcal{D}_{n+1}\right)-\delta
\Bigr).
\]
Taking expectation over $\mathcal{D}_{n+1}$ yields
\[
\Pr\!\left(Y_{n+1}\notin C_{\alpha}^{d}(X_{n+1})\right)
\ge
e^{-\varepsilon}
\Bigl(
\Pr\!\left(Y_{n+1}\notin C_{\alpha_-}^{o}(X_{n+1})\right)-\delta
\Bigr).
\]
Equivalently,
\[
\Pr\!\left(Y_{n+1}\in C_{\alpha}^{d}(X_{n+1})\right)
\le
1-e^{-\varepsilon}
\Bigl(
\Pr\!\left(Y_{n+1}\notin C_{\alpha_-}^{o}(X_{n+1})\right)-\delta
\Bigr).
\]

Since the oracle scores are almost surely distinct, oracle CP at level $\alpha_-$ \\
satisfies
$\Pr\!\left(Y_{n+1}\in C_{\alpha_-}^{o}(X_{n+1})\right)\le 1-\alpha_-+\frac{1}{n+1}$,
and hence
$\Pr\!\left(Y_{n+1}\notin C_{\alpha_-}^{o}(X_{n+1})\right)\ge \alpha_- - \frac{1}{n+1}$.
Therefore,
\[
\Pr\!\left(Y_{n+1}\in C_{\alpha}^{d}(X_{n+1})\right)
\le
1-e^{-\varepsilon}\left(\alpha_- - \frac{1}{n+1}-\delta\right).
\]
Substituting $\alpha_- = e^{-\varepsilon}(\alpha-\delta)-\frac{1}{n+1}$, we obtain
\begin{align*}
\Pr\!\left(Y_{n+1}\in C_{\alpha}^{d}(X_{n+1})\right)
&\le
1-e^{-\varepsilon}\left(e^{-\varepsilon}(\alpha-\delta)-\frac{2}{n+1}-\delta\right) \\
&=
1-e^{-2\varepsilon}\alpha
+e^{-\varepsilon}\!\left[\delta(1+e^{-\varepsilon})+\frac{2}{n+1}\right].
\end{align*}
\hfill$\square$

\noindent\emph{\bf Proof of Lemma~\ref{lem:dpq_coverage}.}
Let
\[
\hat\mu=\mathcal A(\mathcal D_{tra}), \qquad
R_i=R(Z_i,\hat\mu), \quad i=1,\ldots,N.
\]
For each $j\in[M]$, define
\[
\pi_j(R_1,\ldots,R_N)
:=\Pr(\hat q=e_j\mid R_1,\ldots,R_N)
=
\frac{\exp\!\left(-\frac{\varepsilon w_j}{2\Delta}\right)}
{\sum_{\ell=1}^M \exp\!\left(-\frac{\varepsilon w_\ell}{2\Delta}\right)},
\qquad
\Delta=\max\left(\frac{1}{1-\alpha_0},\frac{1}{\alpha_0}\right).
\]
Since each $w_j$ depends only on the multiset of scores, $\pi_j$ is a symmetric function of $(R_1,\ldots,R_N)$.

\paragraph{Lower bound.}
By definition of $w_j$,
\[
w_j \ge \frac{|\{i:R_i>e_j\}|}{\alpha_0},
\]
so
\[
\lfloor w_j\alpha_0\rfloor \ge |\{i:R_i>e_j\}|,
\qquad
N-\lfloor w_j\alpha_0\rfloor \le |\{i:R_i\le e_j\}|.
\]
Hence
\begin{align*}
\Pr(R_i\le \hat q)
&=\sum_{j=1}^M \mathbb E\!\left[1\{R_i\le e_j\}\,\pi_j(R_1,\ldots,R_N)\right].
\end{align*}
By exchangeability of $R_1,\ldots,R_N$ and symmetry of $\pi_j$,
\begin{align*}
\mathbb E\!\left[1\{R_i\le e_j\}\,\pi_j(R_1,\ldots,R_N)\right]
&=
\frac1N
\mathbb E\!\left[
\sum_{\ell=1}^N 1\{R_\ell\le e_j\}\,\pi_j(R_1,\ldots,R_N)
\right] \\
&=
\frac1N
\mathbb E\!\left[
|\{\ell:R_\ell\le e_j\}|\,\pi_j(R_1,\ldots,R_N)
\right] \\
&\ge
\frac{N-\lfloor w_j\alpha_0\rfloor}{N}\Pr(\hat q=e_j).
\end{align*}
Summing over $j$ gives
\begin{align*}
\Pr(R_i\le \hat q)
&\ge
\sum_{j=1}^M
\frac{N-\lfloor w_j\alpha_0\rfloor}{N}\Pr(\hat q=e_j) \\
&\ge
1-\alpha_0\sum_{j=1}^M \frac{w_j}{N}\Pr(\hat q=e_j).
\end{align*}
Using the assumed weighted-average bound,
\[
\Pr(R_i\le \hat q)
\ge
1-\alpha_0\left(1+\frac{2}{N\varepsilon\alpha_0}\right)
=
1-\alpha_0-\frac{2}{N\varepsilon}
=
1-\beta.
\]

\paragraph{Upper bound.}
Assume now that $\Pr(R_i=e_j)=0$ for all $i\in[N]$ and $j\in[M]$. Then, almost surely,
\[
|\{i:R_i\le e_j\}|=|\{i:R_i<e_j\}|.
\]
By definition of $w_j$,
\[
w_j \ge \frac{|\{i:R_i<e_j\}|}{1-\alpha_0},
\]
hence
\[
|\{i:R_i\le e_j\}| \le \lfloor w_j(1-\alpha_0)\rfloor
\qquad\text{a.s.}
\]
Therefore,
\begin{align*}
\Pr(R_i\le \hat q)
&=\sum_{j=1}^M \mathbb E\!\left[1\{R_i\le e_j\}\,\pi_j(R_1,\ldots,R_N)\right].
\end{align*}
Again by exchangeability and symmetry,
\begin{align*}
\mathbb E\!\left[1\{R_i\le e_j\}\,\pi_j(R_1,\ldots,R_N)\right]
&=
\frac1N
\mathbb E\!\left[
\sum_{\ell=1}^N 1\{R_\ell\le e_j\}\,\pi_j(R_1,\ldots,R_N)
\right] \\
&=
\frac1N
\mathbb E\!\left[
|\{\ell:R_\ell\le e_j\}|\,\pi_j(R_1,\ldots,R_N)
\right] \\
&\le
\frac{\lfloor w_j(1-\alpha_0)\rfloor}{N}\Pr(\hat q=e_j) \\
&\le
\frac{(1-\alpha_0)w_j}{N}\Pr(\hat q=e_j).
\end{align*}
Summing over $j$ yields
\[
\Pr(R_i\le \hat q)
\le
(1-\alpha_0)\sum_{j=1}^M \frac{w_j}{N}\Pr(\hat q=e_j).
\]
Using the same weighted-average bound,
\begin{align*}
\Pr(R_i\le \hat q)
&\le
(1-\alpha_0)\left(1+\frac{2}{N\varepsilon\alpha_0}\right) \\
&=
(1-\alpha_0)\frac{\beta}{\alpha_0}.
\end{align*}
Since $\alpha_0=\beta-\frac{2}{N\varepsilon}$, we have
\[
(1-\alpha_0)\frac{\beta}{\alpha_0}
=
1-\beta+\frac{2}{N\varepsilon\beta-2}.
\]
Therefore
\[
\Pr(R_i\le \hat q)\le 1-\beta+\frac{2}{N\varepsilon\beta-2}.
\]
\hfill$\square$

\noindent \emph{\bf Proof of Theorem~\ref{thm:dcp_conditional_coverage_G}.}
Let
\[
k:=\left\lceil (1-\alpha')(n+1)\right\rceil,
\]
so that, by the definition of \(C_\alpha^d(X_{n+1})\), the dCP threshold is the \(k\)-th order statistic \(R_{(k)}\) of \(R_1,\ldots,R_n\). Hence
\[
\{Y_{n+1}\notin C_\alpha^d(X_{n+1})\}
=
\{R_{n+1}>R_{(k)}\}.
\]

Fix a realization \(\hat\mu_n=\mu\). By assumption, conditional on \(\hat\mu_n=\mu\), the scores \(R_1,\ldots,R_{n+1}\) are i.i.d. with common distribution function \(F_\mu\). Let
\[
F_{n,\mu}(r):=\frac1n\sum_{i=1}^n \mathbf 1\{R_i\le r\}
\]
be the empirical cdf based on \(R_1,\ldots,R_n\), and define
\[
t_\xi:=\sqrt{\frac{\log(2/\xi)}{2n}}.
\]
By the Dvoretzky--Kiefer--Wolfowitz inequality, the event
\[
E_\xi:=\left\{\sup_r |F_{n,\mu}(r)-F_\mu(r)|\le t_\xi\right\}
\]
satisfies
\[
\Pr(E_\xi\mid \hat\mu_n=\mu)\ge 1-\xi.
\]

Now work on the event \(E_\xi\). Since \(R_{(k)}\) is the \(k\)-th order statistic of \(R_1,\ldots,R_n\), we have
\[
F_{n,\mu}(R_{(k)})\ge \frac{k}{n}.
\]
Also, conditional on \(\mathcal G_n\), the quantity \(R_{(k)}\) is fixed, while \(R_{n+1}\) depends only on the fresh point \((X_{n+1},Y_{n+1})\). Because the fresh point is independent of \((\mathcal D_n,\hat\mu_n)\), the conditional distribution of \(R_{n+1}\) given \(\mathcal G_n\) is still \(F_\mu\). Therefore,
\[
\Pr\!\left(
Y_{n+1}\notin C_\alpha^d(X_{n+1})
\,\middle|\,\mathcal G_n
\right)
=
\Pr(R_{n+1}>R_{(k)}\mid \mathcal G_n)
=
1-F_\mu(R_{(k)}).
\]
On \(E_\xi\),
\[
1-F_\mu(R_{(k)})
\le
1-F_{n,\mu}(R_{(k)})+t_\xi
\le
1-\frac{k}{n}+t_\xi.
\]
Since
\[
k=\left\lceil (1-\alpha')(n+1)\right\rceil \ge (1-\alpha')(n+1),
\]
it follows that
\[
1-\frac{k}{n}
\le
1-\frac{(1-\alpha')(n+1)}{n}
=
\alpha'-\frac{1-\alpha'}{n}
\le
\alpha'.
\]
Hence, on \(E_\xi\),
\[
\Pr\!\left(
Y_{n+1}\notin C_\alpha^d(X_{n+1})
\,\middle|\,\mathcal G_n
\right)
\le
\alpha'+t_\xi
\le
\alpha+t_\xi,
\]
because \(\alpha'=e^{-\varepsilon}(\alpha-\delta)\le \alpha\).

Thus \(E_\xi\) implies the event
\[
A_\xi:=
\left\{
\Pr\!\left(
Y_{n+1}\notin C_\alpha^d(X_{n+1})
\,\middle|\, \mathcal G_n
\right)
\le
\alpha+t_\xi
\right\}.
\]
Since \(\Pr(E_\xi\mid \hat\mu_n=\mu)\ge 1-\xi\) for every realized \(\mu\), taking expectation over \(\hat\mu_n\) yields \(\Pr(E_\xi)\ge 1-\xi\). Because \(E_\xi\subseteq A_\xi\), we conclude that
\[
\Pr(A_\xi)\ge 1-\xi.
\]
Recalling the definition of \(t_\xi\), this is exactly
\[
\Pr\!\left(
\Pr\!\left(
Y_{n+1}\notin C_\alpha^{d}(X_{n+1})
\,\middle|\, \mathcal G_n
\right)
\le
\alpha+\sqrt{\frac{\log(2/\xi)}{2n}}
\right)
\ge
1-\xi.
\]
\hfill\(\square\)

\noindent \emph{\bf Proof of Theorem~\ref{the:ERM_dCP}.}

{\bf Step 1: A bound on $\|\zeta\|_2$.}

Since each coordinate of $\zeta$ is centered Gaussian with variance
\[
\sigma_1^2
=
2\log\!\left(\frac{1.25}{\delta}\right)\frac{\tau^2}{\varepsilon_1^2},
\]
for any $t>0$,
\[
\Pr(|\zeta_1|>t)
\le
2\exp\!\left(
-\frac{t^2}{4\sigma_1^2}
\right).
\]
Because $d$ is fixed,
\[
\Pr(\|\zeta\|_2>t)
\le
d\,\Pr\!\left(|\zeta_1|>\frac{t}{\sqrt d}\right)
\le
2d\exp\!\left(
-\frac{t^2}{4d\sigma_1^2}
\right).
\]
By Lemma~\ref{lem:ERM_l2}, $\tau=\frac{2\rho L}{\lambda n}=O(n^{-1})$. Since $\delta=O(n^{-1})$, we also have $\log(1.25/\delta)=O(\log n)$. Therefore
\[
\Pr(\|\zeta\|_2>t)
\le
O\!\left(
\exp\!\left(
-\frac{\varepsilon_1^2 n^2 t^2}{\log n}
\right)
\right).
\]

{\bf Step 2: Comparing the oracle and differential-CP centers and radii.}

Write
\[
C_\alpha^o(X_{n+1})=[a_o-r_o,\;a_o+r_o],
\qquad
C_\alpha^d(X_{n+1})=[a_d-r_d,\;a_d+r_d],
\]
where
\[
a_o=\mu(X_{n+1};\hat{\vartheta}^{(n+1)}),
\qquad
a_d=\mu(X_{n+1};\widetilde{\vartheta}^{(n)}),
\]
and
\[
r_o=R_{(k)}^{(n+1)},
\qquad
r_d=R_{(k')}^{(n)}.
\]
For intervals on $\mathbb R$,
\[
L\!\left(
C_\alpha^o(X_{n+1})\triangle C_\alpha^d(X_{n+1})
\right)
\le
2|a_o-a_d|+2|r_o-r_d|.
\]
By the $L$-Lipschitz continuity of $\mu$,
\[
|a_o-a_d|
\le
L\bigl\|\hat{\vartheta}^{(n+1)}-\widetilde{\vartheta}^{(n)}\bigr\|_2
\le
L\tau+L\|\zeta\|_2.
\]

Next,
\[
|r_o-r_d|
\le
\left|R_{(k)}^{(n+1)}-R_{(k)}^{(n)}\right|
+
\left|R_{(k)}^{(n)}-R_{(k')}^{(n)}\right|.
\]
To control the first term, note that for the absolute residual score,
\[
\bigl|
R(Z_i,\mu(\cdot;\vartheta_1))
-
R(Z_i,\mu(\cdot;\vartheta_2))
\bigr|
\le
\bigl|
\mu(X_i;\vartheta_1)-\mu(X_i;\vartheta_2)
\bigr|.
\]

If $R_{n+1}^{(n+1)} > R_{(k)}^{(n+1)}$, then removing the $(n+1)$-st oracle score does not change the $k$-th order statistic, so
\[
R_{(k)}^{(n+1)}
=
\bigl(R_1^{(n+1)},\ldots,R_n^{(n+1)}\bigr)_{(k)}.
\]
Hence
\[
\left|R_{(k)}^{(n+1)}-R_{(k)}^{(n)}\right|
\le
\max_{1\le i\le n}\left|R_i^{(n+1)}-R_i^{(n)}\right|
\le
L\tau+L\|\zeta\|_2.
\]

If $R_{n+1}^{(n+1)} \le R_{(k)}^{(n+1)}$, then removing the $(n+1)$-st oracle score shifts the $k$-th order statistic among the oracle scores to the $(k+1)$-st one, so
\[
\left|R_{(k)}^{(n+1)}-R_{(k)}^{(n)}\right|
\le
\left|R_{(k+1)}^{(n+1)}-R_{(k)}^{(n+1)}\right|
+
\max_{1\le i\le n}\left|R_i^{(n+1)}-R_i^{(n)}\right|.
\]
Again,
\[
\max_{1\le i\le n}\left|R_i^{(n+1)}-R_i^{(n)}\right|
\le
L\tau+L\|\zeta\|_2.
\]
Therefore, in either case,
\[
\left|R_{(k)}^{(n+1)}-R_{(k)}^{(n)}\right|
\le
\left|R_{(k+1)}^{(n+1)}-R_{(k)}^{(n+1)}\right|
+
L\tau+L\|\zeta\|_2,
\]
where in the first case the spacing term may be set to zero. Combining the above bounds gives
\[
L\!\left(
C_\alpha^o(X_{n+1})\triangle C_\alpha^d(X_{n+1})
\right)
\le
4L(\tau+\|\zeta\|_2)
+
2\left|R_{(k+1)}^{(n+1)}-R_{(k)}^{(n+1)}\right|
+
2\left|R_{(k')}^{(n)}-R_{(k)}^{(n)}\right|.
\]

{\bf Step 3: Control of the two order-statistic gaps.}

Let $P_n$ and $P_{n+1}$ denote the actual laws of $(R_1^{(n)},\ldots,R_n^{(n)})$ and $(R_1^{(n+1)},\ldots,R_{n+1}^{(n+1)})$, respectively. By assumption, there exist i.i.d. reference laws $Q_n$ and $Q_{n+1}$ such that under both reference laws the common distribution function is $F$, and
\[
d_{TV}(P_n,Q_n)\le \xi_{TV},
\qquad
d_{TV}(P_{n+1},Q_{n+1})\le \xi_{TV}.
\]

We first control $R_{(k')}^{(n)}-R_{(k)}^{(n)}$. Under $Q_n$, define
\[
U_i=F(R_i^{(n)}),\qquad
U_{(k)}=F(R_{(k)}^{(n)}),\qquad
U_{(k')}=F(R_{(k')}^{(n)}),\qquad
W=U_{(k')}-U_{(k)}.
\]
Since $Q_n$ is i.i.d. with cdf $F$, the variables $U_1,\ldots,U_n$ are i.i.d. $\mathrm{Unif}(0,1)$, and $W$ is the sum of $k'-k$ adjacent spacings of uniform order statistics. Hence
\[
W\sim \mathrm{Beta}(k'-k,\;n+1-(k'-k)),
\]
so its mean is
\[
\kappa:=\frac{k'-k}{n+1}
\]
and its variance is
\[
\mathrm{Var}(W)=\frac{\kappa(1-\kappa)}{n+2}.
\]

Now
\[
\alpha_1=e^{-\varepsilon_1}(\alpha-\delta)=\alpha+O(\varepsilon_1)+O(\delta),
\]
and since $\delta=O(n^{-1})$ and $\varepsilon_1 n\to\infty$, we have
\[
\kappa=O(\varepsilon_1).
\]
Applying Cantelli's inequality to $W$ with threshold $\kappa$ gives
\[
Q_n(W>2\kappa)
\le
\frac{\kappa(1-\kappa)}{\kappa(1-\kappa)+\kappa^2(n+2)}
=
O\!\left(\frac{1}{\varepsilon_1 n}\right).
\]

Because $\alpha_1=\alpha+O(\varepsilon_1)+O(\delta)$ and $\delta=O(n^{-1})$, both $k/(n+1)$ and $k'/(n+1)$ belong to the neighborhood $[1-\alpha-\epsilon_0,\;1-\alpha+\epsilon_0]$ for all sufficiently large $n$. Therefore Assumption~\ref{assump} implies
\[
R_{(k')}^{(n)}-R_{(k)}^{(n)}
=
F^{-1}(U_{(k')})-F^{-1}(U_{(k)})
\le
h\,|U_{(k')}-U_{(k)}|^\gamma
=
hW^\gamma
\]
under the reference law $Q_n$. Hence
\[
Q_n\!\left(
R_{(k')}^{(n)}-R_{(k)}^{(n)}
>
O(\varepsilon_1^\gamma)
\right)
\le
O\!\left(\frac{1}{\varepsilon_1 n}\right).
\]
Transferring this bound from $Q_n$ to $P_n$ via total variation gives
\[
P_n\!\left(
R_{(k')}^{(n)}-R_{(k)}^{(n)}
>
O(\varepsilon_1^\gamma)
\right)
\le
O\!\left(\frac{1}{\varepsilon_1 n}\right)+\xi_{TV}.
\]

We next control the adjacent spacing $R_{(k+1)}^{(n+1)}-R_{(k)}^{(n+1)}$. Under $Q_{n+1}$, define
\[
W_0:=F(R_{(k+1)}^{(n+1)})-F(R_{(k)}^{(n+1)}).
\]
Since $Q_{n+1}$ is i.i.d. with cdf $F$, $W_0$ is an adjacent spacing of uniform order statistics, and hence $W_0=O_{\Pr}(n^{-1})$. Because both $k/(n+1)$ and $(k+1)/(n+1)$ lie inside $[1-\alpha-\epsilon_0,\;1-\alpha+\epsilon_0]$ for all sufficiently large $n$, Assumption~\ref{assump} yields
\[
R_{(k+1)}^{(n+1)}-R_{(k)}^{(n+1)}
\le
hW_0^\gamma
\]
under $Q_{n+1}$. Therefore
\[
Q_{n+1}\!\left(
R_{(k+1)}^{(n+1)}-R_{(k)}^{(n+1)}
>
O(n^{-\gamma})
\right)
\le
O(n^{-1}),
\]
and hence, by total variation,
\[
P_{n+1}\!\left(
R_{(k+1)}^{(n+1)}-R_{(k)}^{(n+1)}
>
O(n^{-\gamma})
\right)
\le
O(n^{-1})+\xi_{TV}.
\]

{\bf Step 4: Combination.}

From Step 2,
\[
L\!\left(
C_\alpha^o(X_{n+1})\triangle C_\alpha^d(X_{n+1})
\right)
\le
4L(\tau+\|\zeta\|_2)
+
2\left|R_{(k+1)}^{(n+1)}-R_{(k)}^{(n+1)}\right|
+
2\left|R_{(k')}^{(n)}-R_{(k)}^{(n)}\right|.
\]
By Step 1,
\[
\Pr(\|\zeta\|_2>t)
\le
O\!\left(
\exp\!\left(
-\frac{\varepsilon_1^2 n^2 t^2}{\log n}
\right)
\right).
\]
By Step 3,
\[
\Pr\!\left(
R_{(k')}^{(n)}-R_{(k)}^{(n)}>O(\varepsilon_1^\gamma)
\right)
\le
O\!\left(\frac{1}{\varepsilon_1 n}\right)+\xi_{TV},
\]
and
\[
\Pr\!\left(
R_{(k+1)}^{(n+1)}-R_{(k)}^{(n+1)}>O(n^{-\gamma})
\right)
\le
O(n^{-1})+\xi_{TV}.
\]
Since $\tau=O(n^{-1})=o(\varepsilon_1)$ and $\varepsilon_1 n\to\infty$ imply $n^{-\gamma}=o(\varepsilon_1^\gamma)$ and $n^{-1}=O((\varepsilon_1 n)^{-1})$, a union bound yields
\[
\Pr\!\left(
L\!\left(
C_\alpha^o(X_{n+1})\triangle C_\alpha^d(X_{n+1})
\right)
>
O(\varepsilon_1^\gamma+t)
\right)
\le
O\!\left(
\exp\!\left(
-\frac{\varepsilon_1^2 n^2 t^2}{\log n}
\right)
+
\frac{1}{\varepsilon_1 n}
\right)
+
2\xi_{TV}.
\]
This completes the proof.
\hfill$\square$

\section{Proofs for Section 4}
\label{B}
\noindent \emph{\bf Proof of Theorem~\ref{thm:privacy_dpcp}.}
We prove the two statements in turn.

For the first claim, recall that Algorithm~\ref{alg:DPQ} selects a candidate threshold $e_j$ with probability proportional to
\[
\exp\!\left(-\frac{\varepsilon_2\, w_j}{2\Delta}\right),
\]
where
\[
\Delta
=
\max\!\left\{\frac{1}{1-\alpha_0},\frac{1}{\alpha_0}\right\}.
\]
For each fixed $j$, the score
\[
w_j
=
\max\!\left\{
\frac{|\{i:R_i<e_j\}|}{1-\alpha_0},
\frac{|\{i:R_i>e_j\}|}{\alpha_0}
\right\}
\]
depends on the input score vector only through two counting queries. If one score is changed, each count changes by at most one, so $w_j$ changes by at most
\[
\max\!\left\{\frac{1}{1-\alpha_0},\frac{1}{\alpha_0}\right\}
=
\Delta.
\]
Therefore Algorithm~\ref{alg:DPQ} is an instance of the exponential mechanism with privacy budget $\varepsilon_2$, and hence it is $\varepsilon_2$-DP.

For the second claim, the DPCP procedure consists of two sequential stages. In the first stage, it releases the private fitted model $\hat\mu_n=\mathcal A(\mathcal D_n)$, which is $(\varepsilon_1,\delta)$-DP. In the second stage, it computes the score vector from $(\mathcal D_n,\hat\mu_n)$ and then applies Algorithm~\ref{alg:DPQ} with privacy budget $\varepsilon_2$ to obtain the private quantile $\hat q$. Since Algorithm~\ref{alg:DPQ} is $\varepsilon_2$-DP with respect to its score-vector input, the second stage is $\varepsilon_2$-DP conditional on the first-stage output. By adaptive sequential composition, the joint release $(\hat\mu_n,\hat q)$ is therefore $(\varepsilon_1+\varepsilon_2,\delta)$-DP.

Finally, the DPCP set
\[
C_\alpha^{dp}(X_{n+1})
=
\bigl\{
y:\,
R((X_{n+1},y),\hat\mu_n)\le \hat q
\bigr\}
\]
is a deterministic function of $(\hat\mu_n,\hat q)$. Hence it is obtained by post-processing, and therefore $C_\alpha^{dp}(X_{n+1})$ satisfies $(\varepsilon_1+\varepsilon_2,\delta)$-DP.
\hfill$\square$

\noindent \emph{\bf Proof of Theorem~\ref{the:DPCP_cpverage}.}
Conditional on \(\mathcal G_n\), the quantities \(R_1,\ldots,R_n\), \(w_1,\ldots,w_M\), \(\alpha_0\), and \(\alpha_1',\ldots,\alpha_M'\) are fixed. Moreover, once \(\mathcal G_n\) is fixed, the only remaining randomness in \(\hat q\) comes from Algorithm~\ref{alg:DPQ}, while the only remaining randomness in \(R_{n+1}\) comes from the fresh point \((X_{n+1},Y_{n+1})\). Hence \(\hat q\) and \(R_{n+1}\) are conditionally independent given \(\mathcal G_n\).

Let \(\pi_j:=\Pr(\hat q=e_j\mid \mathcal G_n)\). Then
\[
\Pr(R_{n+1}\le \hat q\mid \mathcal G_n)
=
1-\sum_{j=1}^M \pi_j\,\Pr(R_{n+1}>e_j\mid \mathcal G_n).
\]
By Assumption~\ref{assump:dpcp_pointwise_validity},
\[
\sum_{j=1}^M \pi_j\,\Pr(R_{n+1}>e_j\mid \mathcal G_n)
\le
\sum_{j=1}^M \pi_j\bigl(e^{\varepsilon_1}\alpha_j'+\delta\bigr)
=
e^{\varepsilon_1}\sum_{j=1}^M \pi_j\alpha_j'+\delta.
\]
Applying Assumption~\ref{assump:dpcp_avg_nominal}, we obtain
\[
\Pr(R_{n+1}\le \hat q\mid \mathcal G_n)
\ge
1-e^{\varepsilon_1}\alpha_1-\delta.
\]
Since \(\alpha_1=e^{-\varepsilon_1}(\alpha-\delta)\), the right-hand side equals \(1-\alpha\). Therefore
\[
\Pr(R_{n+1}\le \hat q\mid \mathcal G_n)\ge 1-\alpha
\qquad \text{a.s.}
\]
Finally, by the definition of \(C_\alpha^{dp}(X_{n+1})\), the event \(\{Y_{n+1}\in C_\alpha^{dp}(X_{n+1})\}\) is exactly \(\{R_{n+1}\le \hat q\}\). Taking expectation over \(\mathcal G_n\) yields
\[
\Pr\!\left(Y_{n+1}\in C_\alpha^{dp}(X_{n+1})\right)\ge 1-\alpha.
\]
\hfill\(\square\)

\noindent \emph{\bf Proof of Corollary~\ref{cor:dpcp_conditional_coverage}.}
The proof of Theorem~\ref{the:DPCP_cpverage} already establishes that
\[
\Pr(R_{n+1}\le \hat q\mid \mathcal G_n)\ge 1-\alpha
\qquad \text{a.s.}
\]
Since the event \(\{Y_{n+1}\in C_\alpha^{dp}(X_{n+1})\}\) is exactly \(\{R_{n+1}\le \hat q\}\), the stated conditional coverage bound follows immediately.
\hfill\(\square\)

\noindent \emph{\bf Proof of Lemma~\ref{lem:ERM_l2}.}
Let
\[
f_{n+1}(\vartheta):=\ell\!\left(Y_{n+1},\mu(X_{n+1};\vartheta)\right),
\qquad
\Delta:=\hat{\vartheta}^{(n+1)}-\hat{\vartheta}^{(n)}.
\]
By definition,
\[
J(\vartheta;\mathcal{D}_{n+1})
=
J(\vartheta;\mathcal{D}_n)+\frac{1}{n}f_{n+1}(\vartheta).
\]
Since $J(\vartheta;\mathcal{D}_n)$ is $\lambda$-strongly convex and $\hat{\vartheta}^{(n)}$ minimizes $J(\vartheta;\mathcal{D}_n)$, we have
\[
J(\hat{\vartheta}^{(n+1)};\mathcal{D}_n)
\ge
J(\hat{\vartheta}^{(n)};\mathcal{D}_n)
+\frac{\lambda}{2}\|\Delta\|_2^2.
\]
On the other hand, since $\hat{\vartheta}^{(n+1)}$ minimizes $J(\vartheta;\mathcal{D}_{n+1})$,
\[
J(\hat{\vartheta}^{(n+1)};\mathcal{D}_n)
+\frac{1}{n}f_{n+1}(\hat{\vartheta}^{(n+1)})
\le
J(\hat{\vartheta}^{(n)};\mathcal{D}_n)
+\frac{1}{n}f_{n+1}(\hat{\vartheta}^{(n)}).
\]
Combining the two displays yields
\[
\frac{\lambda}{2}\|\Delta\|_2^2
\le
\frac{1}{n}\Bigl(
f_{n+1}(\hat{\vartheta}^{(n)})-f_{n+1}(\hat{\vartheta}^{(n+1)})
\Bigr).
\]
Because $f_{n+1}$ is $\rho L$-Lipschitz in $\vartheta$,
\[
\left|
f_{n+1}(\hat{\vartheta}^{(n)})-f_{n+1}(\hat{\vartheta}^{(n+1)})
\right|
\le
\rho L\,\|\Delta\|_2.
\]
Hence,
\[
\frac{\lambda}{2}\|\Delta\|_2^2
\le
\frac{\rho L}{n}\|\Delta\|_2.
\]
If $\Delta=0$, the result is immediate. Otherwise, dividing both sides by $\|\Delta\|_2$ gives
\[
\|\hat{\vartheta}^{(n+1)}-\hat{\vartheta}^{(n)}\|_2
\le
\frac{2\rho L}{\lambda n}.
\]
Taking the supremum over adjacent datasets completes the proof.
\hfill$\square$

The differential CP interval is given by
\[
C_{\alpha}^{d}(X_{n+1})
=
\left[
\mu\!\left(X_{n+1};\widetilde{\vartheta}^{(n)}\right)-R_{(k')}^{(n)},
\;
\mu\!\left(X_{n+1};\widetilde{\vartheta}^{(n)}\right)+R_{(k')}^{(n)}
\right],
\]
where
\[
\widetilde{\vartheta}^{(n)}=\hat{\vartheta}^{(n)}+\zeta,
\qquad
\zeta\sim\mathcal{N}(0,\sigma_1^2 I_d),
\qquad
\sigma_1^2
=
2\log\!\left(\frac{1.25}{\delta}\right)\frac{\tau^2}{\varepsilon_1^2},
\]
and $R_{(k')}^{(n)}$ is the $k'$-th smallest value among
\[
R_i^{(n)}
=
R\!\left(Z_i,\mu(\cdot;\widetilde{\vartheta}^{(n)})\right),
\qquad
i=1,\ldots,n,
\]
with
\[
k'=\left\lceil (1-\alpha_1)(n+1)\right\rceil,
\qquad
\alpha_1=e^{-\varepsilon_1}(\alpha-\delta).
\]

Since the only difference between DPCP and differential CP lies in whether the calibration quantile is released privately, we first establish the efficiency of the differential CP interval.

\noindent \emph{\bf Proof of Theorem~\ref{the:4.6}.}
We first derive a general bound in terms of \((\varepsilon_1,\varepsilon_2)\), and then specialize to the balanced split \(\varepsilon_1=\varepsilon_2=\varepsilon/2\).

Because Algorithm~\ref{alg:DPQ} is implemented with a rank-based grid, its output \(\hat q\) coincides with an order statistic of \(\{R_i^{(n)}\}_{i=1}^n\), say
\[
\hat q=R_{(\hat k)}^{(n)}
\]
for some random rank \(\hat k\in\{1,\ldots,n\}\). Recall that
\[
\alpha_0=\alpha_1-\frac{2}{n\varepsilon_2},
\qquad
\alpha_1=e^{-\varepsilon_1}(\alpha-\delta),
\qquad
\Delta=\max\left\{\frac{1}{1-\alpha_0},\frac{1}{\alpha_0}\right\}.
\]
Since \(\delta=O(n^{-1})\) and \(\varepsilon_1\to 0\), we have \(\alpha_1\to \alpha\in(0,1)\), and hence \(\alpha_0\to \alpha\in(0,1)\). Therefore \(\Delta=O(1)\).

Let
\[
w_j
=
\max\left\{
\frac{|\{i:R_i^{(n)}<e_j\}|}{1-\alpha_0},
\frac{|\{i:R_i^{(n)}>e_j\}|}{\alpha_0}
\right\}.
\]
By the utility guarantee of the exponential mechanism, for every \(t_2\in(0,1)\),
\[
\Pr\!\left(
w(\hat q)>
\min_{1\le j\le M}w_j+\frac{2\Delta\log(M/t_2)}{\varepsilon_2}
\right)\le t_2.
\]
Since the grid is rank-based, there exists a candidate threshold whose empirical rank is within \(O(1)\) of \((1-\alpha_0)n\), so that \(\min_{1\le j\le M} w_j\le n+O(1)\). Define
\[
\mathcal E_{t_2}
:=
\left\{
w(\hat q)\le n+\frac{2\Delta\log(M/t_2)}{\varepsilon_2}
\right\}.
\]
If \(\hat q=R_{(m)}^{(n)}\), then
\[
w(\hat q)\ge
\max\left\{
\frac{m-1}{1-\alpha_0},
\frac{n-m}{\alpha_0}
\right\}.
\]
Hence on \(\mathcal E_{t_2}\), the selected rank must satisfy
\[
|\hat k-(1-\alpha_0)n|
=
O\!\left(\frac{\log(M/t_2)}{\varepsilon_2}\right).
\]

Now let
\[
k'=\left\lceil(1-\alpha_1)(n+1)\right\rceil,
\qquad
\hat\kappa:=\frac{|\hat k-k'|}{n+1}.
\]
Since \(\alpha_1-\alpha_0=2/(n\varepsilon_2)\), it follows that on \(\mathcal E_{t_2}\),
\[
\hat\kappa
=
O\!\left(\frac{\log(M/t_2)}{n\varepsilon_2}\right).
\]

We next control the random gap between \(\hat q\) and \(R_{(k')}^{(n)}\). Let \(P_n\) denote the actual law of \((R_1^{(n)},\ldots,R_n^{(n)})\). By assumption, there exists an i.i.d. reference law \(Q_n\) with common distribution function \(F\) such that
\[
d_{TV}(P_n,Q_n)\le \xi_{TV}.
\]
Under \(Q_n\), define
\[
\widehat W
:=
\left|F(R_{(\hat k)}^{(n)})-F(R_{(k')}^{(n)})\right|.
\]
Conditional on \(\hat\kappa\le \kappa_0\), the gap \(\widehat W\) is the sum of at most \(\kappa_0(n+1)\) adjacent spacings of uniform order statistics. The usual spacing bound under the idealized i.i.d. model therefore gives
\[
Q_n\!\left(
\widehat W>t_1+\kappa_0
\;\middle|\;
\hat\kappa\le\kappa_0
\right)
\le
\max_{0\le \kappa\le \kappa_0}
\frac{\kappa(1-\kappa)}{\kappa(1-\kappa)+t_1^2(n+2)}
=
\frac{\kappa_0(1-\kappa_0)}{\kappa_0(1-\kappa_0)+t_1^2(n+2)}.
\]

Choose
\[
t_2=\frac{1}{n\varepsilon_2},
\qquad
t_1=\kappa_0=O\!\left(\frac{\log n}{n\varepsilon_2}\right).
\]
This choice is valid for all sufficiently large \(n\) because \(M=O(n)\) and \(\varepsilon_2 n\to\infty\). Then
\[
\Pr(\hat\kappa>\kappa_0)\le t_2=O\!\left(\frac{1}{n\varepsilon_2}\right),
\]
and hence, under the idealized law \(Q_n\),
\[
Q_n\!\left(
\widehat W>
O\!\left(\frac{\log n}{n\varepsilon_2}\right)
\right)
\le
O\!\left(
\frac{1}{n\varepsilon_2}
+
\frac{\varepsilon_2}{\log n}
\right).
\]
Since \(d_{TV}(P_n,Q_n)\le \xi_{TV}\), the same bound holds under the true law up to an additive \(\xi_{TV}\):
\[
P_n\!\left(
\widehat W>
O\!\left(\frac{\log n}{n\varepsilon_2}\right)
\right)
\le
O\!\left(
\frac{1}{n\varepsilon_2}
+
\frac{\varepsilon_2}{\log n}
\right)
+
\xi_{TV}.
\]

By Assumption~\ref{assump}, \(F^{-1}\) is \(\gamma\)-H\"older continuous near the target level, and both \(F(R_{(\hat k)}^{(n)})\) and \(F(R_{(k')}^{(n)})\) lie in that neighborhood for all sufficiently large \(n\). Therefore,
\[
|\hat q-R_{(k')}^{(n)}|
=
\left|
F^{-1}\!\bigl(F(R_{(\hat k)}^{(n)})\bigr)
-
F^{-1}\!\bigl(F(R_{(k')}^{(n)})\bigr)
\right|
\le
h\,\widehat W^\gamma,
\]
which implies
\[
\Pr\!\left(
|\hat q-R_{(k')}^{(n)}|
>
O\!\left(\frac{\log n}{n\varepsilon_2}\right)^\gamma
\right)
\le
O\!\left(
\frac{1}{n\varepsilon_2}
+
\frac{\varepsilon_2}{\log n}
\right)
+
\xi_{TV}.
\]

Since \(C_\alpha^{dp}(X_{n+1})\) and \(C_\alpha^d(X_{n+1})\) have the same center, we have
\[
L\!\left(
C_\alpha^{dp}(X_{n+1})\triangle C_\alpha^d(X_{n+1})
\right)
\le
2|\hat q-R_{(k')}^{(n)}|.
\]
Hence
\[
\Pr\!\left(
L\!\left(
C_\alpha^{dp}(X)\triangle C_\alpha^d(X)
\right)
>
O\!\left(\frac{\log n}{n\varepsilon_2}\right)^\gamma
\right)
\le
O\!\left(
\frac{1}{n\varepsilon_2}
+
\frac{\varepsilon_2}{\log n}
\right)
+
\xi_{TV}.
\]

On the other hand, Theorem~\ref{the:ERM_dCP} yields
\[
\Pr\!\left(
L\!\left(
C_\alpha^o(X)\triangle C_\alpha^d(X)
\right)
>
O(\varepsilon_1^\gamma+t)
\right)
\le
O\!\left(
\exp\!\left(
-\frac{\varepsilon_1^2n^2t^2}{\log n}
\right)
+
\frac{1}{n\varepsilon_1}
\right)
+
2\xi_{TV}.
\]
Using the triangle inequality,
\[
L\!\left(
C_\alpha^o(X)\triangle C_\alpha^{dp}(X)
\right)
\le
L\!\left(
C_\alpha^o(X)\triangle C_\alpha^d(X)
\right)
+
L\!\left(
C_\alpha^d(X)\triangle C_\alpha^{dp}(X)
\right),
\]
we obtain
\[
\Pr\!\left(
L\!\left(
C_\alpha^o(X)\triangle C_\alpha^{dp}(X)
\right)
>
O\!\left(
\left(\varepsilon_1+\frac{\log n}{n\varepsilon_2}\right)^\gamma+t
\right)
\right)
\]
\[
\le
O\!\left(
\exp\!\left(
-\frac{\varepsilon_1^2n^2t^2}{\log n}
\right)
+
\frac{1}{n\varepsilon_1}
+
\frac{1}{n\varepsilon_2}
+
\frac{\varepsilon_2}{\log n}
\right)
+
3\xi_{TV}.
\]

Finally, under the balanced split
\[
\varepsilon_1=\varepsilon_2=\varepsilon/2,
\qquad
\varepsilon=\Theta(n^{-\eta}),
\qquad
\eta\in(0,1),
\]
we have
\[
\left(\varepsilon_1+\frac{\log n}{n\varepsilon_2}\right)^\gamma
=
O\!\left(
\left(
\frac{1}{n^\eta}
+
\frac{\log n}{n^{1-\eta}}
\right)^\gamma
\right),
\]
\[
\frac{1}{n\varepsilon_1}
+
\frac{1}{n\varepsilon_2}
=
O\!\left(\frac{1}{n^{1-\eta}}\right),
\qquad
\frac{\varepsilon_2}{\log n}
=
O\!\left(\frac{1}{n^\eta\log n}\right),
\]
and
\[
\exp\!\left(
-\frac{\varepsilon_1^2n^2t^2}{\log n}
\right)
=
\exp\!\left(
-\frac{n^{2-2\eta}t^2}{\log n}
\right)
\]
up to multiplicative constants in the exponent. Substituting these bounds into the previous display yields
\begin{align*}
\Pr\!\left(
L\!\left(
C_\alpha^o(X)\triangle C_\alpha^{dp}(X)
\right)
>
O\!\left(
\left(
\frac{1}{n^\eta}
+
\frac{\log n}{n^{1-\eta}}
\right)^\gamma+t
\right)
\right)
\le\;&
O\!\left(
\frac{1}{n^{1-\eta}}
+
\frac{1}{n^\eta\log n}
+
\exp\!\left(
-\frac{n^{2-2\eta}t^2}{\log n}
\right)
\right)
\\
&\quad+3\xi_{TV}.
\end{align*}
\hfill\(\square\)

\section{A Heuristic Local-Sensitivity Refinement for Differential CP}
\label{C}

In Theorem~\ref{thm:dcp_coverage}, the adjusted level $
e^{-\varepsilon}(\alpha-\delta)$
arises from the global worst-case privacy guarantee. This correction is valid uniformly over all adjacent datasets, and is therefore fully distribution-free. However, such a worst-case adjustment may be conservative in practice, especially when the effective perturbation induced by the privacy mechanism is much smaller for most datasets than that implied by the global sensitivity bound.

A natural way to reduce this conservativeness is to exploit local sensitivity information. The basic idea is to partition the data space into regions according to the magnitude of the perturbation induced by replacing one observation, and then calibrate the effective miscoverage adjustment separately within each region. Intuitively, if the fitted model is relatively stable in a certain part of the data space, then the local privacy loss may be substantially smaller than the global upper bound $\varepsilon$, suggesting that a less conservative conformal correction may still be valid there.

More concretely, consider adjacent datasets that differ in only one sample. For a given learning algorithm $\mathcal{A}$, one may associate to each such local perturbation a quantity measuring the effective change in the output of the algorithm, for example through a local sensitivity proxy. If the data space can be partitioned into subsets on which this local perturbation is uniformly controlled by different levels
\[
0 \le \varepsilon_1 \le \varepsilon_2 \le \cdots \le \varepsilon,
\]
then it is natural to expect that the global correction $e^{-\varepsilon}(\alpha-\delta)$ could be replaced by a less conservative weighted adjustment reflecting the distribution of the data across these subsets.

This intuition is closely related to the fact that differential privacy provides a uniform bound over the entire data space, whereas conformal coverage analysis averages over the joint distribution of the data. Thus, although the global DP bound is necessary for worst-case privacy protection, it may overstate the effective loss relevant to average-case coverage behavior. A refinement based on local sensitivity partitioning may therefore improve efficiency while preserving the basic logic of differential CP.

We emphasize, however, that such a refinement is substantially more delicate than the global analysis developed in the main text. A rigorous treatment would require, at a minimum, a precise definition of the local privacy ratio, a measurable partition of the data space, and a careful analysis of how the resulting conditional bounds interact with exchangeability and conformal ranking. Since these issues are beyond the scope of the present paper, we do not pursue a formal theorem here.

Nevertheless, the idea can be illustrated numerically in a simple regression example. Suppose we observe i.i.d.\ samples
\[
(x_1,y_1),\ldots,(x_n,y_n), \qquad n=100,
\]
generated from
\[
y_i=x_i+5+\xi_i, \qquad \xi_i\sim N(0,1),
\]
where $x_i$ is restricted to a bounded interval $[x_{\mathrm{lo}},x_{\mathrm{up}}]$, and each $y_i$ is truncated accordingly so that the output range is bounded. We fit the location model
\[
y=x+b
\]
by least squares under the bounded data domain. Figure~\ref{fig4} illustrates how the quantile level required by differential CP to ensure nominal $90\%$ coverage varies with the privacy budget $\varepsilon$ under two different input distributions for $X$.

The figure suggests that, when $\varepsilon$ is large, the globally adjusted level may indeed be conservative, and that the effective adjustment depends on the distribution of the covariates. This observation supports the potential value of a sharper calibration based on local sensitivity information. Developing such a refinement rigorously is an interesting direction for future work.

\begin{figure}[htbp]
    \centering
    \includegraphics[width=0.8\textwidth]{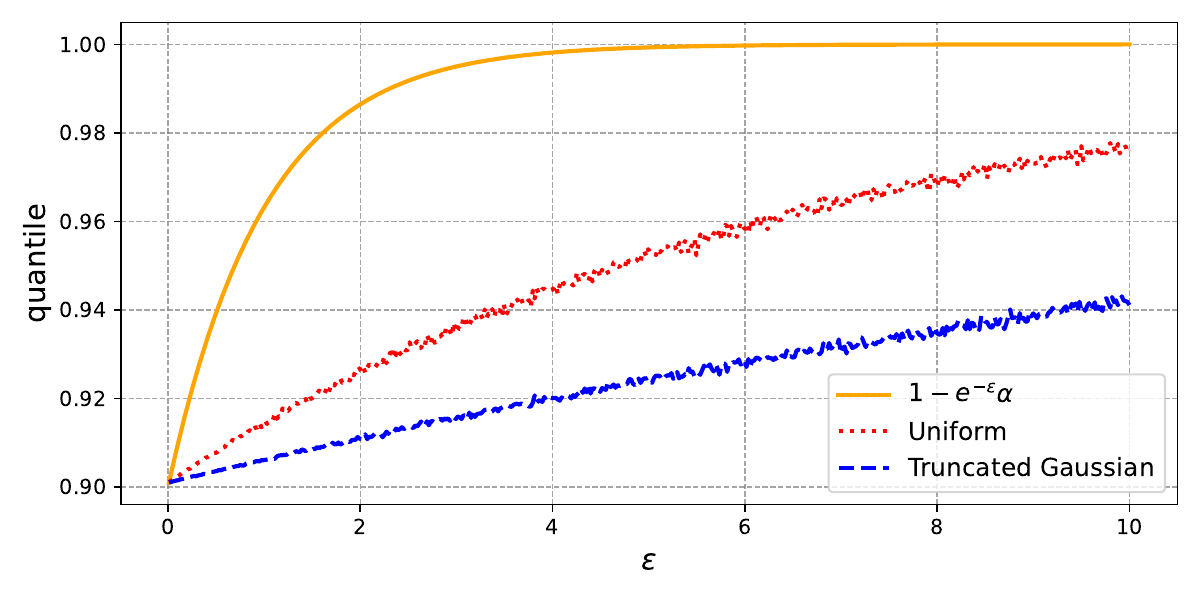}
    \caption{Illustration of the quantile level required by differential CP to ensure $90\%$ coverage as a function of the privacy budget $\varepsilon$. The blue curve corresponds to a truncated Gaussian distribution for $X$ with mean $0$ and standard deviation $10$ on $[-30,30]$, while the red curve corresponds to a uniform distribution for $X$ on $[-10,10]$.}
    \label{fig4}
\end{figure}

\newpage
\bigskip

\section{Additional Experimental Details}
\label{D}

In this appendix, we provide additional implementation details for the numerical experiments reported in Section~5.
The datasets considered include Abalone \citep{nash1994abalone}, Bike Sharing \citep{fanaee2013bikesharing}, Communities and Crime \citep{redmond2002communities}, Physicochemical Properties of Protein Tertiary Structure \citep{rana2013physicochemical}, Power Consumption of Zone~1, Zone~2, and Zone~3 in Tetouan City \citep{salam2018power}, as well as the MNIST \citep{lecun1998mnist} and CIFAR-10 \citep{krizhevsky2009cifar10} datasets.

For the regression tasks, we implement ordinary least squares, ridge regression, and lasso regression by training a single-layer neural network with linear activation. DP is enforced using the Opacus library, with the privacy parameter fixed at \(\delta=10^{-5}\) in all regression experiments. The noise multiplier is chosen to match the target privacy budget, while learning rates and regularization parameters are selected by cross-validation. The nonconformity score is taken to be the absolute residual of the fitted model. Each experiment is repeated 100 times, with the data randomly re-split into training and test sets in each repetition. Within each repetition, DPCP and split CP are evaluated on identical data splits to ensure a fair comparison.

We additionally investigate the performance of DPCP in classification settings. For these experiments, we construct standard deep learning models using Keras and enforce differential privacy through differentially private stochastic gradient descent. Following \citet{sadinle2019least}, we define the nonconformity score for classification as
\begin{align*}
R_i = 1 - \Pr(Y_i \mid X_i),
\end{align*}
where $\Pr(Y_i \mid X_i)$ denotes the predicted class probability.

In the first classification experiment, we employ a convolutional neural network consisting of two convolutional layers followed by two pooling layers and evaluate the performance in the MNIST dataset. We set $\delta = 5 \times 10^{-5}$, use stratified random sampling to select 20{,}000 observations per class for model training and the construction of prediction sets, and reserve an additional 10{,}000 observations for coverage evaluation. The experiment is repeated 100 times.

In the second classification experiment, we use a lightweight ResNet-18 architecture composed of four stacked residual blocks, each containing two basic residual units, and perform experiments on the CIFAR-10 data set. The privacy parameter is set to $\delta = 2 \times 10^{-5}$. We employ stratified random sampling to select 50{,}000 observations per class for model training and prediction set construction, and reserve 10{,}000 observations for coverage evaluation. Due to the increased computational cost, this experiment is repeated 10 times.

\begin{figure}
    \centering
    \includegraphics[width=0.8\textwidth]{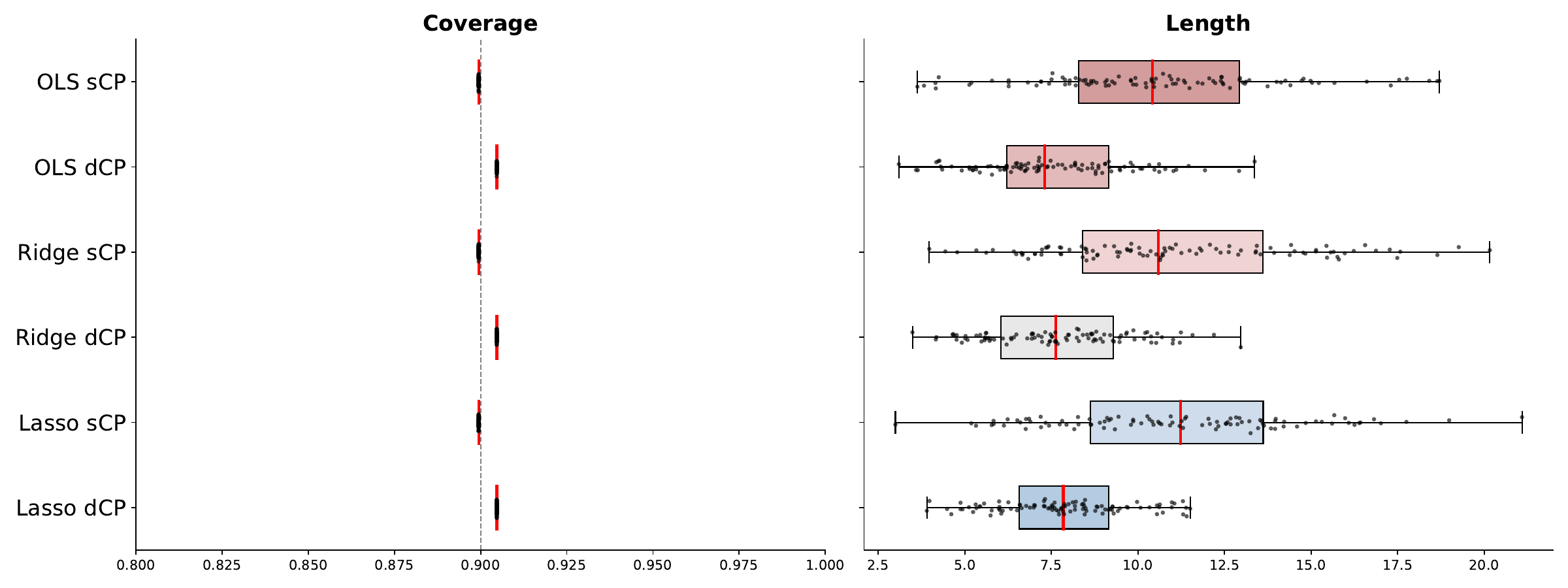}
    \includegraphics[width=0.8\textwidth]{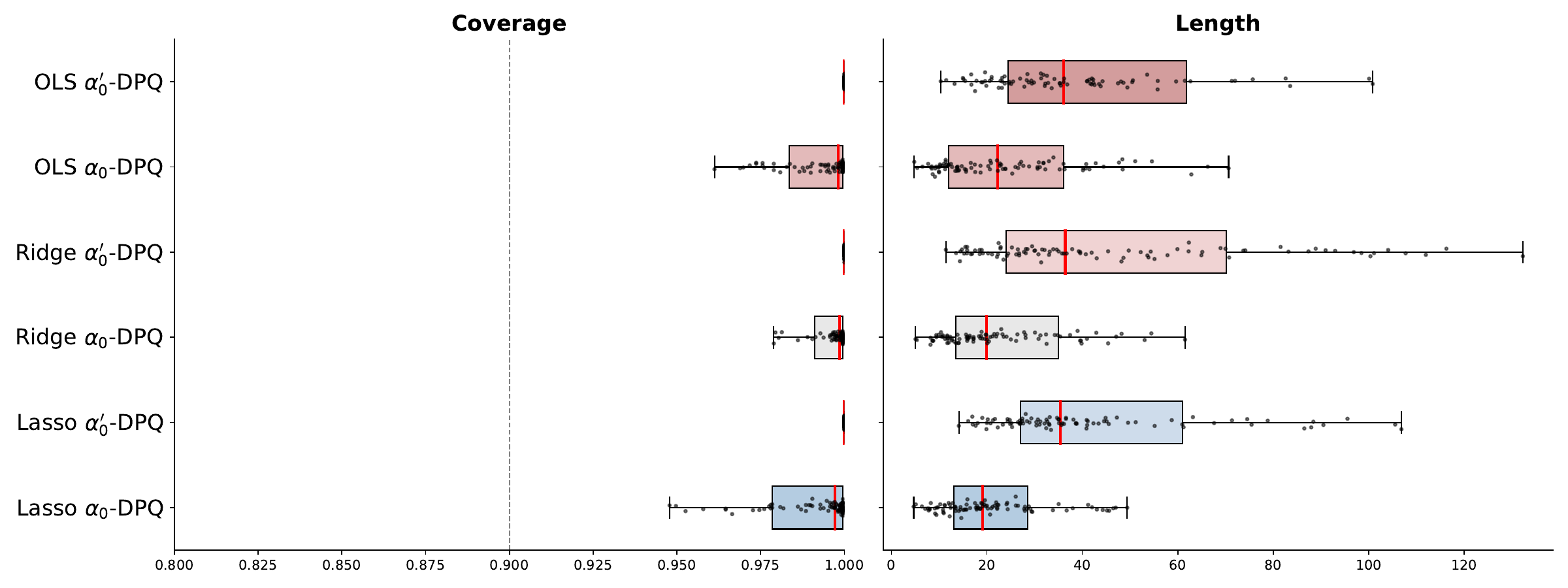}
    \includegraphics[width=0.8\textwidth]{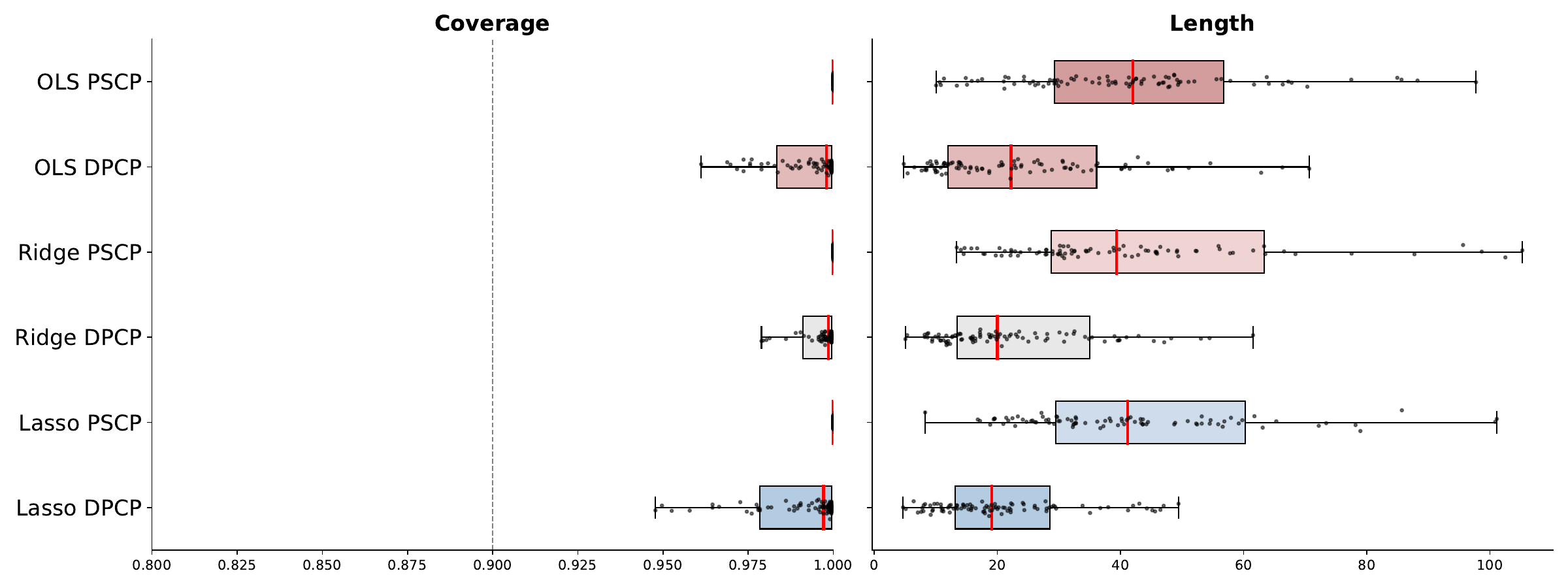}
    \caption{Coverage and length of prediction intervals on the abalone dataset.}
    \label{fig:abalone}
\end{figure}

\begin{figure}
    \centering
    \includegraphics[width=0.8\textwidth]{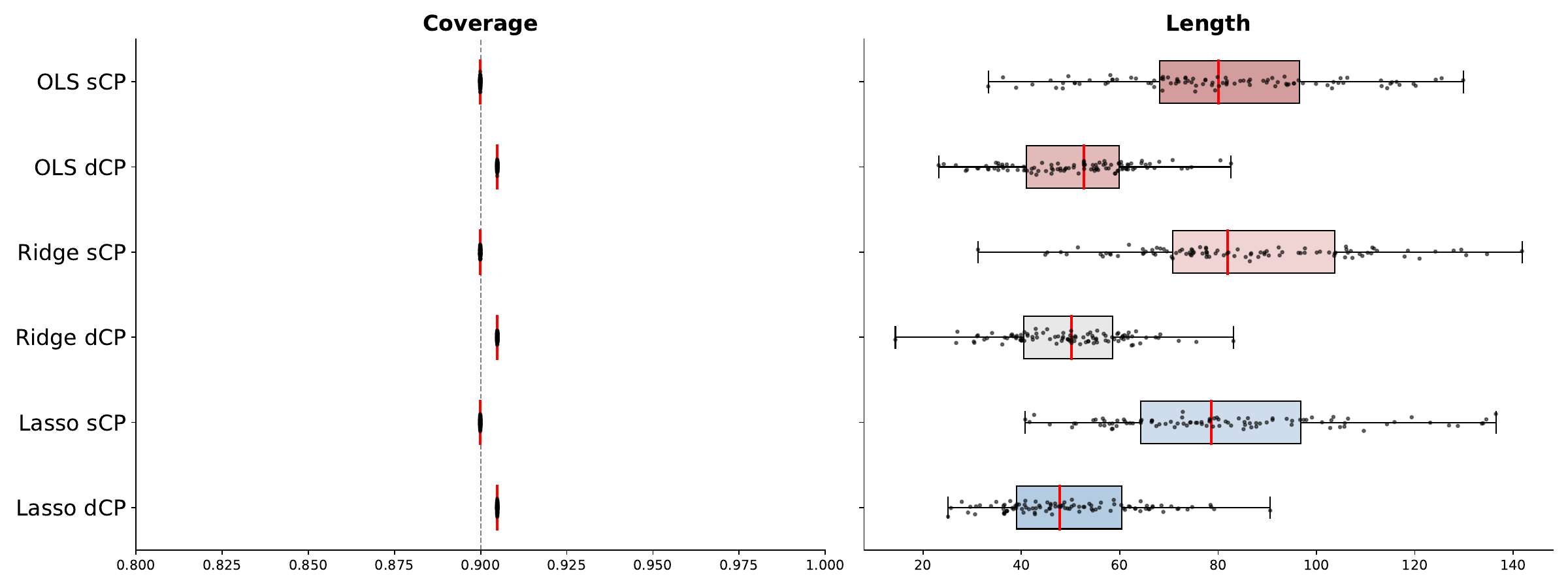}
    \includegraphics[width=0.8\textwidth]{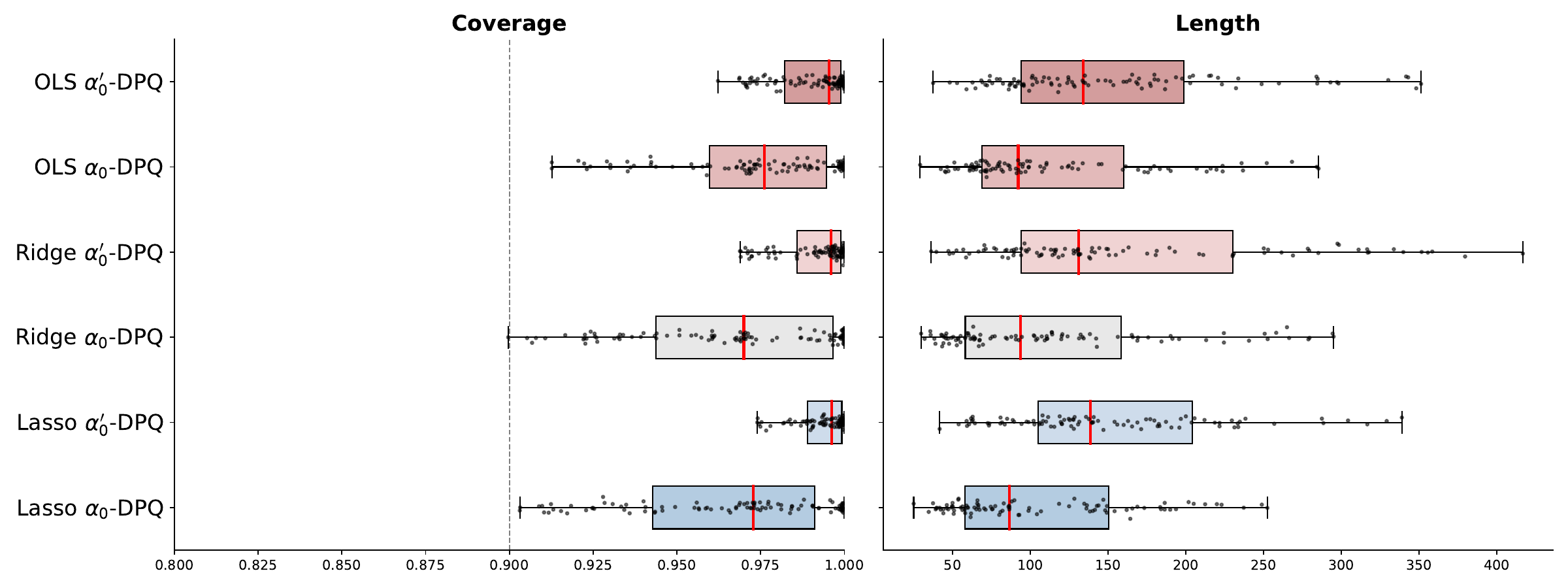}
    \includegraphics[width=0.8\textwidth]{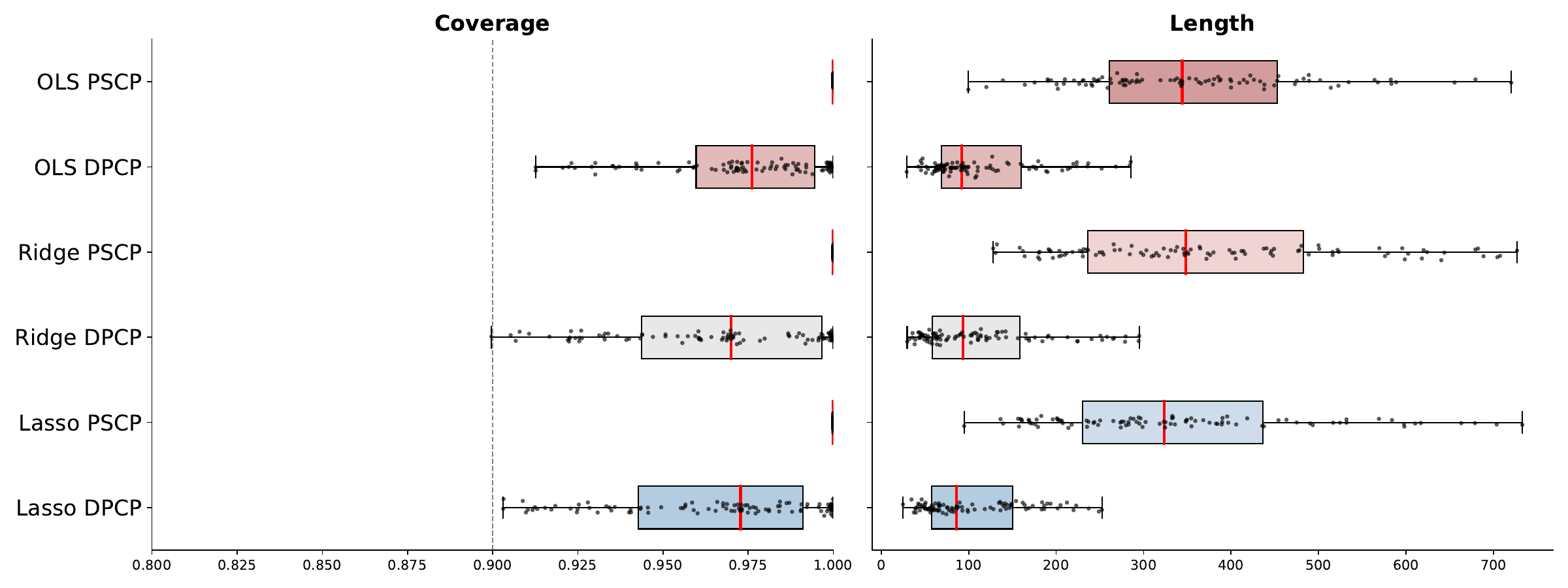}
    \caption{Coverage and length of prediction intervals on the bike sharing dataset.}
    \label{fig:bike_sharing}
\end{figure}

\begin{figure}
    \centering
    \includegraphics[width=0.8\textwidth]{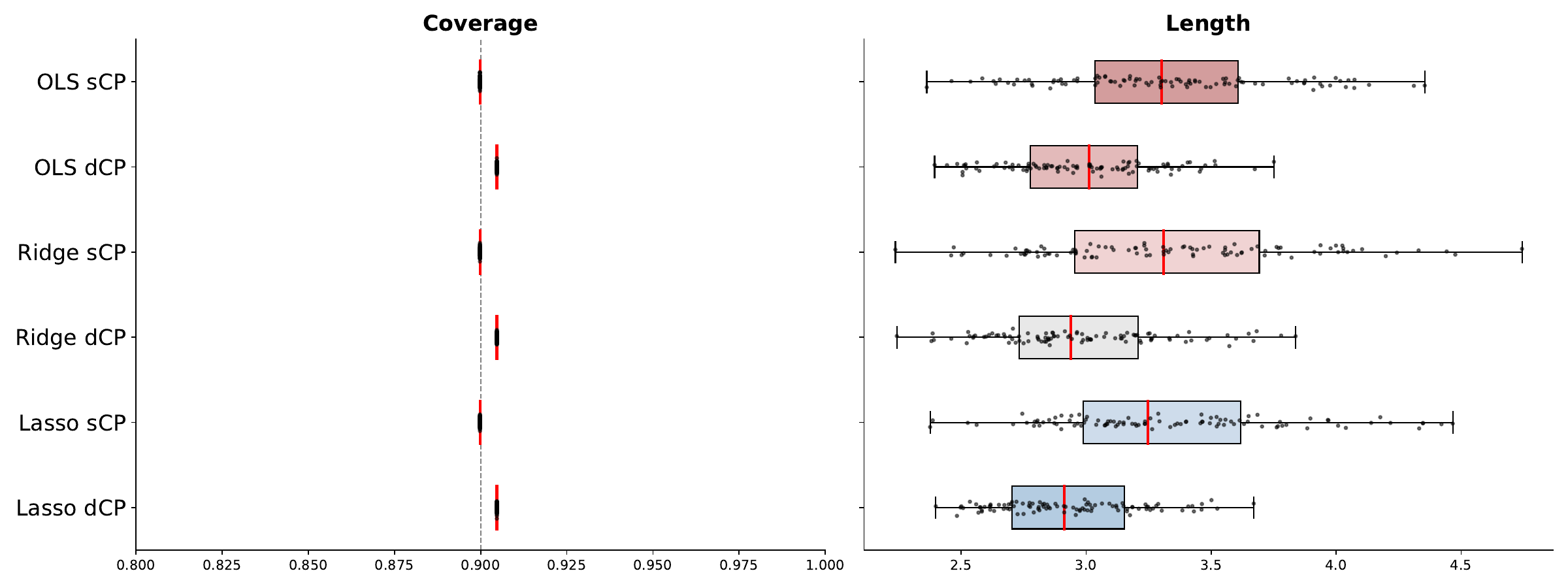}
    \includegraphics[width=0.8\textwidth]{image/communities_and_crime_data.pkl_DP.pdf}
    \includegraphics[width=0.8\textwidth]{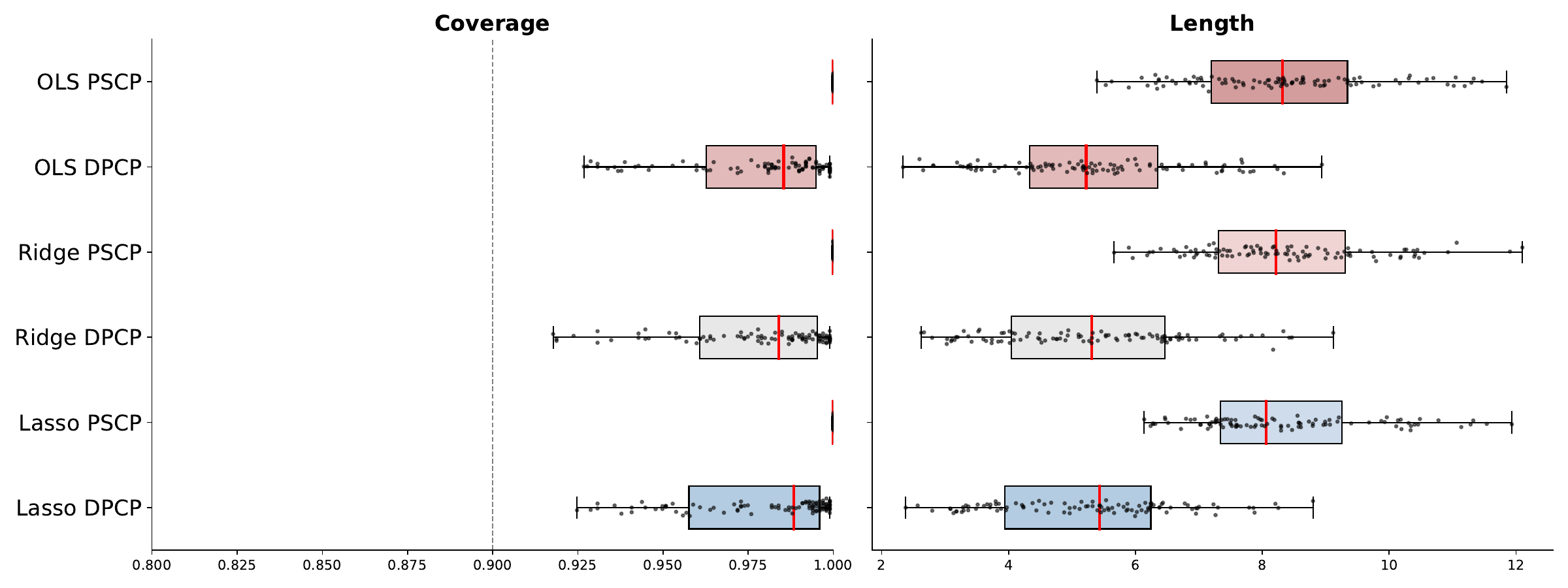}
    \caption{Coverage and length of prediction intervals on the communities and crime dataset.}
    \label{fig:communities_and_crime}
\end{figure}

\begin{figure}
    \centering
    \includegraphics[width=0.8\textwidth]{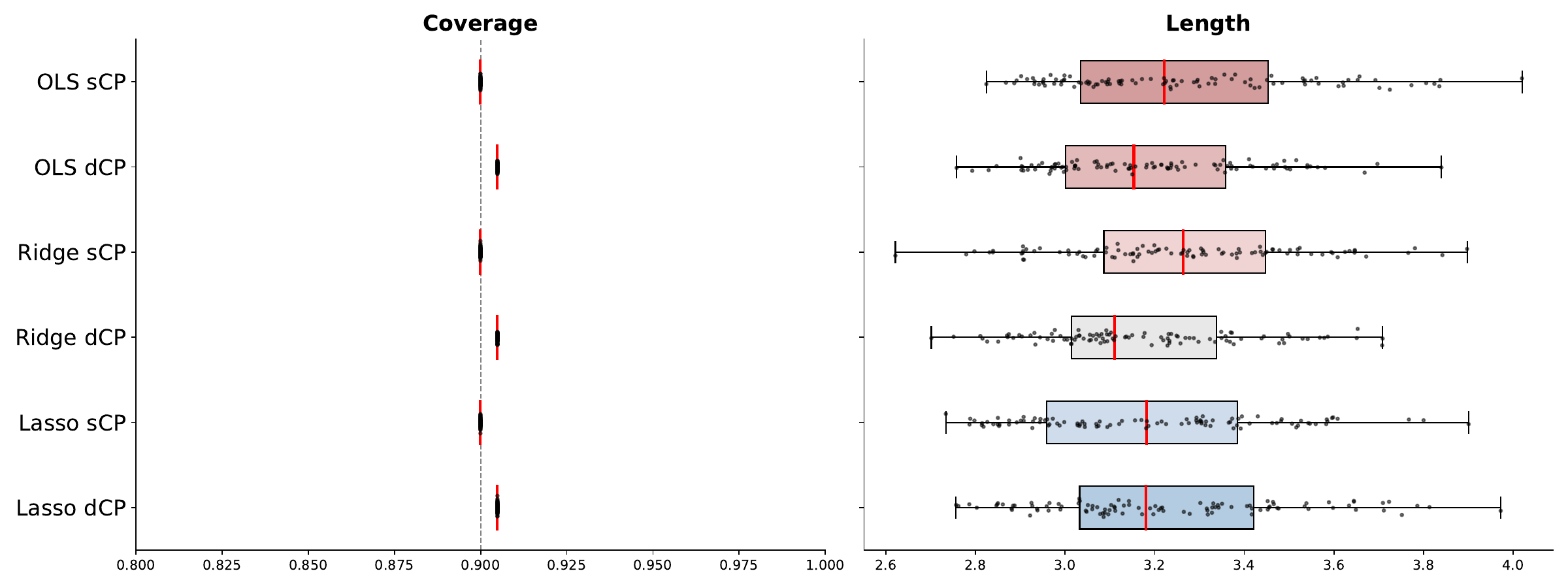}
    \includegraphics[width=0.8\textwidth]{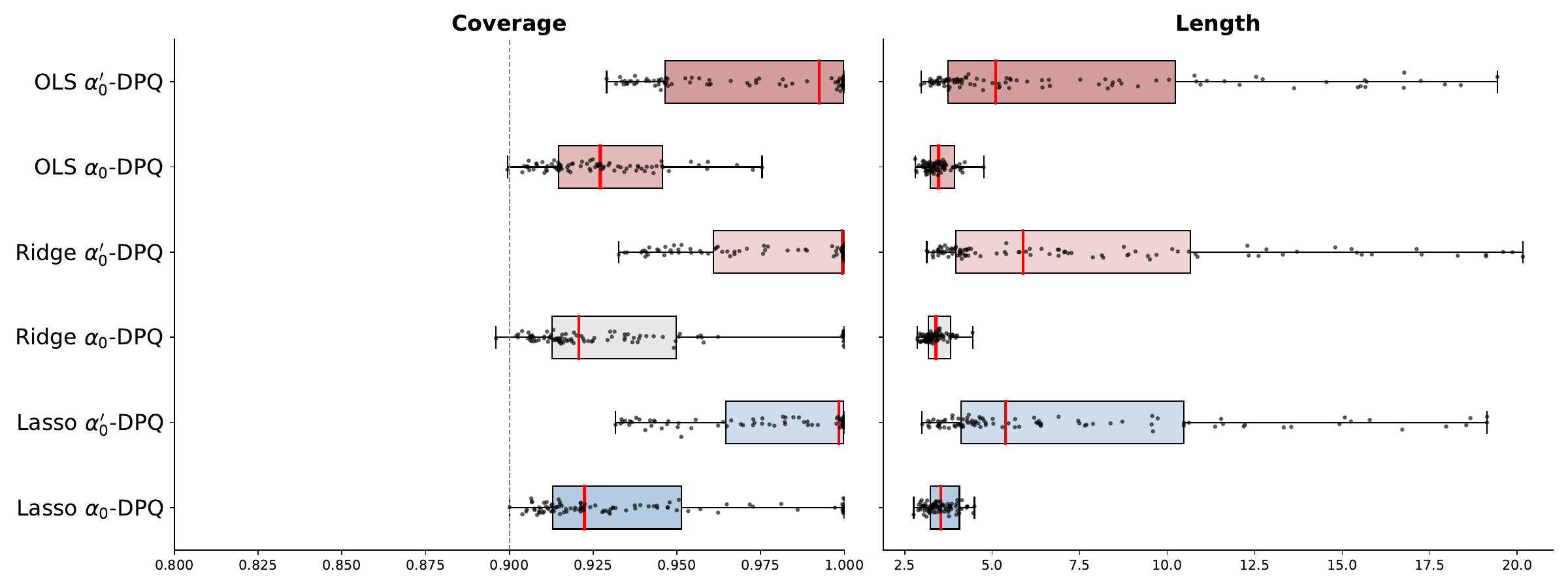}
    \includegraphics[width=0.8\textwidth]{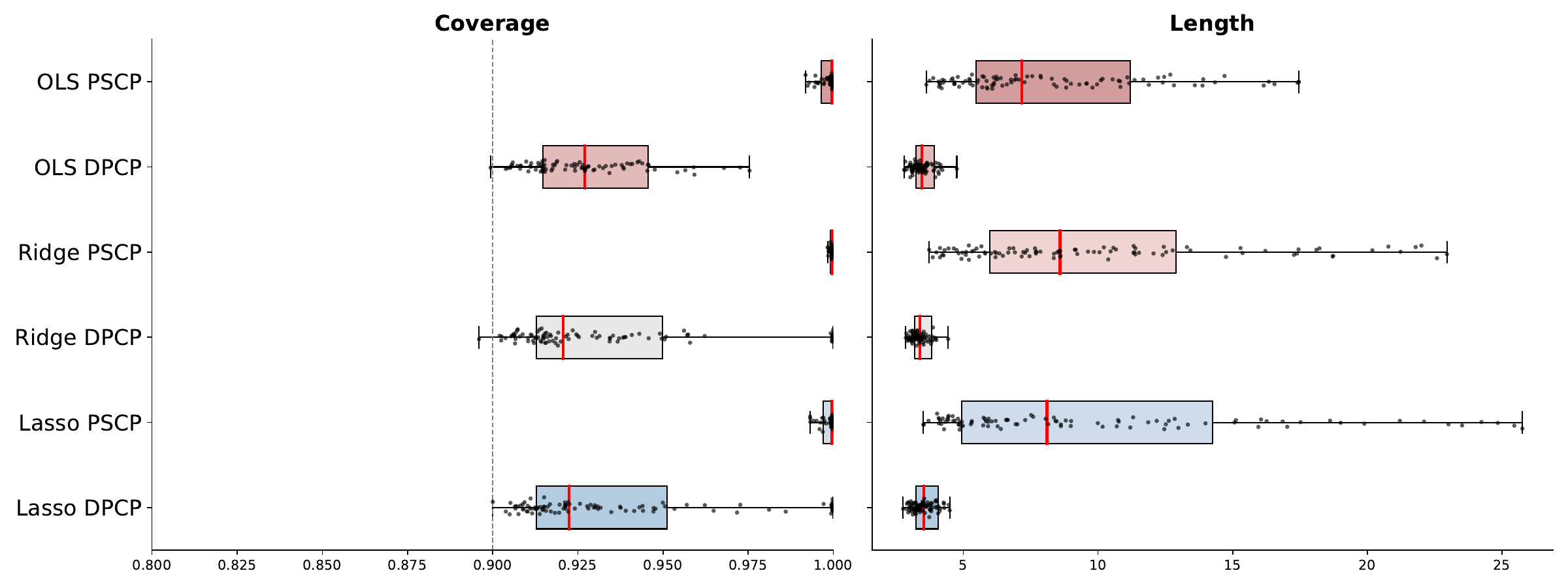}
    \caption{Coverage and length of prediction intervals on the physicochemical properties of protein tertiary structure dataset.}
    \label{fig:protein_properties}
\end{figure}

\begin{figure}
    \centering
    \includegraphics[width=0.8\textwidth]{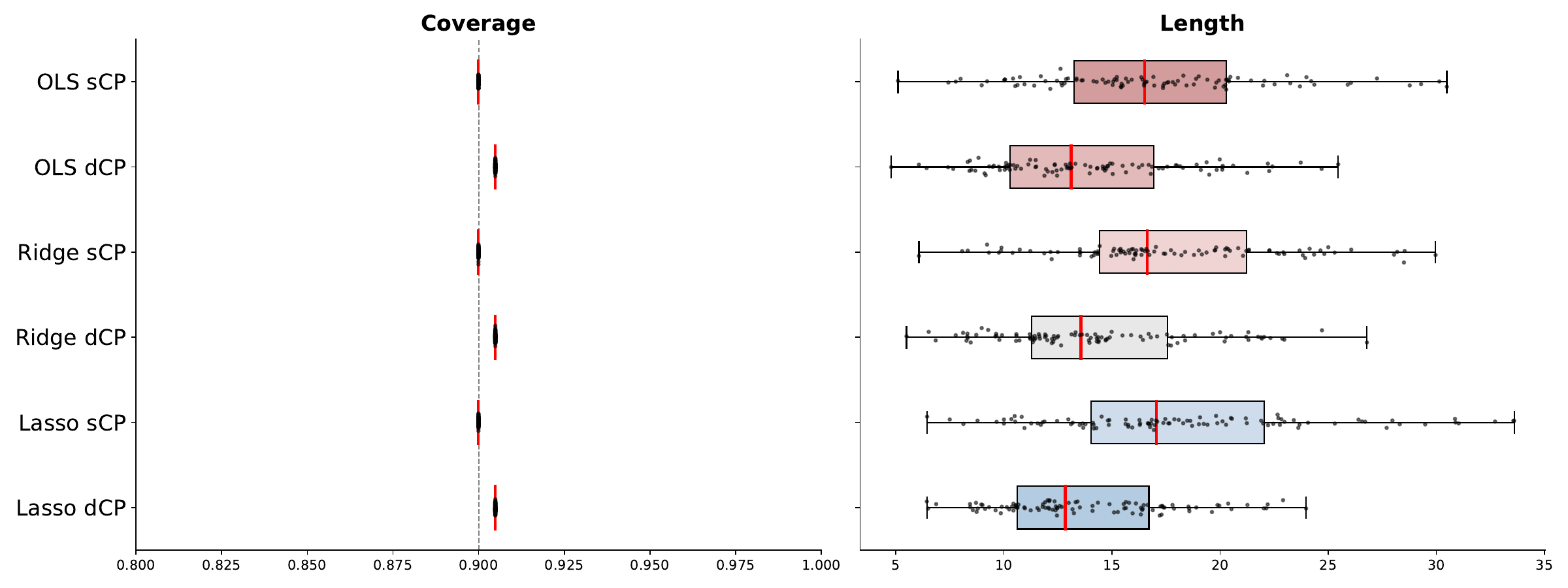}
    \includegraphics[width=0.8\textwidth]{image/power_consumption_of_Zone1_data.pkl_DP.pdf}
    \includegraphics[width=0.8\textwidth]{image/power_consumption_of_Zone1_data.pkl_CPDP.pdf}
    \caption{Coverage and length of prediction intervals on the power consumption of Zone 1 dataset.}
    \label{fig:power_consumption_zone1}
\end{figure}

\begin{figure}
    \centering
    \includegraphics[width=0.8\textwidth]{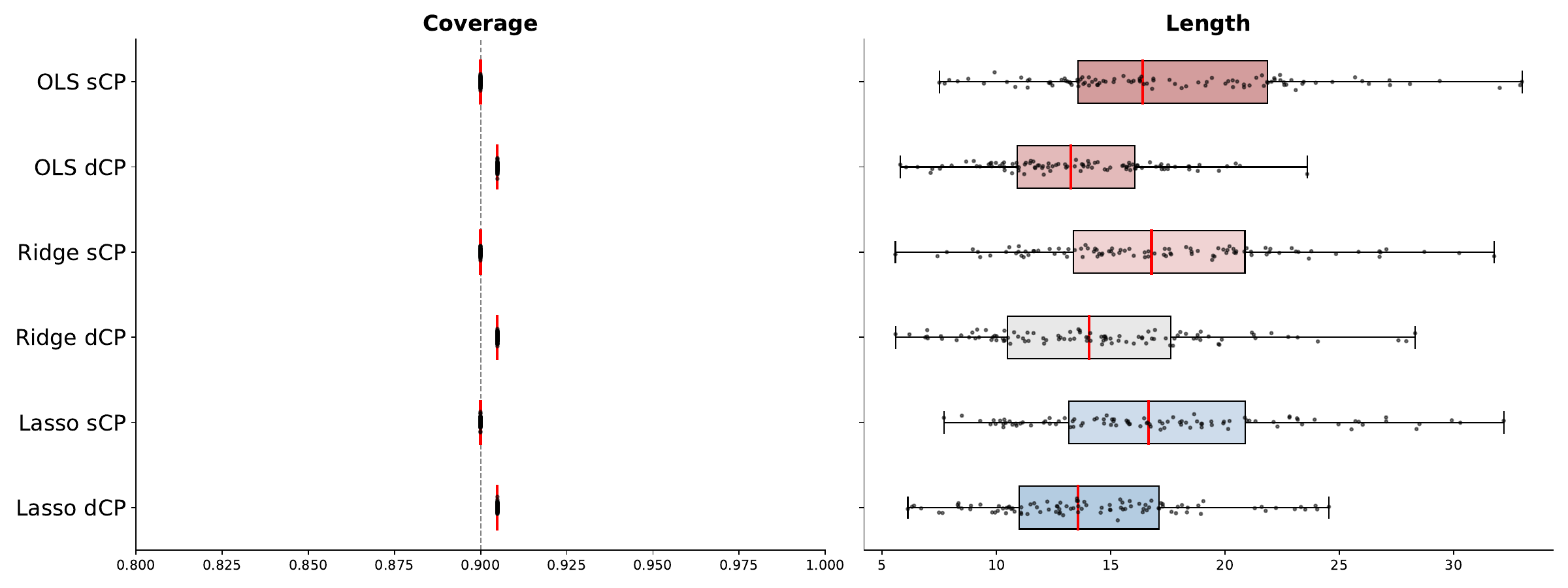}
    \includegraphics[width=0.8\textwidth]{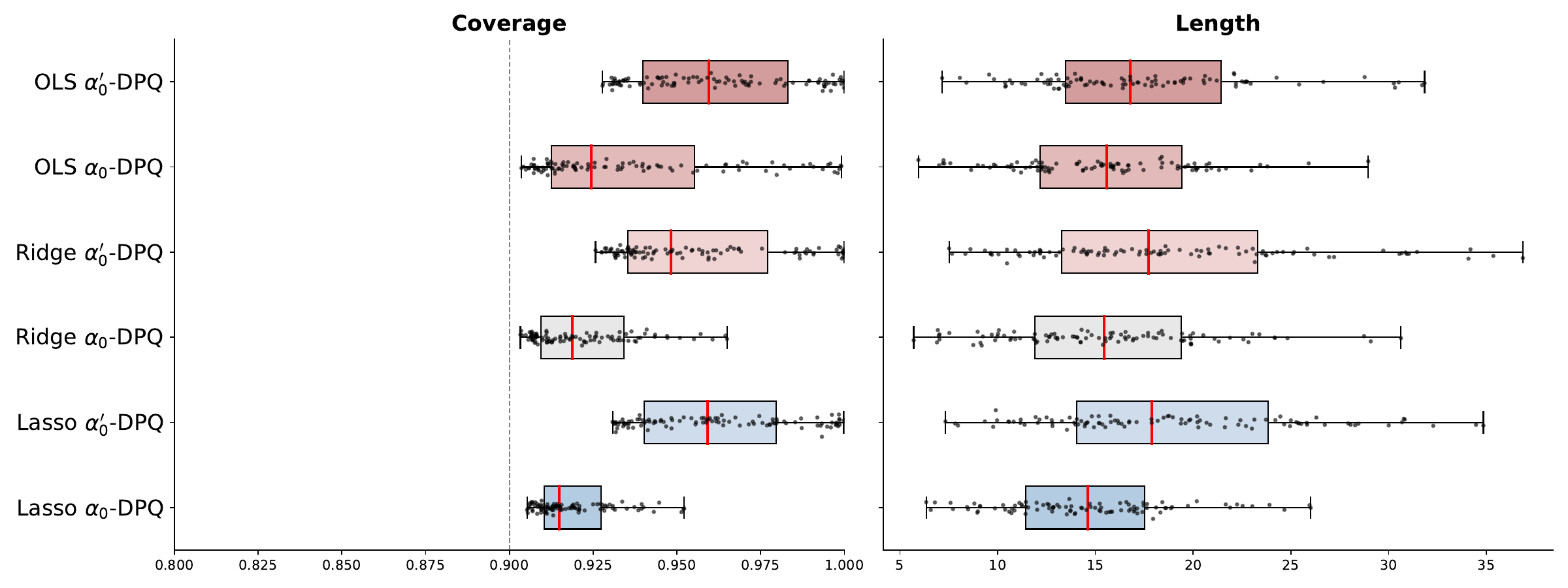}
    \includegraphics[width=0.8\textwidth]{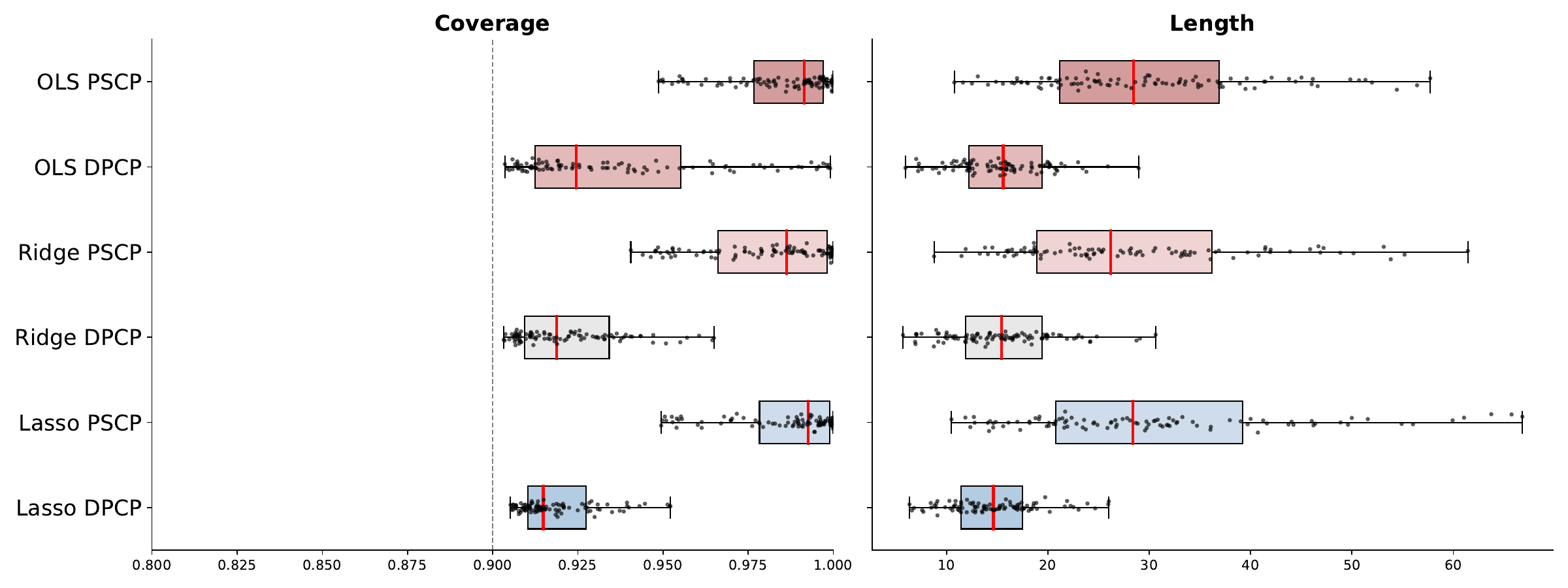}
    \caption{Coverage and length of prediction intervals on the power consumption of Zone 2 dataset.}
    \label{fig:power_consumption_zone2}
\end{figure}

\begin{figure}
    \centering
    \includegraphics[width=0.8\textwidth]{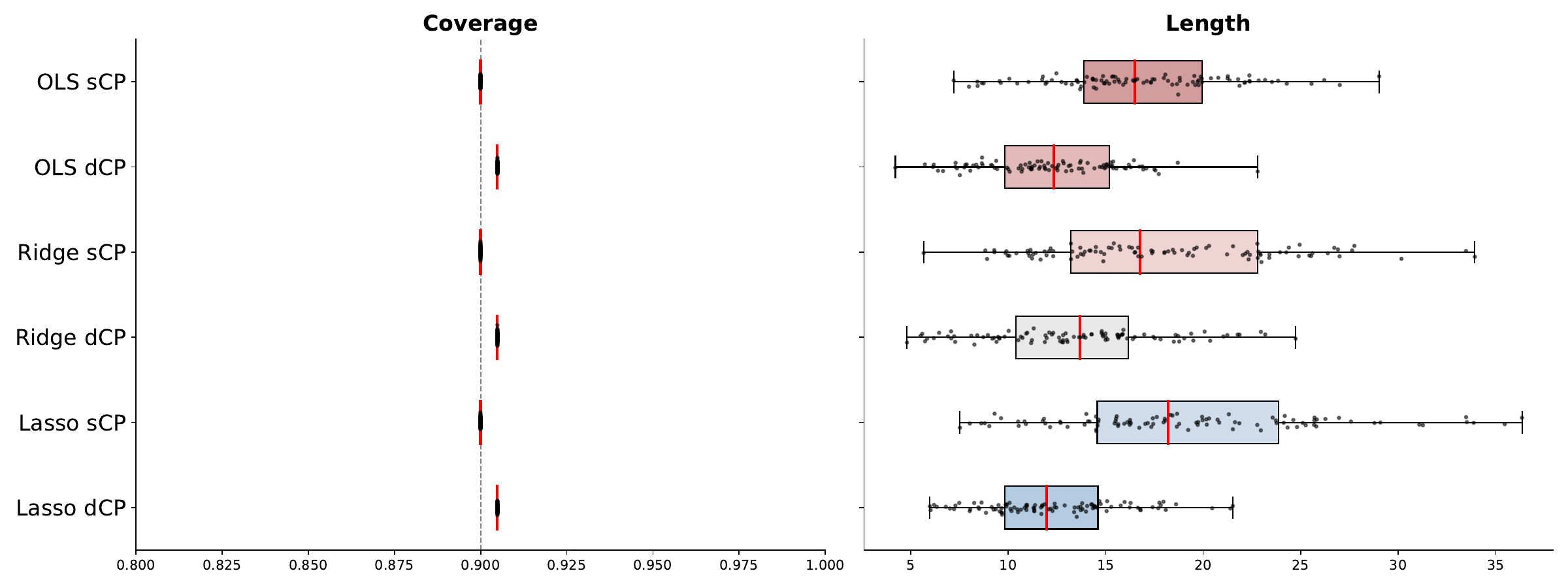}
    \includegraphics[width=0.8\textwidth]{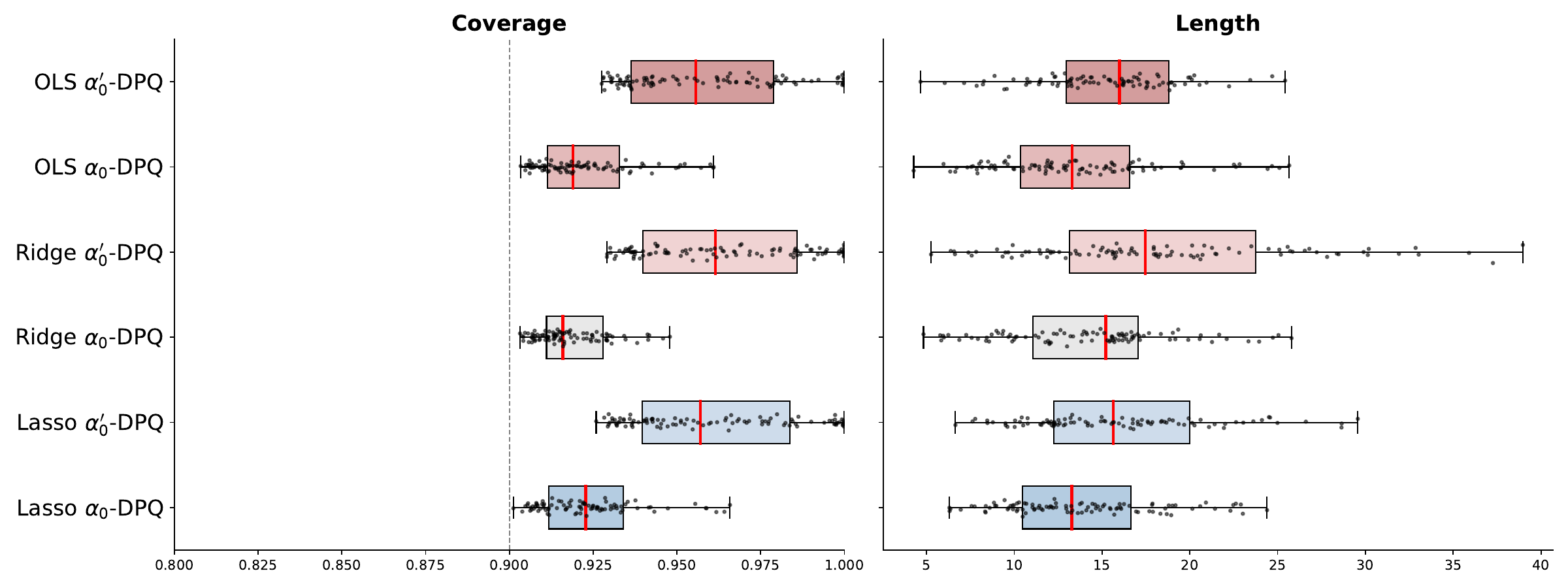}
    \includegraphics[width=0.8\textwidth]{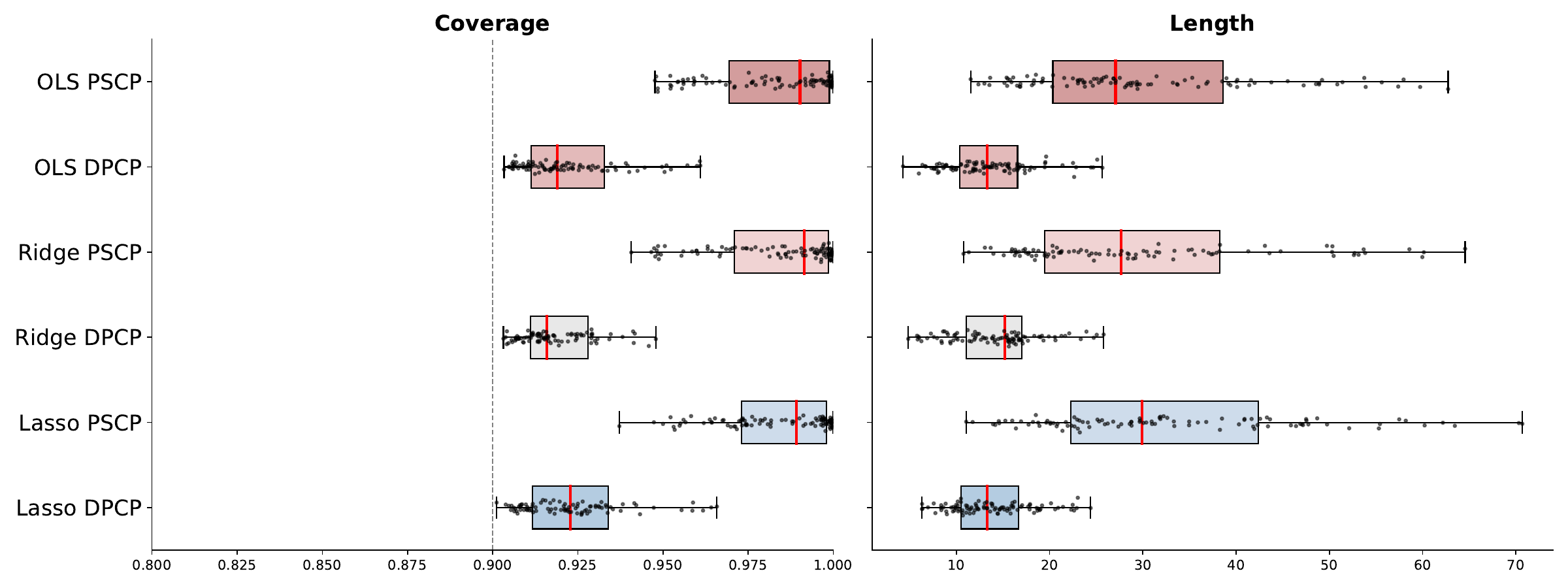}
    \caption{Coverage and length of prediction intervals on the power consumption of Zone 3 dataset.}
    \label{fig:power_consumption_zone3}
\end{figure}

\begin{figure}
    \centering
    \includegraphics[width=0.8\textwidth]{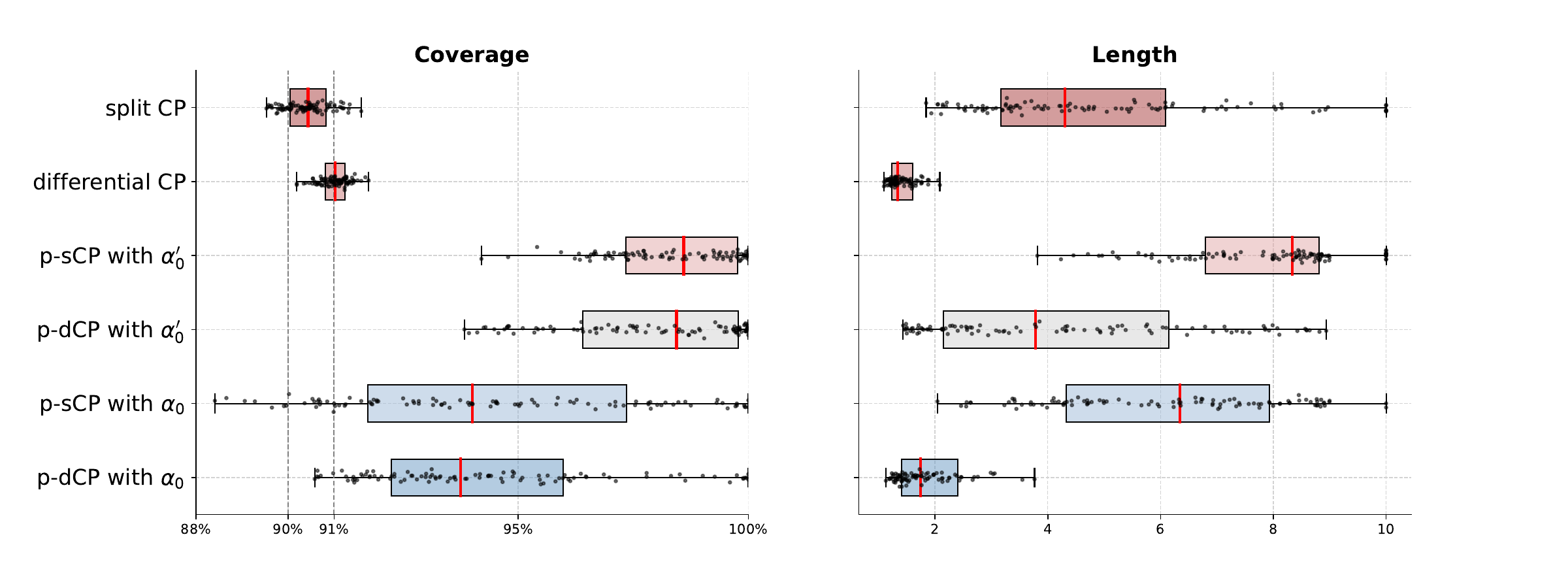}
    \caption{Coverage and length of prediction intervals for the CNN model on the MNIST dataset.}
    \label{fig:MNIST}
\end{figure}

\begin{figure}
    \centering
    \includegraphics[width=0.8\textwidth]{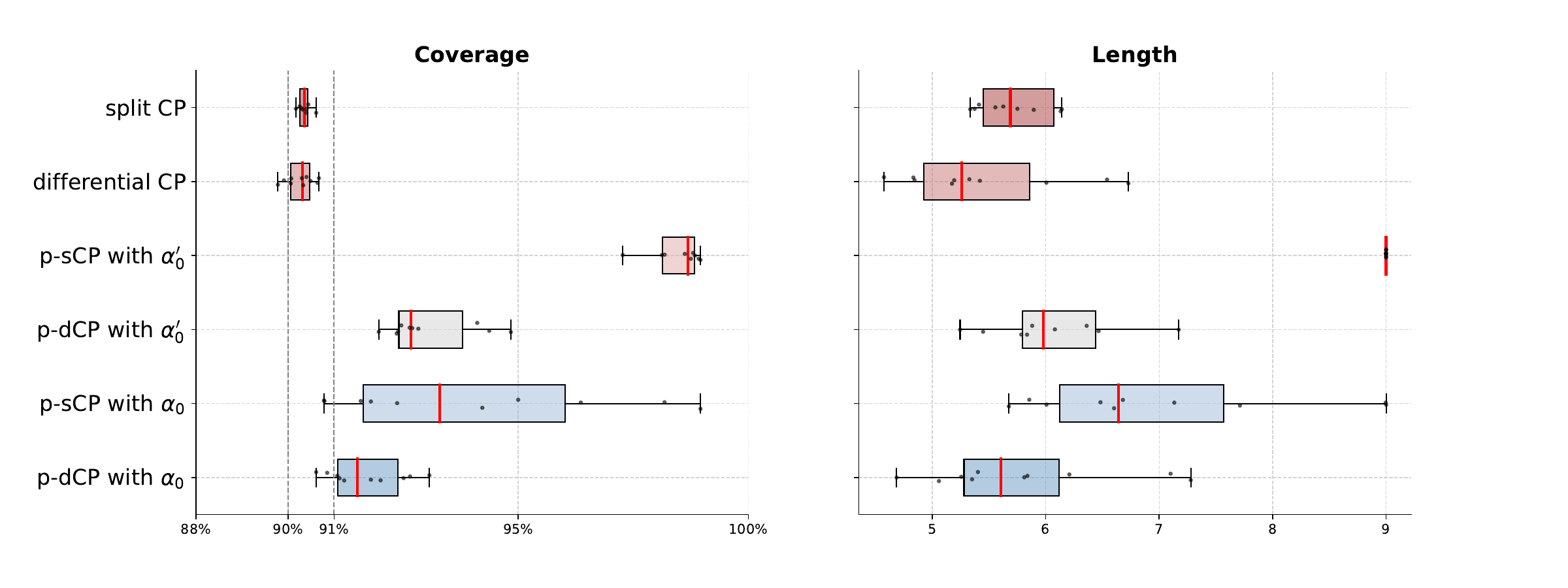}
    \caption{Coverage and length of prediction intervals for the RNN model on the CIFAR-10 dataset.}
    \label{fig:CIFAR-10}
\end{figure}

\newpage

\bibliography{reference}

\begin{thebibliography}{43}
\providecommand{\natexlab}[1]{#1}
\providecommand{\url}[1]{\texttt{#1}}
\expandafter\ifx\csname urlstyle\endcsname\relax
  \providecommand{\doi}[1]{doi: #1}\else
  \providecommand{\doi}{doi: \begingroup \urlstyle{rm}\Url}\fi

\bibitem[Angelopoulos and Bates(2023)]{angelopoulos2023conformal}
Anastasios~N. Angelopoulos and Stephen Bates.
\newblock Conformal prediction: A gentle introduction.
\newblock \emph{Foundations and Trends{\textregistered} in Machine Learning}, 16\penalty0 (4):\penalty0 494--591, 2023.

\bibitem[Angelopoulos et~al.(2022)Angelopoulos, Bates, Zrnic, and Jordan]{angelopoulos2022private}
Anastasios~N. Angelopoulos, Stephen Bates, Tijana Zrnic, and Michael~I Jordan.
\newblock Private prediction sets.
\newblock \emph{Harvard Data Science Review}, 4\penalty0 (2), 2022.

\bibitem[Angelopoulos et~al.(2024)Angelopoulos, Foygel~Barber, and Bates]{angelopoulos2024theoretical}
Anastasios~N. Angelopoulos, Rina Foygel~Barber, and Stephen Bates.
\newblock Theoretical foundations of conformal prediction, 2024.

\bibitem[Bun and Steinke(2016)]{bun2016concentrated}
Mark Bun and Thomas Steinke.
\newblock Concentrated differential privacy: Simplifications, extensions, and lower bounds.
\newblock In \emph{Theory of cryptography conference}, pages 635--658. Springer, 2016.

\bibitem[Bun et~al.(2017)Bun, Steinke, and Ullman]{bun2017make}
Mark Bun, Thomas Steinke, and Jonathan Ullman.
\newblock Make up your mind: The price of online queries in differential privacy.
\newblock In \emph{Proceedings of the twenty-eighth annual ACM-SIAM symposium on discrete algorithms}, pages 1306--1325. SIAM, 2017.

\bibitem[Chaudhuri et~al.(2011)Chaudhuri, Monteleoni, and Sarwate]{chaudhuri2009differentially}
Kamalika Chaudhuri, Claire Monteleoni, and Anand~D Sarwate.
\newblock Differentially private empirical risk minimization.
\newblock \emph{Journal of Machine Learning Research}, 12\penalty0 (3), 2011.

\bibitem[Chaudhuri et~al.(2019)Chaudhuri, Imola, and Machanavajjhala]{chaudhuri2019capacity}
Kamalika Chaudhuri, Jacob Imola, and Ashwin Machanavajjhala.
\newblock Capacity bounded differential privacy.
\newblock \emph{Advances in Neural Information Processing Systems}, 32, 2019.

\bibitem[Dong et~al.(2022)Dong, Roth, and Su]{dong2022gaussian}
Jing Dong, Aaron Roth, and Weijie Su.
\newblock Gaussian differential privacy.
\newblock \emph{Journal of the Royal Statistical Society Series B: Statistical Methodology}, 84\penalty0 (1):\penalty0 3--37, 2022.

\bibitem[Drechsler(2023)]{drechsler2023differential}
J{\"o}rg Drechsler.
\newblock Differential privacy for government agencies—are we there yet?
\newblock \emph{Journal of the American Statistical Association}, 118\penalty0 (541):\penalty0 761--773, 2023.

\bibitem[Dwork(2006)]{dwork2006differential}
Cynthia Dwork.
\newblock Differential privacy.
\newblock In \emph{Automata, Languages and Programming}, volume 4052 of \emph{Lecture Notes in Computer Science}, pages 1--12. Springer, 2006.

\bibitem[Dwork and Roth(2014)]{dwork2013algorithmic}
Cynthia Dwork and Aaron Roth.
\newblock The algorithmic foundations of differential privacy.
\newblock \emph{Foundations and Trends{\textregistered} in Theoretical Computer Science}, 9\penalty0 (3-4):\penalty0 211--407, 2014.

\bibitem[Fanaee-T(2013)]{fanaee2013bikesharing}
Hadi Fanaee-T.
\newblock Bike sharing.
\newblock UCI Machine Learning Repository, 2013.

\bibitem[Fontana et~al.(2023)Fontana, Zeni, and Vantini]{fontana2023conformal}
Matteo Fontana, Gianluca Zeni, and Simone Vantini.
\newblock Conformal prediction: a unified review of theory and new challenges.
\newblock \emph{Bernoulli}, 29\penalty0 (1):\penalty0 1--23, 2023.

\bibitem[Foygel~Barber et~al.(2021)Foygel~Barber, Cand{\`e}s, Ramdas, and Tibshirani]{barber2021predictive}
Rina Foygel~Barber, Emmanuel~J Cand{\`e}s, Aaditya Ramdas, and Ryan~J Tibshirani.
\newblock Predictive inference with the jackknife+.
\newblock \emph{The Annals of Statistics}, 49\penalty0 (1):\penalty0 486--507, 2021.

\bibitem[Foygel~Barber et~al.(2023)Foygel~Barber, Cand\`{e}s, Ramdas, and Tibshirani]{barber2023conformal}
Rina Foygel~Barber, Emmanuel~J Cand\`{e}s, Aaditya Ramdas, and Ryan~J Tibshirani.
\newblock Conformal prediction beyond exchangeability.
\newblock \emph{The Annals of Statistics}, 51\penalty0 (2):\penalty0 816--845, 2023.

\bibitem[Gawlikowski et~al.(2023)Gawlikowski, Tassi, Ali, Lee, Humt, Feng, Kruspe, Triebel, Jung, Roscher, et~al.]{gawlikowski2023survey}
Jakob Gawlikowski, Cedrique Rovile~Njieutcheu Tassi, Mohsin Ali, Jongseok Lee, Matthias Humt, Jianxiang Feng, Anna Kruspe, Rudolph Triebel, Peter Jung, Ribana Roscher, et~al.
\newblock A survey of uncertainty in deep neural networks.
\newblock \emph{Artificial Intelligence Review}, 56\penalty0 (Suppl 1):\penalty0 1513--1589, 2023.

\bibitem[Gibbs and Cand{\`e}s(2024)]{gibbs2024conformal}
Isaac Gibbs and Emmanuel~J Cand{\`e}s.
\newblock Conformal inference for online prediction with arbitrary distribution shifts.
\newblock \emph{Journal of Machine Learning Research}, 25\penalty0 (162):\penalty0 1--36, 2024.

\bibitem[Gui et~al.(2024)Gui, Hore, Ren, and Foygel~Barber]{gui2024conformalized}
Yu~Gui, Rohan Hore, Zhimei Ren, and Rina Foygel~Barber.
\newblock Conformalized survival analysis with adaptive cut-offs.
\newblock \emph{Biometrika}, 111\penalty0 (2):\penalty0 459--477, 2024.

\bibitem[Humbert et~al.(2023)Humbert, Le~Bars, Bellet, and Arlot]{humbert2023one}
Pierre Humbert, Batiste Le~Bars, Aur{\'e}lien Bellet, and Sylvain Arlot.
\newblock One-shot federated conformal prediction.
\newblock In \emph{International Conference on Machine Learning}, pages 14153--14177. PMLR, 2023.

\bibitem[Jin and Cand{\`e}s(2025)]{jin2025model}
Ying Jin and Emmanuel~J Cand{\`e}s.
\newblock Model-free selective inference under covariate shift via weighted conformal p-values.
\newblock \emph{Biometrika}, art. asaf066, 2025.

\bibitem[Kourou et~al.(2015)Kourou, Exarchos, Exarchos, Karamouzis, and Fotiadis]{Kourou2014MachineLA}
Konstantina Kourou, Themis~P Exarchos, Konstantinos~P Exarchos, Michalis~V Karamouzis, and Dimitrios~I Fotiadis.
\newblock Machine learning applications in cancer prognosis and prediction.
\newblock \emph{Computational and structural biotechnology journal}, 13:\penalty0 8--17, 2015.

\bibitem[Krizhevsky et~al.(2009)Krizhevsky, Hinton, et~al.]{krizhevsky2009cifar10}
Alex Krizhevsky, Geoffrey Hinton, et~al.
\newblock Learning multiple layers of features from tiny images.
\newblock 2009.

\bibitem[LeCun(1998)]{lecun1998mnist}
Yann LeCun.
\newblock The mnist database of handwritten digits.
\newblock 1998.

\bibitem[Lei et~al.(2013)Lei, Robins, and Wasserman]{lei2013distribution-free}
Jing Lei, James Robins, and Larry Wasserman.
\newblock Distribution-free prediction sets.
\newblock \emph{Journal of the American Statistical Association}, 108\penalty0 (501):\penalty0 278--287, 2013.

\bibitem[Lei et~al.(2018)Lei, G’Sell, Rinaldo, Tibshirani, and Wasserman]{lei2018distribution}
Jing Lei, Max G’Sell, Alessandro Rinaldo, Ryan~J Tibshirani, and Larry Wasserman.
\newblock Distribution-free predictive inference for regression.
\newblock \emph{Journal of the American Statistical Association}, 113\penalty0 (523):\penalty0 1094--1111, 2018.

\bibitem[Mashrur et~al.(2020)Mashrur, Luo, Zaidi, and Robles-Kelly]{mashrur2020machine}
Akib Mashrur, Wei Luo, Nayyar~A Zaidi, and Antonio Robles-Kelly.
\newblock Machine learning for financial risk management: a survey.
\newblock \emph{IEEE Access}, 8:\penalty0 203203--203223, 2020.

\bibitem[Melotti et~al.(2023)Melotti, Lu, Conde, Zhao, Asvadi, Goncalves, and Premebida]{Melotti2021ProbabilisticAF}
Gledson Melotti, Weihao Lu, Pedro Conde, Dezong Zhao, Alireza Asvadi, Nuno Goncalves, and Cristiano Premebida.
\newblock Probabilistic approach for road-users detection.
\newblock \emph{IEEE Transactions on Intelligent Transportation Systems}, 24\penalty0 (9):\penalty0 9253--9267, 2023.

\bibitem[Mironov(2017)]{mironov2017renyi}
Ilya Mironov.
\newblock R{\'e}nyi differential privacy.
\newblock In \emph{2017 IEEE 30th computer security foundations symposium (CSF)}, pages 263--275. IEEE, 2017.

\bibitem[Nash et~al.(1994)Nash, Sellers, Talbot, Cawthorn, and Ford]{nash1994abalone}
Warwick Nash, Tracy Sellers, Simon Talbot, Andrew Cawthorn, and Wes Ford.
\newblock Abalone.
\newblock UCI Machine Learning Repository, 1994.

\bibitem[Nowotarski and Weron(2018)]{Nowotarski2016RecentAI}
Jakub Nowotarski and Rafa{\l} Weron.
\newblock Recent advances in electricity price forecasting: A review of probabilistic forecasting.
\newblock \emph{Renewable and Sustainable Energy Reviews}, 81:\penalty0 1548--1568, 2018.

\bibitem[Oliveira et~al.(2024)Oliveira, Orenstein, Ramos, and Romano]{oliveira2024split}
Roberto~I Oliveira, Paulo Orenstein, Thiago Ramos, and Joao~Vitor Romano.
\newblock Split conformal prediction and non-exchangeable data.
\newblock \emph{Journal of Machine Learning Research}, 25\penalty0 (225):\penalty0 1--38, 2024.

\bibitem[Papadopoulos et~al.(2002)Papadopoulos, Proedrou, Vovk, and Gammerman]{papadopoulos2002inductive}
Harris Papadopoulos, Kostas Proedrou, Volodya Vovk, and Alex Gammerman.
\newblock Inductive confidence machines for regression.
\newblock In \emph{European conference on machine learning}, pages 345--356. Springer, 2002.

\bibitem[Piao et~al.(2019)Piao, Shi, Yan, Zhang, and Liu]{piao2019privacy}
Chunhui Piao, Yajuan Shi, Jiaqi Yan, Changyou Zhang, and Liping Liu.
\newblock Privacy-preserving governmental data publishing: A fog-computing-based differential privacy approach.
\newblock \emph{Future Generation Computer Systems}, 90:\penalty0 158--174, 2019.

\bibitem[Plassier et~al.(2023)Plassier, Makni, Rubashevskii, Moulines, and Panov]{plassier2023conformal}
Vincent Plassier, Mehdi Makni, Aleksandr Rubashevskii, Eric Moulines, and Maxim Panov.
\newblock Conformal prediction for federated uncertainty quantification under label shift.
\newblock In \emph{International Conference on Machine Learning}, pages 27907--27947. PMLR, 2023.

\bibitem[Rana(2013)]{rana2013physicochemical}
Prashant Rana.
\newblock Physicochemical properties of protein tertiary structure.
\newblock UCI Machine Learning Repository, 2013.

\bibitem[Redmond(2002)]{redmond2002communities}
Michael Redmond.
\newblock Communities and crime.
\newblock UCI Machine Learning Repository, 2002.

\bibitem[Sadinle et~al.(2019)Sadinle, Lei, and Wasserman]{sadinle2019least}
Mauricio Sadinle, Jing Lei, and Larry Wasserman.
\newblock Least ambiguous set-valued classifiers with bounded error levels.
\newblock \emph{Journal of the American Statistical Association}, 114\penalty0 (525):\penalty0 223--234, 2019.

\bibitem[Salam and El~Hibaoui(2018)]{salam2018power}
Abdulwahed Salam and Abdelaaziz El~Hibaoui.
\newblock Power consumption of tetouan city.
\newblock UCI Machine Learning Repository, 2018.

\bibitem[Shafer and Vovk(2008)]{shafer2008tutorial}
Glenn Shafer and Vladimir Vovk.
\newblock A tutorial on conformal prediction.
\newblock \emph{Journal of machine learning research}, 9, 2008.

\bibitem[Tibshirani et~al.(2019)Tibshirani, Foygel~Barber, Cand\`{e}s, and Ramdas]{tibshirani2019conformal}
Ryan~J Tibshirani, Rina Foygel~Barber, Emmanuel~J Cand\`{e}s, and Aaditya Ramdas.
\newblock Conformal prediction under covariate shift.
\newblock \emph{Advances in neural information processing systems}, 32, 2019.

\bibitem[Vovk et~al.(2005)Vovk, Gammerman, and Shafer]{vovk2005algorithmic}
Vladimir Vovk, Alexander Gammerman, and Glenn Shafer.
\newblock \emph{Algorithmic learning in a random world}.
\newblock Springer, 2005.

\bibitem[Wang and Tsai(2022)]{wang2022protection}
Yi-Ren Wang and Yun-Cheng Tsai.
\newblock The protection of data sharing for privacy in financial vision.
\newblock \emph{Applied Sciences}, 12\penalty0 (15):\penalty0 7408, 2022.

\bibitem[Yin et~al.(2024)Yin, Shi, Wang, and Blei]{yin2024conformal}
Mingzhang Yin, Claudia Shi, Yixin Wang, and David~M Blei.
\newblock Conformal sensitivity analysis for individual treatment effects.
\newblock \emph{Journal of the American Statistical Association}, 119\penalty0 (545):\penalty0 122--135, 2024.

\end{thebibliography}

\end{document}